\documentclass[10pt]{article}

\usepackage[utf8]{inputenc} 
\usepackage[T1]{fontenc}    
\usepackage{hyperref}       
\usepackage{url}            
\usepackage{booktabs}       
\usepackage{amsfonts}       
\usepackage{nicefrac}       
\usepackage{microtype}      
\usepackage{comment}

\title{Discovering Potential Correlations via Hypercontractivity}

\author{Hyeji Kim\thanks{Coordinated Science Lab and Department of Electrical and Computer Engineering, University of Illinois at Urbana-Champaign, Urbana, IL 61801, USA; \texttt{hyejikim@illinois.edu} (H.K), \texttt{wgao9@illinois.edu} (W.G), \texttt{pramodv@illinois.edu} (P.V)}, \;\;
Weihao Gao$^*$, \;\;
Sreeram Kannan\thanks{Electrical Engineering Department, University of Washington; \texttt{ksreeram@uw.edu}},\;\;
 Sewoong Oh\thanks{Coordinated Science Lab and Department of Industrial and Enterprise Systems Engineering, University of Illinois at Urbana-Champaign, Urbana, IL 61801, USA; \texttt{swoh@illinois.edu}}, \;\;
 Pramod Viswanath$^*$\\ 
}
\date{}

\usepackage{caption}
\usepackage{subcaption}
\usepackage{graphicx}
\usepackage{hyperref}
\usepackage[cmex10]{amsmath}
\usepackage{graphicx}
\usepackage[numbers,sort&compress]{natbib}
\usepackage{verbatim, animate}
\usepackage{comment, amsmath, amsfonts, amsthm}
\usepackage[mathscr]{euscript}
\usepackage{epsfig}
\usepackage{amsfonts,amssymb,latexsym,amsmath, amsxtra,verbatim}
\usepackage[all]{xy}
\usepackage{psfrag}
\usepackage{xcolor}
\usepackage{lineno}
\usepackage{pst-pdf}
\usepackage{bbm}
\usepackage[utf8]{inputenc} 
\usepackage[T1]{fontenc}    
\usepackage{hyperref}       
\usepackage{url}            
\usepackage{booktabs}       
\usepackage{amsfonts}       
\usepackage{nicefrac}       
\usepackage{microtype}      
\usepackage{float}

\newtheorem{Theorem}{\textbf{Theorem}}

\newtheorem{lemma}{\textbf{lemma}}

\newtheorem{Proposition}{\textbf{Proposition}}

\newtheorem{Example}{\textbf{Example}}
\newtheorem{Remark}{\textbf{Remark}}

\newcommand{\Appendix}{Section}

\newcommand{\ba}{\bar{\alpha}}

\usepackage{xspace}
\usepackage{bbm}
\catcode`@=11
\chardef\mathlig@atcode\count255

\def\actively#1#2{\begingroup\uccode`\~=`#2\relax\uppercase{\endgroup#1~}}
\def\mathlig@gobble{\afterassignment\mathlig@next@cmd\let\mathlig@next= }
\def\mathlig@delim{\mathlig@delim}
\def\mathlig@defcs#1{\expandafter\def\csname#1\endcsname}
\def\mathlig@let@cs#1#2{\expandafter\let\expandafter#1\csname#2\endcsname}
\def\mathlig@appendcs#1#2{\expandafter\edef\csname#1\endcsname{\csname#1\endcsname#2}}

\def\mathlig#1#2{\mathlig@checklig#1\mathlig@end\mathlig@defcs{mathlig@back@#1}{#2}\ignorespaces}

\def\mathlig@checklig#1#2\mathlig@end{%
 \expandafter\ifx\csname mathlig@forw@#1\endcsname\relax
 \expandafter\mathchardef\csname mathlig@back@#1\endcsname=\mathcode`#1%
 \mathcode`#1"8000\actively\def#1{\csname mathlig@look@#1\endcsname}%
 \mathlig@dolig#1\mathlig@delim
\fi
\mathlig@checksuffix#1#2\mathlig@end
}

\def\mathlig@checksuffix#1#2\mathlig@end{%
\ifx\mathlig@delim#2\mathlig@delim\relax\else\mathlig@checksuffix@{#1}#2\mathlig@end\fi
}
\def\mathlig@checksuffix@#1#2#3\mathlig@end{%
\expandafter\ifx\csname mathlig@forw@#1#2\endcsname\relax\mathlig@dosuffix{#1}{#2}\fi
\mathlig@checksuffix{#1#2}#3\mathlig@end
}

\def\mathlig@dosuffix#1#2{%
\mathlig@appendcs{mathlig@toks@#1}{#2}%
\mathlig@dolig{#1}{#2}\mathlig@delim
}

\def\mathlig@dolig#1#2\mathlig@delim{%
 \mathlig@defcs{mathlig@look@#1#2}{%
 \mathlig@let@cs\mathlig@next{mathlig@forw@#1#2}\futurelet\mathlig@next@tok\mathlig@next}%
 \mathlig@defcs{mathlig@forw@#1#2}{%
  \mathlig@let@cs\mathlig@next{mathlig@back@#1#2}%
  \mathlig@let@cs\checker{mathlig@chck@#1#2}%
  \mathlig@let@cs\mathligtoks{mathlig@toks@#1#2}%
  \expandafter\ifx\expandafter\mathlig@delim\mathligtoks\mathlig@delim\relax\else
  \expandafter\checker\mathligtoks\mathlig@delim\fi
  \mathlig@next
 }%
 \mathlig@defcs{mathlig@toks@#1#2}{}%
 \mathlig@defcs{mathlig@chck@#1#2}##1##2\mathlig@delim{%
  \ifx\mathlig@next@tok##1%
   \mathlig@let@cs\mathlig@next@cmd{mathlig@look@#1#2##1}\let\mathlig@next\mathlig@gobble
  \fi 
  \ifx\mathlig@delim##2\mathlig@delim\relax\else
   \csname mathlig@chck@#1#2\endcsname##2\mathlig@delim
  \fi
 }%
 \ifx\mathlig@delim#2\mathlig@delim\else
  \mathlig@defcs{mathlig@back@#1#2}{\csname mathlig@back@#1\endcsname #2}%
 \fi
}%

\catcode`@\mathlig@atcode

\newcommand{\muspace}{\mspace{1mu}}

\DeclareRobustCommand{\scond}{\mathchoice{\muspace\vert\muspace}{\vert}{\vert}{\vert}}
\mathlig{|}{\scond}

\DeclareRobustCommand{\discint}{\mathchoice{\mspace{-1.5mu}:\mspace{-1.5mu}}{\mspace{-1.5mu}:\mspace{-1.5mu}}{:}{:}}
\mathlig{::}{\discint}

\newcommand{\Xcal}{\mathcal{X}}
\newcommand{\Ycal}{\mathcal{Y}}

\renewcommand{\Pr}{\mathscr{P}}

\def\a{\alpha}

\def\e{\epsilon}

\newcommand{\E}{\textsf{E}}
\let\P\relax
\newcommand{\P}{\textsf{P}}

\def\textiid{i.i.d.\@\xspace}
\newcommand\iid{\ifmmode\text{ i.i.d. } \else \textiid \fi}

\def\mathllap{\mathpalette\mathllapinternal}
\def\mathllapinternal#1#2{%
  \llap{$\mathsurround=0pt#1{#2}$}}

\def\clap#1{\hbox to 0pt{\hss#1\hss}}
\def\mathclap{\mathpalette\mathclapinternal}
\def\mathclapinternal#1#2{%
  \clap{$\mathsurround=0pt#1{#2}$}}

\let\oldstackrel\stackrel
\renewcommand{\stackrel}[2]{\oldstackrel{\mathclap{#1}}{#2}}

\renewcommand{\hbar}{h\mathllap{\overline{\vphantom{h}\hphantom{\rule{4.6pt}{0pt}}}\mspace{0.77mu}}}

\catcode`~=11 
\newcommand{\urltilde}{\kern -.06em\lower -.06em\hbox{~}\kern .02em}
\catcode`~=13 

\newcommand{\indep}{\rotatebox[origin=c]{90}{$\models$}}

\newcommand{\defMI}{\sup_{U: U-X-Y, I(U;X) > 0} \frac{I(U;Y)}{I(U;X)}}

\newcommand\numberthis{\addtocounter{equation}{1}\tag{\theequation}}

\usepackage{float}

\begin{document}

\maketitle
\begin{abstract}
Discovering a correlation from one variable to another variable is of fundamental scientific and practical interest. While existing correlation measures are suitable for discovering {\em average} correlation, they fail to discover hidden or {\em potential} correlations.
 To bridge this gap, (i) we postulate a~set of natural axioms that we expect a measure of potential correlation to satisfy; (ii) we show  that the {\em rate} of information bottleneck, i.e., the {\em hypercontractivity}  coefficient, satisfies all the proposed axioms; (iii) we provide a novel estimator to estimate the hypercontractivity coefficient from samples; and (iv)  we provide numerical experiments  demonstrating that this proposed estimator discovers potential correlations among various indicators of WHO datasets, is robust in discovering gene interactions from gene expression time series data, and is statistically more powerful than the estimators for other correlation measures in binary hypothesis testing of canonical examples of potential correlations. 
 \end{abstract}


\section{Introduction}

Measuring the strength of an association between two random variables is a fundamental topic of broad scientific interest. Pearson's correlation coefficient \cite{Pearson1895} dates from over a century ago and has been generalized seven decades ago as maximal correlation (mCor) to handle nonlinear dependencies ~\cite{Hirschfeld1935,gebelein1941statistische,Renyi1959}. Novel correlation measures to identify different kinds of associations continue to be proposed in the literature; these include maximal information coefficient (MIC) \cite{ReshefEtAl2011} and distance correlation (dCor) \cite{Szekely--Rizzo--Bakirov2007}.
Despite the differences, a common theme of measurement of the empirical {\em average} dependence unites the different dependence measures. Alternatively, these are {\em factual} measures of dependence and their relevance is restricted when we seek a {\em potential} dependence of one random variable on another.
For instance, consider a hypothetical city with  very few smokers.
A standard measure of correlation on the historical data in this town on smoking and lung cancer will fail to discover the fact that smoking causes cancer,
since the average correlation is very small.
On the other hand, clearly, there is a potential 
 correlation between smoking and lung cancer; indeed applications of this nature abound in several scenarios in modern data science, including a recent one on genetic pathway discovery~\cite{Krishnaswamy14}.

Discovery of a potential correlation naturally leads one to ask for a measure of potential correlation that is statistically well-founded and addresses practical needs. Such is the focus of this work, where our proposed measure of potential correlation is  based on a novel interpretation of the {\em Information Bottleneck} (IB) principle \cite{Tishby99theinformation}. The IB principle has been used to address one of the fundamental tasks in supervised learning:
given samples $\{X_i,Y_i\}_{i=1}^n$,
how do we find a {\em compact} summary of a variable $X$ that is most {\em informative} in explaining  another variable $Y$. The output of the IB principle is a compact summary of $X$ that is most relevant to $Y$ and has a wide range of applications~\cite{dhillon2003adivisive,Bekkerman2003}.

We use this IB principle to create a measure of correlation based on the following intuition: if $X$ is (potentially) correlated with $Y$, then a relatively compact summary of $X$ can still be very informative about $Y$. In other words,
 the maximal ratio of how informative a summary can be in explaining $Y$ to how compact a summary is with respect to $X$ is, conceptually speaking, an indicator of potential correlation from $X$ to $Y$. 
  Quantifying the compactness by $I(U;X)$ and the information by $I(U;Y)$ we consider the {\em rate of information bottleneck} as a measure of potential correlation:
\begin{eqnarray}
 	s (X;Y) \;\;  \equiv  \;\; \sup_{U\text{--}X\text{--}Y}  \frac{I(U;Y)}{I(U;X)} \;,
	\label{def:hypercontractivity}
\end{eqnarray}
where $U-X-Y$ forms a Markov chain and the supremum is over all summaries $U$ of $X$. 
This intuition is made precise in Section~\ref{sec:axiom}, 
where we formally define a natural notion of potential correlation (Axiom 6), and 
show that the rate of information bottleneck $s(X;Y)$ captures this potential correlation (Theorem~\ref{thm:axiom}) while other standard measures of correlation fail (Theorem \ref{prop-max}).

This ratio has only recently been identified as the {\em hypercontractivity} coefficient \cite{Anantharam--Gohari--Kamath--Nair2013}. 
Hypercontractivity has a distinguished and central role in a large number of technical arenas including  quantum physics~\cite{quantum,nelson1973construction}, theoretical computer science~\cite{KKL,o2014analysis}, mathematics~\cite{Bonami1970,beckner1975inequalities,gross1975} and probability theory \cite{ahlswede1976spreading,mossel2013reverse}.
In this paper, we provide a novel interpretation to the hypercontractivity coefficient as a measure of potential correlation by demonstrating that it satisfies a natural set of axioms such a measure is expected to obey.

For practical use in discovering correlations,
the standard correlation coefficients are equipped with
corresponding natural sample-based estimators.
However, for hypercontractivity coefficient, estimating it from samples is
 widely acknowledged to be challenging,
especially for continuous random variables~\cite{private_correspondence,deepVIB, dropout2016}.
There is no existing algorithm to estimate the hypercontractivity coefficient in general \cite{private_correspondence},
and there is no existing algorithm for solving IB from samples either \cite{deepVIB, dropout2016}.
We provide a novel estimator of the hypercontractivity coefficient
-- the first of its kind -- by
bringing together the recent theoretical discoveries in
\cite{Anantharam--Gohari--Kamath--Nair2013,nair2014gaussian} of an alternate definition of hypercontractivity coefficient
as ratio of Kullback-Leibler divergences defined in \eqref{defKL-dicrete},
and recent advances in joint optimization (the max step in Equation~\ref{def:hypercontractivity}) and estimating information measures from samples using importance sampling \cite{Gao16}.

Our {\bf main contributions} are the following:

\begin{itemize}
	\item We postulate a set of natural axioms that a measure of potential correlation from $X$ to $Y$ should satisfy (Section~\ref{axiom-subsection}).
	
	\item We show that $\sqrt{s(X;Y)}$, our proposed measure of potential correlation,
		satisfies all the axioms we postulate (Section~\ref{sec2.2}). 
		In comparison, we prove that existing standard measures of correlation
		not only fail to satisfy the proposed axioms,
		but also fail to capture  canonical potential correlations captured by $\sqrt{s(X;Y)}$ (Section~\ref{sec2.3}). Another natural candidate is mutual information, but it is not clear how to interpret the value of mutual information as it is unnormalized, unlike all other measures of correlation which are between zero and one. 

	\item Computation of the hypercontractivity coefficient from samples is known to be a challenging open problem. We in troduce a novel estimator to compute hypercontractivity coefficient from i.i.d.\ samples in a statistically consistent manner for continuous random variables, using ideas from importance sampling and kernel density estimation (Section ~\ref{sec:estimate}). 
			
	\item In a series of synthetic experiments, we show empirically that our estimator for the hypercontractivity coefficient is statistically more powerful in discovering a potential correlation than existing correlation estimators; a larger power means a larger successful detection rate for a fixed false alarm rate (Section~\ref{sec:power}). 
	
	\item We show applications of our estimator of hypercontractivity coefficient in two important datasets: In Section~\ref{sec:who}, we demonstrate that it discovers hidden potential correlations among various national indicators in WHO datasets, including how aid is potentially correlated with the income growth. 
	In Section~\ref{sec:dremi}, we consider the following gene pathway recovery problem:  
	we are given samples of four gene expressions time series. Assuming we know that gene A causes B, that B causes C, and that C causes D, the problem is to discover that these causations occur in the sequential order: A to B, and then B to C, and then C to D. We show empirically that the estimator of the hypercontractivity coefficient recovers this order accurately from a vastly smaller number of samples compared to other state-of-the art causal influence estimators.
		
	\end{itemize}

\section{Axiomatic approach to measure potential correlations}
\label{sec:axiom}

We propose a set of axioms that we expect a measure of potential correlation to satisfy. We then show that hypercontractivity coefficient, first introduced in~\cite{ahlswede1976spreading}, satisfies all the proposed axioms, hence propose hypercontractivity coefficient as a measure of potential correlation. 
We also show that other standard correlation coefficients and mutual information, on the other hand, violate the proposed axioms.


\subsection{Axioms for potential correlation}\label{axiom-subsection}
We postulate that a {\em measure of potential correlation} $\rho^*: \Xcal \times \Ycal \rightarrow [0,1]$ between two random variables $X \in \Xcal$ and $Y \in \Ycal$ should satisfy: 
\begin{enumerate}
	\item $\rho^*(X,Y)$ is defined for any pair of non-constant random variables $X$ and $Y$.
	\item $0 \le \rho^*(X,Y) \le 1$.
	\item $\rho^*(X,Y) = 0$ iff $X$ and $Y$ are statistically independent.
	\item For bijective Borel-measurable functions $f,g: \mathbb{R} \rightarrow \mathbb{R}$, $\rho^*(X,Y) = \rho^* (f(X), g(Y))$.
	\item If $(X,Y) \sim \mathcal{N}(\mu, \Sigma)$, then $\rho^* (X,Y) = |\rho|$, where $\rho$ is the Pearson correlation coefficient.
	\item $\rho^*(X,Y) = 1$ if there exists a subset $\Xcal_r \subseteq \Xcal$ such that for a pair of continuous random variables $(X,Y) \in \Xcal_r \times \Ycal$, 
	$Y = f(X)$ 
	for a Borel-measurable and non-constant continuous function  $f$. 
\end{enumerate}

\begin{figure}[!ht] 
	\begin{center}
 	\includegraphics[scale=0.5]{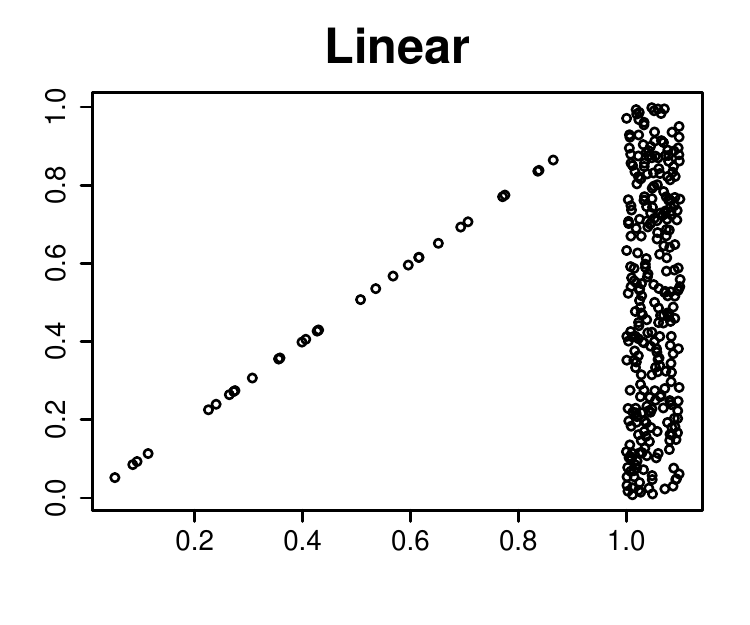}
	\put(-57,-2){$X$}
	\put(-115,50){$Y$}
	\hspace{0.5cm}
 	\includegraphics[scale=0.5]{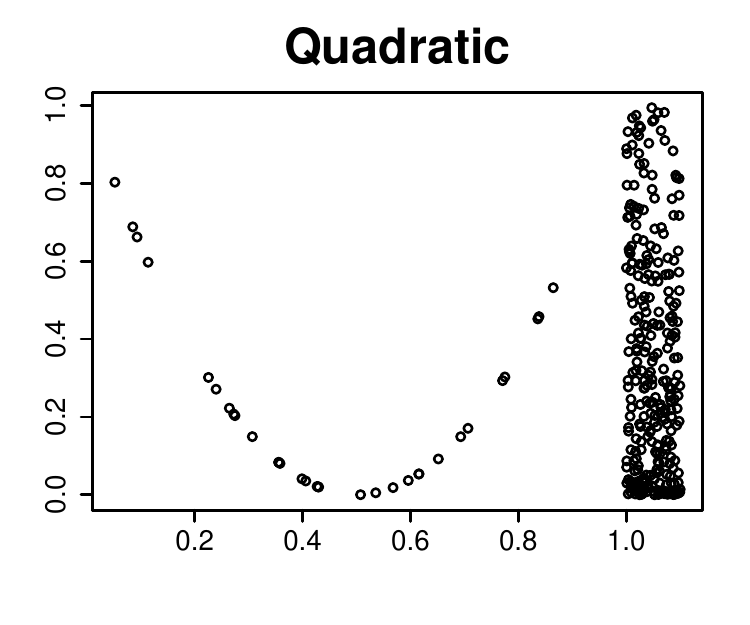}
	\put(-57,-2){$X$}
	\put(-115,50){$Y$}
 	\end{center}
	\caption{A measure of potential correlation should capture the rare correlation in $X\in[0,1]$ in these examples
			which satisfy Axiom 6 for a linear and a quadratic function, respectively.}
	\label{fig:axiom}
\end{figure}
Axioms 1-5 are identical to a subset of the celebrated axioms of R\'enyi in \cite{Renyi1959},
which ensure that the measure is properly normalized and invariant under bijective transformations, and recovers
the Pearson correlation for jointly Gaussian random variables.
R\'enyi's original axioms for a {\em measure of correlation}
in \cite{Renyi1959} included  Axioms 1-5 and also that the measure $\rho^*$ of correlation should satisfy
\begin{enumerate}
	\item[6'.] {\color{black} $\rho^*(X,Y) = 1$ if for Borel-measurable functions $f$ or $g$, $Y = f(X)$ or $X = g(Y)$. }
	\item[7'.] {\color{black} $\rho^*(X;Y) = \rho^*(Y;X)$.}
\end{enumerate}
The Pearson correlation violates a subset (3, 4, and 6') of R\'enyi's axioms.
Together with recent empirical successes in multimodal deep learning
(e.g.~\cite{ngiam2011multimodal,srivastava2012multimodal,andrew2013deep}),
R\'enyi's axiomatic approach has been a major justification of
Hirschfeld-Gebelein-R\'enyi (HGR) 
maximal correlation coefficient defined as  ${\rm mCor}(X,Y) := \sup_{f,g} \E[f(X)g(Y)]$,
which satisfies all R\'enyi's axioms \cite{Hirschfeld1935}.
Here, the supremum is over all measurable functions
with  $\E[f(X)] = \E[g(Y)] = 0$
and   $ \E[f^2(X)] =\E[g^2(Y)] = 1$.
However, maximal correlation is not the only measure satisfying all of R\'enyi's axioms, as we show in the following.
\begin{Proposition}
	\label{pro:renyi}\textnormal{
	For any function $F:[0,1]\times[0,1]\to[0,1]$ satisfying
	$F(x,y)=F(y,x)$, $F(x,x)=x$, and $F(x,y) = 0$ only if $xy = 0$,
	the symmetrized $F(\sqrt{s(X;Y)},\sqrt{s(Y;X)})$ satisfies
	 all R\'enyi's axioms.
}\end{Proposition}

This follows from the fact that the hypercontractivity coefficient $\sqrt{s(X;Y)}$
satisfies all but the symmetry in Axiom 7' (Theorem~\ref{thm:axiom}), and it follows that a symmetrized version satisfies all axioms,
e.g.~$(1/2)(\sqrt{s(X;Y)}+\sqrt{s(Y;X)})$ and $(s(X;Y)s(Y;X))^{1/4}$. 
A formal proof is provided in \Appendix~\ref{sec:renyi_proof}.

From the original R\'enyi's axioms, for a potential correlation measure,
 we remove Axiom 7' that ensures symmetry,
as directionality is fundamental in measuring the potential correlation from $X$ to $Y$.
We further replace Axiom 6' by Axiom 6,
as a variable $X$ has a full potential to be correlated with $Y$ if
there exists a domain $\Xcal_r$ such that $X$ and $Y$ are deterministically dependent and non-degenerate (i.e.~not a constant function), as
illustrated in Figure~\ref{fig:axiom} for a linear function and a quadratic function.

\subsection{The hypercontractivity coefficient satisfies all axioms}\label{sec2.2}
We show that the hypercontractivity coefficient defined in Equation~\eqref{def:hypercontractivity} satisfies all Axioms 1-6. 
Intuitively, $s(X;Y)$ 
measures  how much potential correlation $X$ has with $Y$.
For example, if $X$ and $Y$ are independent, then $s(X;Y)=0$ as $X$ has no correlation with $Y$ (Axiom 3).
By data processing inequality, it follows that it is a measure between zero and one (Axiom 2) and also invariant under bijective transformations (Axiom 4). For jointly
Gaussian variables $X$ and $Y$ with the Pearson correlation $\rho$, we can show that $s(X;Y)=s(Y;X)=\rho^2$. Hence, the squared-root of $s(X;Y)$ satisfies Axiom 5. In fact, $\sqrt{s(X;Y)}$ satisfies all desired axioms for potential correlation, and 
we make this precise in the following theorem whose proof is provided in \Appendix~\ref{sec:axiom_proof}.
\begin{Theorem}
	\label{thm:axiom}
	\textnormal{
	Hypercontractivity coefficient $\sqrt{s(X;Y)}$ satisfies Axioms 1-6.
	}
\end{Theorem}
In particular, the hypercontractivity coefficient satisfies Axiom 6 for potential correlation,   
unlike other measures of correlation
(see Theorem~\ref{prop-max} 
 for examples).
If there is a potential for $X$ in a possibly rare regime in $\Xcal$
to be fully correlated with $Y$ such that $Y=f(X)$, then the hypercontractivity coefficient is maximum:
$s(X;Y)=1$. In the following section, we show that existing correlation measures, on the other hand, violate the proposed axioms. 

\subsection{Standard correlation coefficients violate the Axioms.}\label{sec2.3}

We analyze
 existing  measures of correlations under the scenario with potential correlation (Axiom 6), where we find 
that none of the existing correlation measures satisfy Axiom 6. 
Suppose $X$ and $Y$ are independent (i.e.~no correlation) in
a subset $\Xcal_d$ of the domain $\Xcal$,
  and allow $X$ and $Y$ to be arbitrarily correlated in the rest $\Xcal_r$ of the domain, such that
$\Xcal=\Xcal_d\cup \Xcal_r$.
We further assume that the independent part  is dominant and
the correlated part  is rare;
let $\alpha := \P(X\in \Xcal_r) $ and we consider the scenario when $\alpha$ is small.
A good measure of potential correlation is expected to capture the correlation in $\Xcal_r$
even if it is rare (i.e., $\alpha$ is small).
To make this task more challenging, we assume that 
 the conditional distribution of $Y|\{X \in \Xcal_r\}$ is the
 same as  $Y|\{X \notin \Xcal_r\}$.
Figure~\ref{fig:axiom} (of this section)  illustrates
sampled points for two examples from such a scenario
and more examples  are in
 Figure~\ref{sample-points}. 
Our main result is the analysis of HGR maximal correlation (mCor) \cite{Hirschfeld1935},
distance correlation (dCor) \cite{Szekely--Rizzo--Bakirov2007}, maximal information coefficients (MIC) \cite{ReshefEtAl2011},
which shows that these measures are vanishing with $\alpha$ even if the dependence 
in the rare regime is very high.  Suppose $Y|(X\in\Xcal_r)=f(X)$, then
all three correlation coefficients are vanishing as $\alpha$ gets small.
This in particular  violates Axiom 6.
The reason is that standard correlation coefficients  measure the {\em average correlation} whereas the hypercontractivity coefficient measures
the {\em potential correlation}.
The experimental comparisons on the power of these measures
confirm our analytical predictions
in Figure~\ref{fig:005} in Section~\ref{sec:exp}. 
The formal statement is below and the proof  is provided in \Appendix~\ref{sec:correlation_proof}.

\begin{Theorem}\label{prop-max}
\textnormal{Consider a pair of continuous random variables $(X,Y) \in \Xcal \times \Ycal$. Suppose $\Xcal$ is partitioned as  $\Xcal_r  \cup \Xcal_d  =\Xcal $
such that $P_{Y|X}( S|X\in \Xcal_r) =  P_{Y|X}( S|X\in \Xcal_d)$ for all $S\subseteq \Ycal$, 
and $Y$ is independent of $X$ for $X \in \Xcal_d$. Let 
$\alpha = \P\{X \in \Xcal_r\}$. The HGR maximal correlation coefficient is
\begin{align}
	{\rm mCor}(X,Y) \; \; = \;\;  \sqrt{\a} \;{\rm mCor}(X_r,Y) \;,
	\label{eq:mcor_bound}
\end{align}
the distance correlation coefficient is
\begin{align}
	\mathrm{dCor}(X,Y) \;\;=\;\; \alpha \; \mathrm{dCor}(X_r,Y) \;,
	\label{eq:dcor_bound}
\end{align}
the maximal information coefficient is upper bounded by
\begin{align}
	\mathrm{MIC}(X,Y) \;\; \le\;\;  \alpha \; \mathrm{MIC}(X_r,Y) \;,
	\label{eq:mic_bound}
\end{align}
  where $X_r$ is the random variable $X$ conditioned on the rare domain $X\in \Xcal_r$.
  }
  \end{Theorem}

Under the
rare/dominant scenario considered in Theorem~\ref{prop-max}, $s(X;Y) \ge \mathrm{mCor}^2(X;Y)$. It is well known that this inequality holds for any $X$ and $Y$~\cite{ahlswede1976spreading}. In particular, Theorem 3 in~\cite{Makur--Zheng2015} 
 shows that hypercontractivity coefficient is a natural extension of the popular HGR maximal correlation coefficient as follows.
 \begin{Remark}[Connection between $s(X;Y)$ and mCor$(X,Y)$~\cite{Makur--Zheng2015}]\label{rem:connect}
\textnormal{
The squared HGR maximal correlation is a special case of the hypercontractivity optimization in Equation \eqref{defKL-dicrete} restricted to
searching over a distribution $r(x)$ in a close neighborhood of $p(x)$.
}\end{Remark}

As $s(X;Y)$ searches over a larger space, it is always larger than or equal to mCor$^2(X;Y)$.
This gives an intuitive justification for using $s(X;Y)$ as a measure of potential influence;
we allow search over larger space, but properly normalized by the KL divergence, in a hope to find a potential distribution $r(x)$ that can influence $Y$ significantly. While hypercontractivity coefficient is a natural  extension of HGR maximal correlation coefficient, there is an important difference between hypercontractivity coefficient and HGR maximal correlation coefficient (and other correlation measures); hypercontractivity is directional. 
\begin{Remark}[Asymmetry of $s(X;Y)$]\textnormal{
Hypercontractivity coefficient is \emph{asymmetric} in $X$ and $Y$ while HGR maximal correlation, distance correlation, and MIC are \emph{symmetric}.}
\end{Remark}


Under the
rare/dominant scenario considered in Theorem~\ref{prop-max}, 
the hypercontracitivy coefficient $s(X;Y)$ is large because 
it 
measures the potential correlation from $X$ to $Y$.
On the other hand, inverse hypercontractivity coefficient $s(Y;X)$, which measures the potential correlation from $Y$ to $X$, 
is small
as there is no apparent
potential correlation from $Y$ to $X$.
This is made precise in the following proposition. 

\begin{Proposition}\label{YX-domrare}\textnormal{
Under the hypotheses of Theorem~\ref{prop-max},
 the hypercontractivity coefficient from $Y$ to $X$ is
\[
s(Y;X) = \a \; s(Y;X_r),
\]
where $X_r$ is the random variable $X$ conditioned on the rare domain $X\in \Xcal_r$. 
}
\end{Proposition}
Proof is provided in \Appendix~\ref{sec:inverse_proof}.

\subsection{Mutual Information violates the Axioms} \label{sec2.4}
Beside standard correlation measures, another measure widely used to quantify the strength of dependence is mutual information. 
We can show that 
mutual information satisfies Axiom 6 if we replace 1 by $\infty$. 
However there are two key problems: 
\begin{itemize}
\item 
Practically, mutual information is \emph{unnormalized}, i.e., $I(X;Y) \in [0,\infty)$. Hence, it provides no absolute indication of the strength of the dependence.
\item 
Mathematically, we are looking for a quantity that \emph{tensorizes}, i.e., doesn't change when there are many i.i.d. copies of the same pair of random variables. 
\begin{Remark}[Tensorization property of $s(X;Y)$~\cite{Witsenhausen}]\textnormal{Hypercontractivity coefficient tensorizes, i.e,
\[
 s(X_1, ..., X_n; Y_1,..,Y_n) = s(X_1,Y_1), \  \text{ for i.i.d.\ } (X_i,Y_i), \ \ i=1,\cdots,n.
 \]}\end{Remark}
 On the other hand, mutual information is additive, i.e., 
\[
I(X_1, \cdots, X_n; Y_1, \cdots, Y_n) = n I(X_1;Y_1), \  \text{ for i.i.d.\ } (X_i,Y_i), \ \ i=1,\cdots,n.
\]
 Tensorizing quantities capture the strongest relationship among independent copies while additive quantities capture the sum. For instance, mutual information could be large because a small amount of information accumulates over many of the independent components of $X$ and $Y$ (when $X$ and $Y$ are high dimensional) while tensorizing quantities would rule out this scenario, where there is no strong dependence. When the components are not independent, hypercontractivity indeed pools information from different components to find the strongest direction of dependence, which is a desirable property.
\end{itemize}

 One natural way to normalize mutual information is by the log of the cardinality of the input/output alphabets~\cite{bell1962mutual}. One can interpret a popular correlation measure MIC as a similar effort for normalizing mutual information and is one of our baselines.

Given that other correlation measures and mutual information do not satisfy our axioms, a natural  question to ask is whether hypercontractivity is a unique solution that satisfies all the proposed axioms. In the following, we show that  the hypercontractivity coefficient is not the only one satisfying all the proposed axioms -- just as HGR correlation is not the only measure satisfying R\'enyi's original axioms.


\subsection{Hypercontractivity ribbon}\label{sec2.5} 

We 
show that a family of measures known as {\em hypercontractivity ribbon}, which includes hypercontractivity coefficient as a special case, satisfies all the axioms. The hypercontractivity ribbon~\cite{ahlswede1976spreading, nair2014equivalent} is a class of measures parametrized by $\alpha \ge 1$ as
\begin{eqnarray}
r_\alpha(X;Y) = \sup_{r(x,y) \neq p(x,y)} \frac{D(r(y)\|p(y))}{D(r(x)\|p(x)) + \alpha D(r(y|x)\|p(y|x))},
	\label{ribbon0}
\end{eqnarray}
where $D(r(x)||p(x))$ denotes the KL divergence of $r(x)$ and $p(x)$. 
 An alternative characterization of hypercontractivity ribbon in terms of mutual information is provided in~\cite{nair2014equivalent,nair2014gaussian};
 \begin{align}\label{ribbon-mi}
r_\alpha(X;Y)= \sup_{p(u|x,y)} \frac{I(U;Y)}{ I(U;X) + \alpha I(U;Y|X)}.
 \end{align}
from which we can see that hypercontractivity coefficient is a special case of hypercontractivity ribbon~\cite{Anantharam--Gohari--Kamath--Nair2013}:
\begin{align*}
s(X;Y) &= \lim_{\alpha \to \infty} r_\alpha(X;Y) = \lim_{\alpha \to \infty} s_\alpha(X;Y).
\end{align*}

\begin{Proposition}\textnormal{
The (re-parameterized) hypercontractivity ribbon $s_\alpha(X;Y) := (\alpha\,  r_\alpha(X;Y) - 1)/(\alpha\,-1)$, for $\alpha > 1$, satisfies Axioms 1-6.
}
\end{Proposition}

\begin{proof}
By definition, $s_\alpha(X;Y)$ is defined for any pair of non-constant random variables (Axiom 1) and is between 0 and 1 by data processing inequality (Axiom2). We can show that $s_\alpha(X;Y)$ satisfies Axioms 3 and 4, in a similar way to show $s(X;Y)$ satisfies Axioms 3 and 4.  Also, $s_\alpha(X;Y) = \rho^2$ for a jointly Gaussian $X$,$Y$ with Pearson correlation $\rho$~\cite{nair2014gaussian} (Axiom 5). Finally, $s_\alpha(X;Y)$ satisfies Axiom 6 because $r_\alpha(X;Y)$ is non-increasing in $\a$, which implies that 
$s_\alpha(X;Y) = r_\alpha(X;Y) = 1$ if $s(X;Y) = 1$.
\end{proof}



Although hypercontracitivy ribbon satisfies all axioms, a few properties of the hypercontractivity coefficient make it
 more attractive than hypercontractivity ribbon for practical use; hypercontractivity coefficient can be efficiently estimated from samples (see Section~\ref{sec:estimate}). Hypercontractivity coefficient is a natural extension of the popular HGR maximal correlation coefficient (Remark~\ref{rem:connect}).  

\subsection{Multidimensional $X$ and $Y$}\label{sec2.6}
\label{sec:multi}

In this section, we discuss potential correlation of multidimensional $X$ and $Y$. 
While most of the correlation coefficients, including  the hypercontractivity coefficient, are
well-defined for multi-dimensional $X$ and $Y$, the axioms are specific to univariate $X$ and $Y$. To bridge this gap, we propose replacing Axiom 5, as this is the only axiom specific to univariate random variables.

\begin{enumerate}
	\item[]Axiom 5'.  If $(X,Y) \sim \mathcal{N}\left(\mu,
		\Sigma=\begin{bmatrix} \Sigma_X & \Sigma_{XY}\\
		\Sigma_{YX} & \Sigma_Y
		\end{bmatrix}\right)$, then $\rho^*(X,Y) = \| \Sigma_X^{-1/2} \Sigma_{XY}\Sigma_Y^{-1/2} \|$, where $\|\cdot\| $ is the spectral norm of a matrix.
\end{enumerate}
This recovers the original Axiom 5 when restricted to univariate $X$ and $Y$. 
This naturally generalizes both R\'enyi's axioms and the proposed potential correlation axioms to multidimensional  $X$ and $Y$.

\begin{Proposition}
	\label{propo:multi}\textnormal{
	Axiom 5', together with original R\'enyi's Axioms 1-4, 6', and 7', recovers
	maximal correlation (mCor)
	as a measure satisfying all Axioms even in this multi-dimensional case.
	Axiom 5', together with our proposed Axioms 1-4, and 6, recovers the  hypercontractivity coefficient 	$\sqrt{s(X;Y)}$  
	as a measure satisfying all axioms.
}\end{Proposition}

The second statement in the proposition follows from the analyses of the hypercontractivity coefficient of Gaussian distributions in \cite{Chechik--Amir--Naftali--Yair2005}. A formal proof is provided in \Appendix~\ref{sec:multi-proof}.

%

\subsection{Noisy, discrete, noisy and discrete potential correlations}\label{sec2.7}

In this section, we consider more general scenarios of potential correlation than the one in Axiom 6. We consider (i) noisy potential correlation where $Y = f(X) + Z$ for a Gaussian noise $Z$ for $(X,Y) \in \Xcal_r \times \Ycal$, (ii) discrete potential correlation, where $\Xcal_r = \{1,\cdots,k\}$, and (iii) noisy discrete potential correlation -- a random corruption model. 
For these three examples, we obtain a lower bound on $s(X;Y)$. 

\begin{Example}\label{example-noisy}
\textnormal{Suppose that for a pair of random variables $(X,Y) \in \Xcal \times \Ycal$, there exists a subset $\Xcal_r \subseteq \Xcal$ for which $\P\{X \in \Xcal_r\} = \a$ $(\a > 0)$, and for $(X,Y) \in \Xcal_r \times \Ycal$,  $(X,Y) \sim \mathcal{N}(\mathbf{0}, \Sigma)$, where 
\[
\Sigma = \begin{bmatrix}
1 &\rho \\
\rho &1
\end{bmatrix}.
\]
Then
\begin{align}\label{ex-noisy}
s(X;Y) \ge \frac{\log\frac{1}{1-\rho^2} + \log\frac{1}{1+\rho^2} }{\log\frac{1}{1-\rho^2} + \frac{H(\a)}{\a}}
\end{align}
Proof is in \Appendix~\ref{append-nr}.  
}\end{Example}

We now consider for discrete $(X,Y)$. We start with the case for which $X$ and $Y$ are perfectly correlated for $(X,Y) \in \Xcal_r \times \Ycal$.

\begin{Example}\label{LB-div}\textnormal{
Suppose that for a pair of discrete random variables $(X,Y) \in \Xcal \times \Ycal$, there exists a subset $\Xcal_r = \{1,2,\cdots,k\} \subseteq \Xcal$ for which $\P\{X \in \Xcal_r\} = \a$ $(\a > 0)$, and $X|\{X \in \Xcal_r\} \sim \textnormal{Unif}[1:k]$ and $Y = X$ for $X \in \Xcal_r$. Then,
\[
s(X;Y) \ge \frac{\log k}{\log k + \log (1/\alpha)}.
\]
The inequality holds by considering $r(x) = \mathbb{I}_{\{X = 1\}}$ in~\eqref{defKL-dicrete}.
}\end{Example}

We conjecture this lower bound is indeed tight for $\a \le 0.5$ based on numerical simulations. From this lower-bound, we can see the trade-off between $k$ and $\a$. As $k \to \infty$, the lower bounds approaches to 1. As $\a \to 1$, the lower bound approaches to 1. As $\a  \to 0$, the lower bound approaches to 0. In the following, we consider the case where $X$ and $Y$ are not perfectly correlated in $(\Xcal_r \times \Ycal)$ for discrete $(X,Y)$.  In particular, we consider a random corruption model for $(\Xcal_r \times \Ycal)$ and obtain a lower bound on $s(X;Y)$.

\begin{Example} \label{example-discrete-noisy}
\textnormal{
Suppose that for a pair of random variables $(X,Y) \in \Xcal \times \Ycal$, there exists a subset $\Xcal_r \subseteq \Xcal$ for which $\P\{X \in \Xcal_r\} = \a$ $(\a > 0)$, and for  $(X,Y) \in \Xcal_r \times \Ycal$,
\[
Y = \begin{cases}
X &\text{w.p.\ } 1-\frac{k}{k-1}\epsilon,\\
\text{Unif}[1:k] &\text{w.p.\ } \frac{1}{k-1}\epsilon
\end{cases}.
\]
Then
\begin{align}\label{noise-s}
s(X;Y) \ge \frac{(1-\e)\log k(1-\e) + \e \log k\e/(k-1)}{\log(k/\a)} = \frac{\log k -H_2(\e) - \e \log (k-1)}{\log(k/\a)}.
\end{align}
On the other hand,
\begin{align}\label{noise-rho}
\mathrm{mCor}^2(X;Y) = \a \left(1-\frac{k}{k-1}\e\right)^2, \ \ 0 \le \e \le \frac{k-1}{k}.
\end{align}
Proof is in \Appendix~\ref{append-ndr}.
}\end{Example}

In Figure~\ref{ex123}, we show plots of lower bounds on $s(X;Y)$ and mCor$(X;Y)$ in Examples~\ref{example-noisy}-\ref{example-discrete-noisy}; from these figures, we can see that $s(X;Y)$ increases as $\rho \to 1$ and $k \to \infty$. In comparison, mCor$(X;Y)$ remains small. 

\begin{figure}[!ht] 
\begin{center}
 \includegraphics[width=0.31\linewidth]{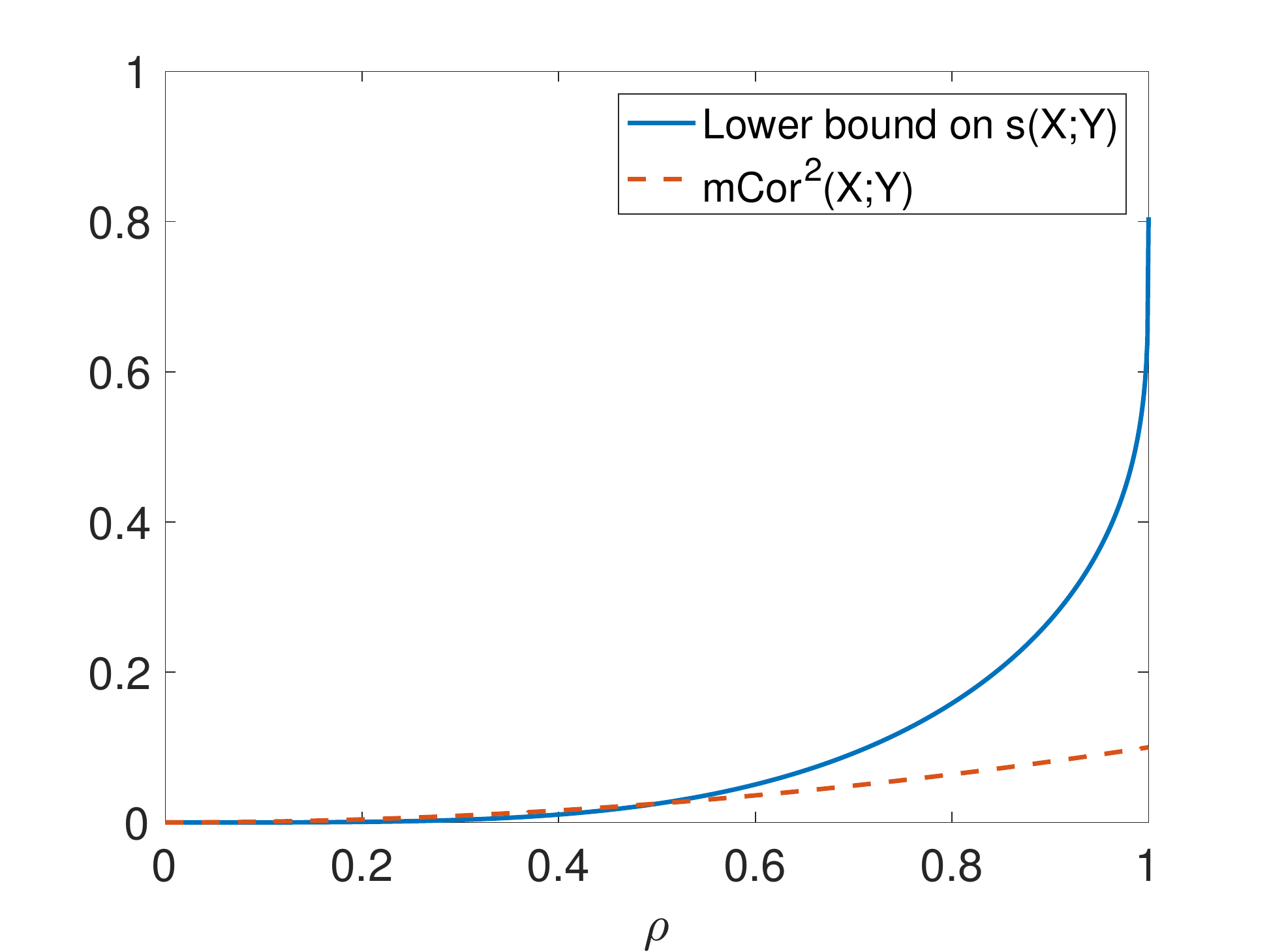} 
 \includegraphics[width=0.31\linewidth]{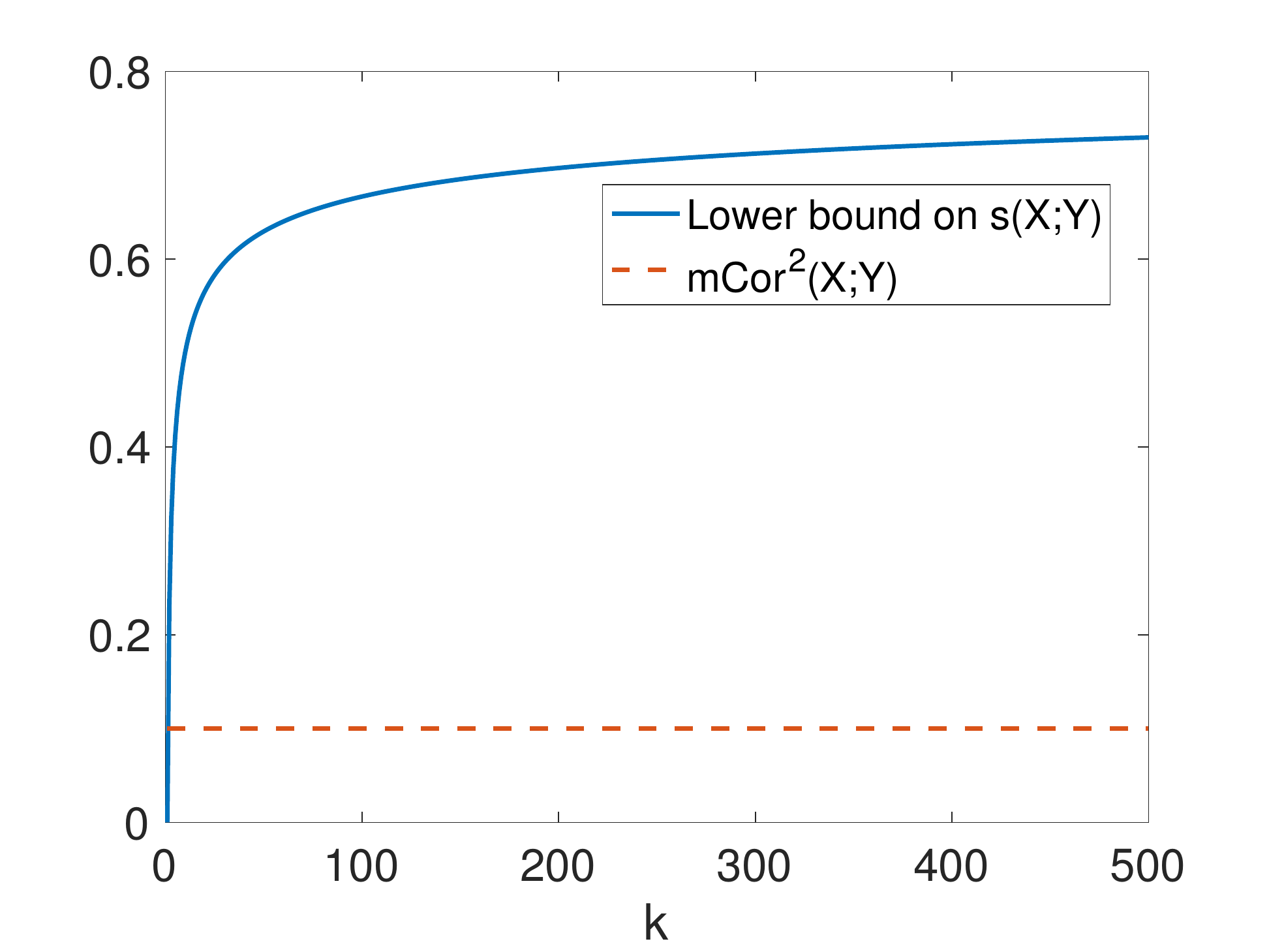} 
 \includegraphics[width=0.31\linewidth]{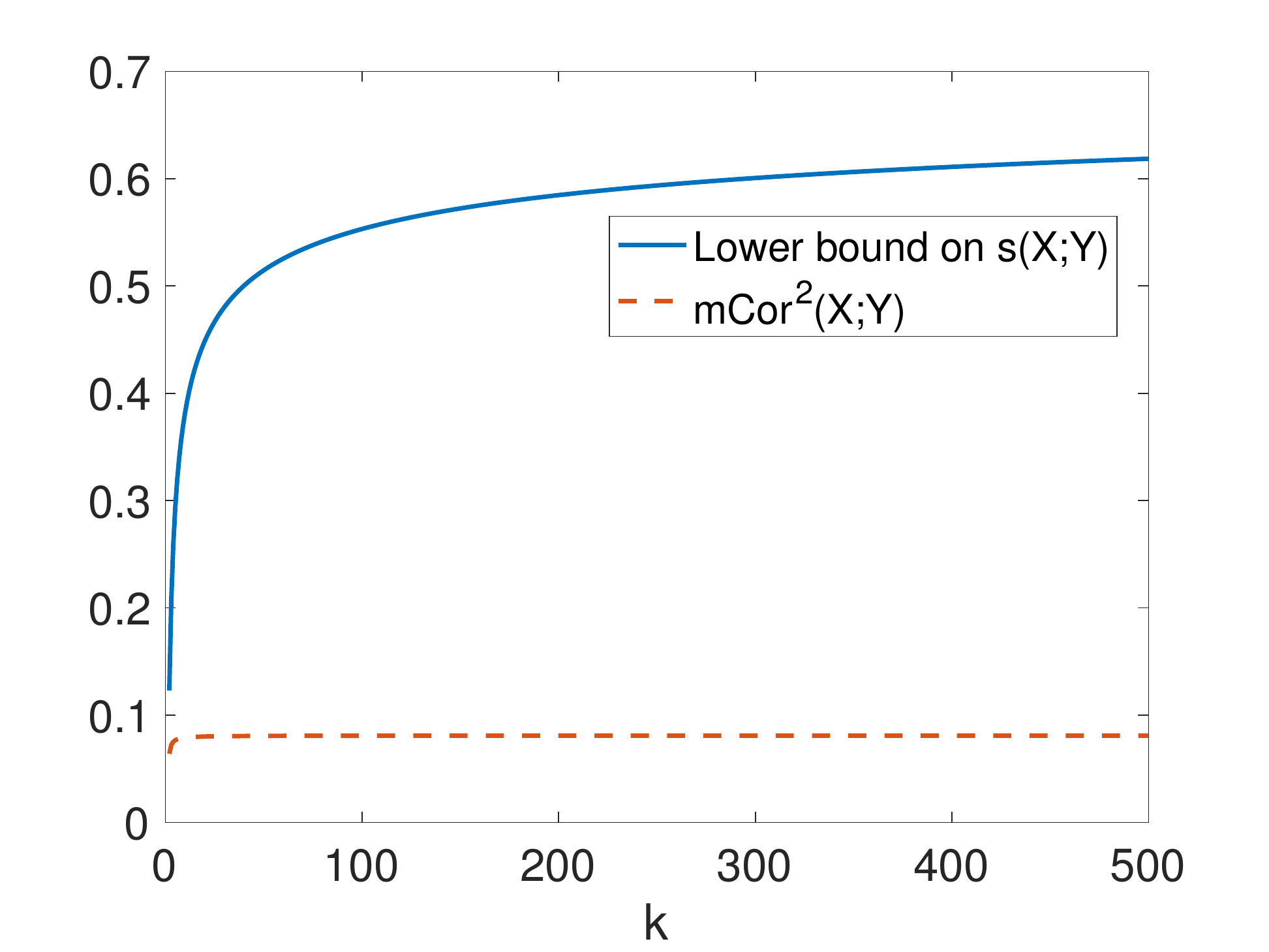} 
 \end{center}
 \caption{Lower bound on $s(X;Y)$ and mCor$(X;Y)$ for $\a = 0.1$ in (Left) Example 1 (Middle) Example 2 (Right) Example 3 for $\epsilon=0.1$.}\label{ex123}
\end{figure}

\section{Estimator of the hypercontractivity coefficient from samples}
\label{sec:estimate}

In this section, we present an algorithm\footnote{Code available at \texttt{https://github.com/wgao9/hypercontractivity}} to compute the hypercontractivity coefficient $s(X;Y)$ from i.i.d. samples $\{X_i,Y_i\}_{i=1}^n$. The computation of the hypercontractivity coefficient from samples is known to be challenging for continuous random variables~\cite{deepVIB, dropout2016}, and to the best of our knowledge, there is no known efficient algorithm to compute the hypercontractivity coefficient from samples. Our estimator is the first efficient algorithm to compute the hypercontractivity coefficient, based on the following equivalent definition of the hypercontractivity coefficient, shown recently in~\cite{Anantharam--Gohari--Kamath--Nair2013}:
\begin{eqnarray}
\label{defKL-dicrete}
    s(X;Y) &\equiv& \sup_{r_x \neq p_x} \frac{D(r_y||p_y)}{D(r_x||p_x)} \;.
\end{eqnarray}

There are two main challenges for computing $s(X;Y)$. The first challenge is -- given a marginal distribution $r_x$ and samples from $p_{xy}$, how do we estimate the KL divergences $D(r_y||p_y)$ and $D(r_x||p_x)$. The second challenge is the optimization over the infinite dimensional simplex. We need to combine estimation and optimization together in order to compute $s(X;Y)$. Our approach is to combine ideas from traditional kernel density estimates and from importance sampling. Let $w_i = r_x(X_i)/p_x(X_i)$ be the {\em likelihood ratio} evaluated at sample $i$. We propose the estimation and optimization  be solved jointly as follows:

{\bf Estimation:} To estimate KL divergence $D(r_x||p_x)$, notice that
\begin{eqnarray*}
    D(r_x||p_x) &=& \mathbb{E}_{X \sim p_x} \left[ \frac{r_x(X)}{p_x(X)} \log \frac{r_x(X)}{p_x(X)} \right].
\end{eqnarray*}
Using empirical average to replace the expectation over $p_x$, we propose
\begin{eqnarray*}
    \widehat{D}(r_x||p_x) = \frac{1}{n} \sum_{i=1}^n \frac{r_x(X_i)}{p_x(X_i)} \log \frac{r_x(X_i)}{p_x(X_i)} = \frac{1}{n} \sum_{i=1}^n w_i \log w_i \;.
\end{eqnarray*}
For $D(r_y||p_y)$, we follow the similar idea, but the challenge is in computing $v_j = r_y(Y_j)/p_y(Y_j)$. To do this, notice that $r_{xy} = r_x p_{y|x}$, so
\begin{eqnarray*}
    r_y(Y_j) &=& \mathbb{E}_{X \sim r_x} \left[\, p_{y|x}(Y_j|X) \,\right] = \mathbb{E}_{X \sim p_x} \left[\, p_{y|x}(Y_j|X) \frac{r_x(X)}{p_x(X)} \,\right] \;.
\end{eqnarray*}
Replacing the expectation by empirical average again, we get the following estimator of $v_j$:
\begin{eqnarray*}
    \widehat{v}_j &=& \frac{1}{n} \sum_{i=1}^n \frac{p_{y|x}(Y_j|X_i)}{p_y(Y_j)} \frac{r_x(X_i)}{p_x(X_i)} = \frac{1}{n} \sum_{i=1}^n \underbrace{\frac{p_{xy}(X_i,Y_j)}{p_x(X_i)p_y(Y_j)}}_{A_{ji}} w_i \;.
\end{eqnarray*}
We can write this expression in matrix form as $\widehat{\bf v} = {\bf A}^T {\bf w}$. We use a kernel density estimator from~\cite{NCCA} to estimate the matrix ${\bf A}$, but our approach is compatible with any density estimator of choice.

{\bf Optimization:} Given the estimators of the KL divergences, we are able to convert the problem of computing $s(X;Y)$ into an optimization problem over the vector ${\bf w}$. Here a constraint of $(1/n)\sum_{i=1}^n w_i = 1$ is needed to satisfy $\mathbb{E}_{p_x}[r_x/p_x] = 1$. To improve numerical stability, we use $\log s(X;Y)$ as the objective function. Then the optimization problem has the following form:
\begin{eqnarray*}
    {\rm max}_{{\bf w}} &\,& \log \left(\, ( {\bf w}^T {\bf A} \log ({\bf A}^T {\bf w}) \,\right) - \log \left(\, {\bf w}^T \log {\bf w}\,\right) \,\notag\\
    {\rm subject\,to} &\,& \frac{1}{n}\sum_{i=1}^n w_i = 1 \,\notag\\
                      &\,& w_i \geq 0, \forall \, i
\end{eqnarray*}
where ${\bf w}^T \log {\bf w} = \sum_{i=1}^n w_i \log w_i$ for short. Although this problem is not convex, we apply gradient descent to maximize the objective. In practice, we initialize $w_i = 1 + \mathcal{N}(0, \sigma^2)$ for $\sigma^2 = 0.01$. Hence, the initial $r_x$ is perturbed mildly from $p_x$. Although we are not guaranteed to achieve the global maximum, we consistently observe in extensive numerical experiments that we have 50\%-60\% probability of achieving the same maximum value, which we believed to be the global maximum. A theoretical  analysis of the landscape of local and global optima and their regions of attraction with respect to gradient descent  is an interesting and challenging open question, outside the scope of this paper.

\subsection{Consistency of Estimation}  While a theoretical understanding of the performance of gradient descent on  the optimization step (where the number of samples is fixed) above is technically very challenging, we can study the performance of the solution as the number of samples increases. In particular we show below (under suitable simplifying assumptions to get to the essence of the proof) that the optimal solution to the finite sample optimization problem is {\em consistent}.  Suppose that $\mathcal{X}$ is discrete. Further we restrict the optimization over a quantized and bounded set $T_{\Delta}$, where ${\bf w} \in T_{\Delta}$ is quantized by a gap $\Delta$ and satisfies: (1) $C_1 \leq w_i \leq C_2$ for all $i$; (2) $(1/n)\sum_{i=1}^n w_i \log w_i > C_0$. We further assume that we have access of ${\bf A} = P_{xy}(X_i,Y_j)/P_x(X_i)P_y(Y_j)$. Define $\widehat{s}_{\Delta}(X;Y) = \max_{{\bf w} \in T_{\Delta}} {\bf w}^T {\bf A} \log ({\bf A}^T {\bf w})/{\bf w}^T \log {\bf w}$, then with two further simplifying  conditions on the joint distribution (formally stated in \Appendix~\ref{sec:proof-consistency}), we can prove consistency of our estimation procedure: 

\begin{Theorem}
	\label{thm:convergence}
	As $n$ goes to infinity, $\widehat{s}_{\Delta}(X;Y)$ converges to $s(X;Y)$ up to a resolution of quantization in probability, i.e., for any $\varepsilon >0$, $\Delta > 0$ and $s(\Delta) = O(\Delta)$, we have
\begin{eqnarray}
    \lim_{n \to \infty} \Pr \left(\, |\, \widehat{s}_{\Delta}(X;Y) - s(X;Y) \,| > \varepsilon +s(\Delta)\,\right) = 0 \;.
\end{eqnarray}
\end{Theorem}

\section{Experimental results}\label{sec:exp}

We present experimental results on synthetic and real datasets showing
that the hypercontractivity coefficient $(a)$  is more powerful in detecting potential correlation compared to existing measures;
$(b)$ discovers hidden potential correlations among various national indicators in WHO datasets; and
$(c)$ is more robust in discovering pathways of gene interactions from gene expression time series data.

\subsection{Synthetic data: power test on potential correlation }
\label{sec:power}

As our estimator (and the measure itself) involves a maximization, it is possible that we
are sensitive to outliers and may capture spurious noise. Via a series of experiments we show that the hypercontractivity coefficient and our estimator are capturing the true potential correlation.
As shown in Figure~\ref{sample-points}, we generate pairs of datasets
-- one where $X$ and $Y$ are independent and
one where there is a potential correlation  as per our scenario. 
We run experiment with eight types of functional associations, following the examples
from \cite{ReshefEtAl2011,SimonTibshirani,GHH12}.
For the correlated datasets,
out of $n$ samples  $\{(x_i,y_i)\}_{i=1}^n$,
$\a n$ rare but correlated samples are in $\Xcal=[0,1]$ and
$(1-\a) n$ dominant but independent samples are in $\Xcal\in[1,1.1]$.
The rare but correlated  samples are generated as $x_i \sim \text{Unif}[0,1], y_i \sim f(x_i) + \mathcal{N}(0,\sigma^2)$ for $i \in [1:\a n]$. The dominant samples are generated as  $x_i \sim \text{Unif}[1,1.1], y_i \sim f(\text{Unif}[0,1]) + \mathcal{N}(0,\sigma^2)$ for $i \in [\a n+1, n]$.

Table~\ref{comp-table} shows the  hypercontractivity coefficient and the other correlation coefficients for correlated and independent datasets shown in Figure~\ref{sample-points}, along with the chosen value of $\a$ and $\sigma^2$. 
Correlation estimates with the largest separation for each row is shown in bold.
The hypercontractivity coefficient gives the largest separation between the correlated and the independent dataset for most functional types.

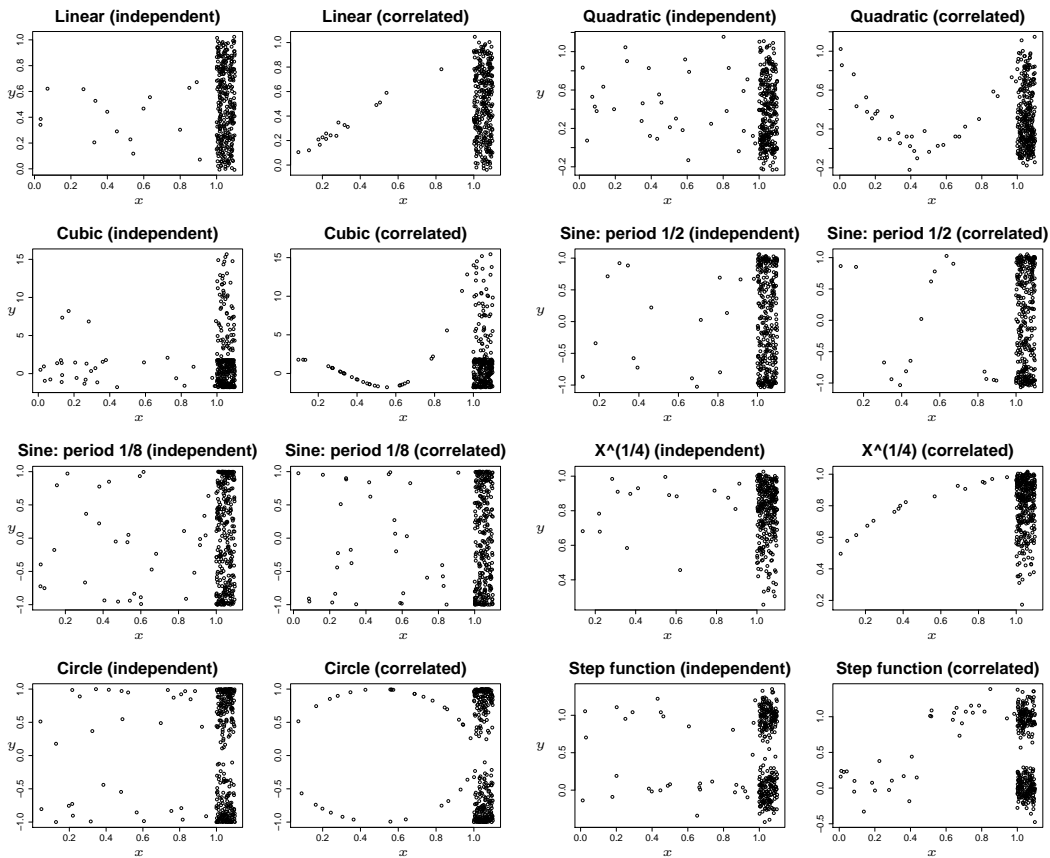
\begin{figure}[H]
\begin{center}
 \includegraphics[scale=0.45]{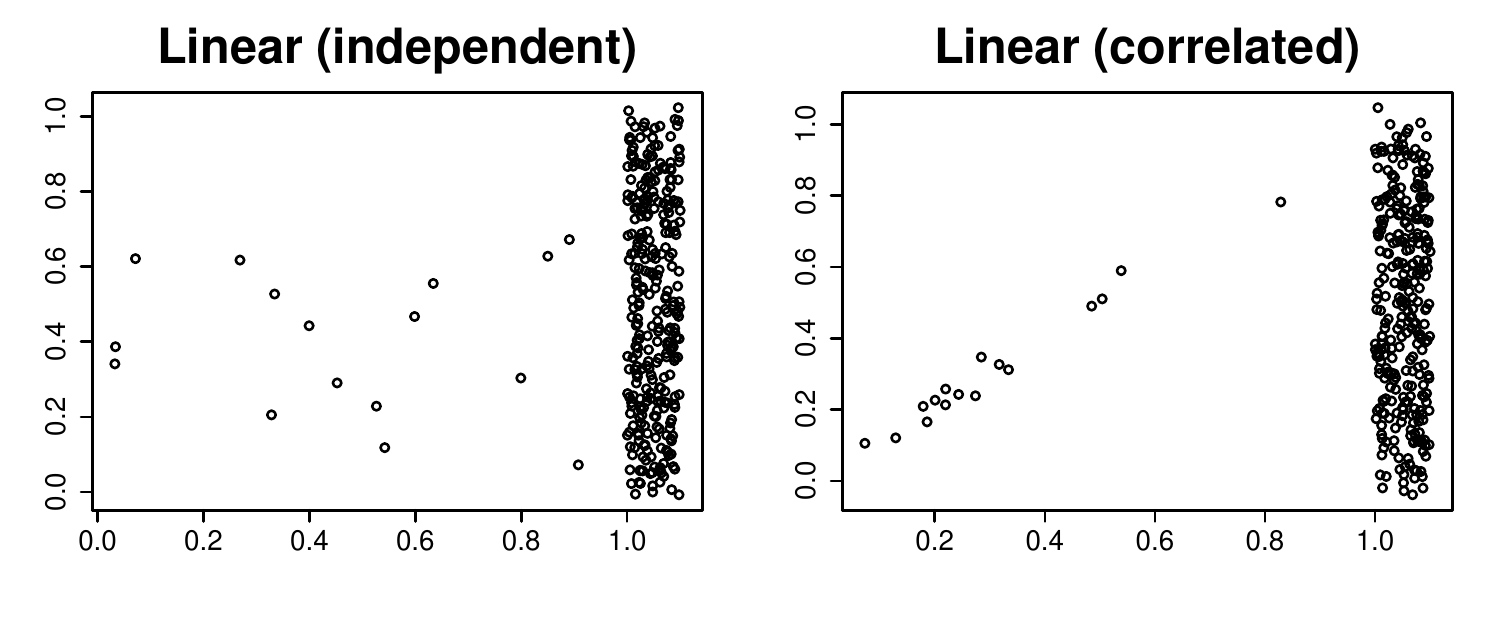}\ \ \ 
 \includegraphics[scale=0.45]{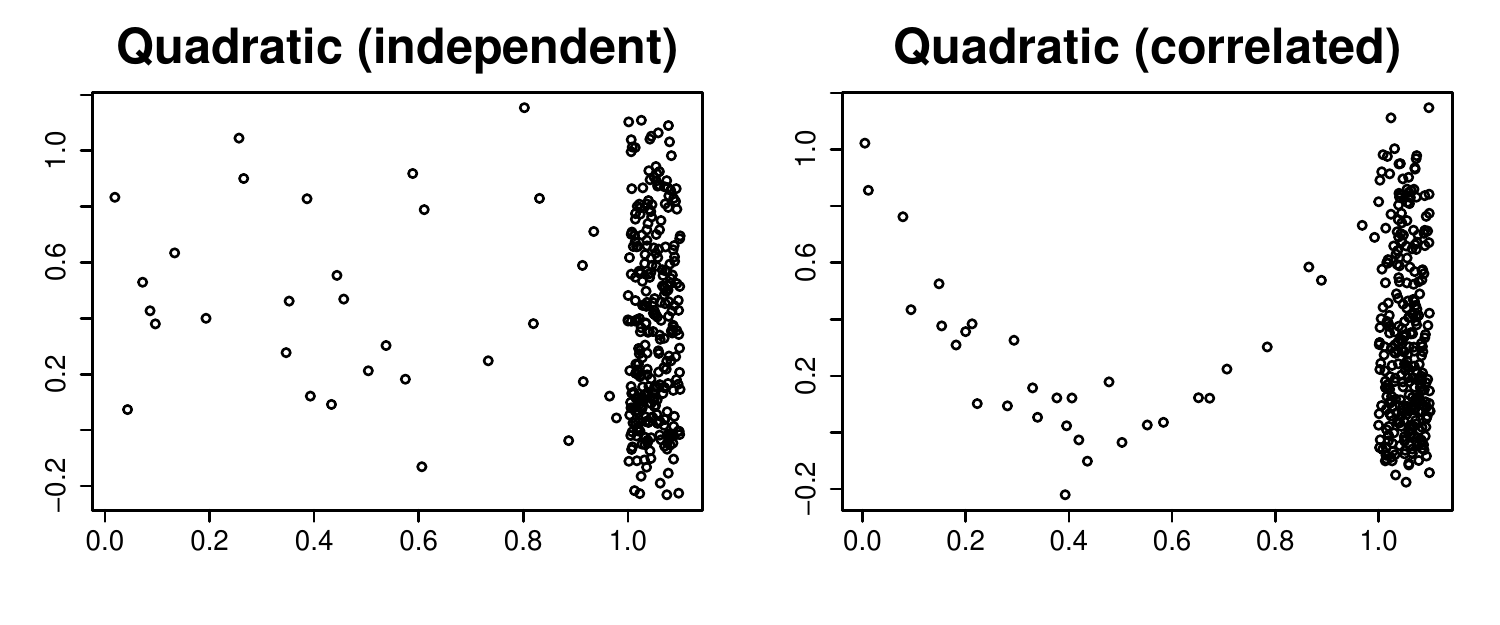}
 \put(-50,4){\tiny $x$}
\put(-150,4){\tiny $x$}
\put(-250,4){\tiny $x$}
\put(-350,4){\tiny $x$}
\put(-198,45){\tiny $y$}
\put(-398,45){\tiny $y$}\\
 \includegraphics[scale=0.45]{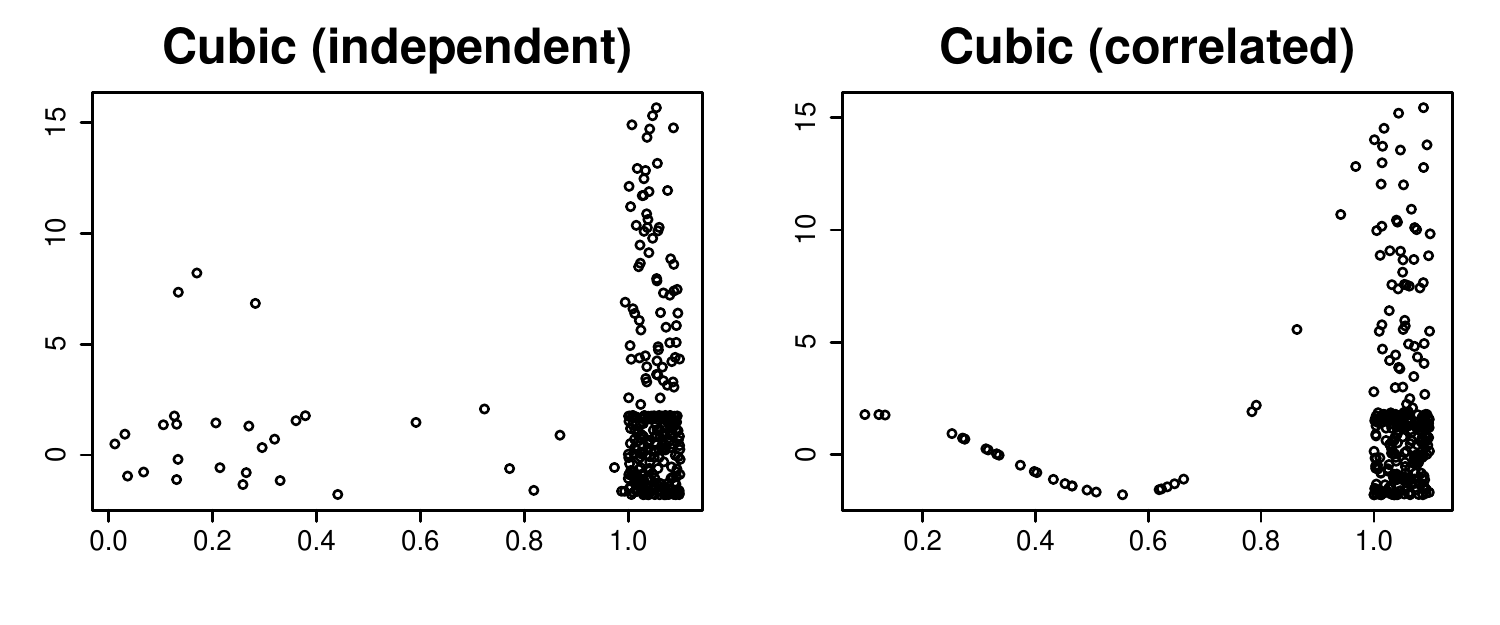}\ \ \ \includegraphics[scale=0.45]{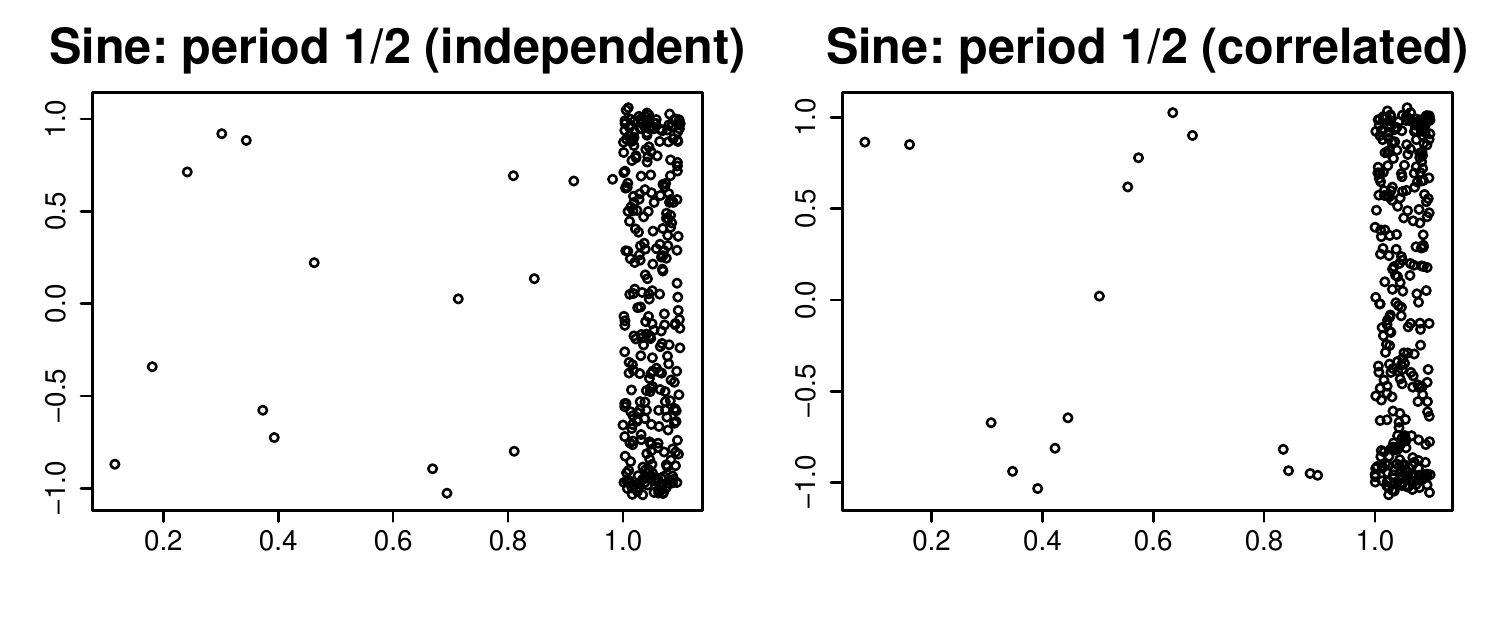}
 \put(-50,4){\tiny $x$}
\put(-150,4){\tiny $x$}
\put(-250,4){\tiny $x$}
\put(-350,4){\tiny $x$}
\put(-198,45){\tiny $y$}
\put(-398,45){\tiny $y$}\\
\includegraphics[scale=0.45]{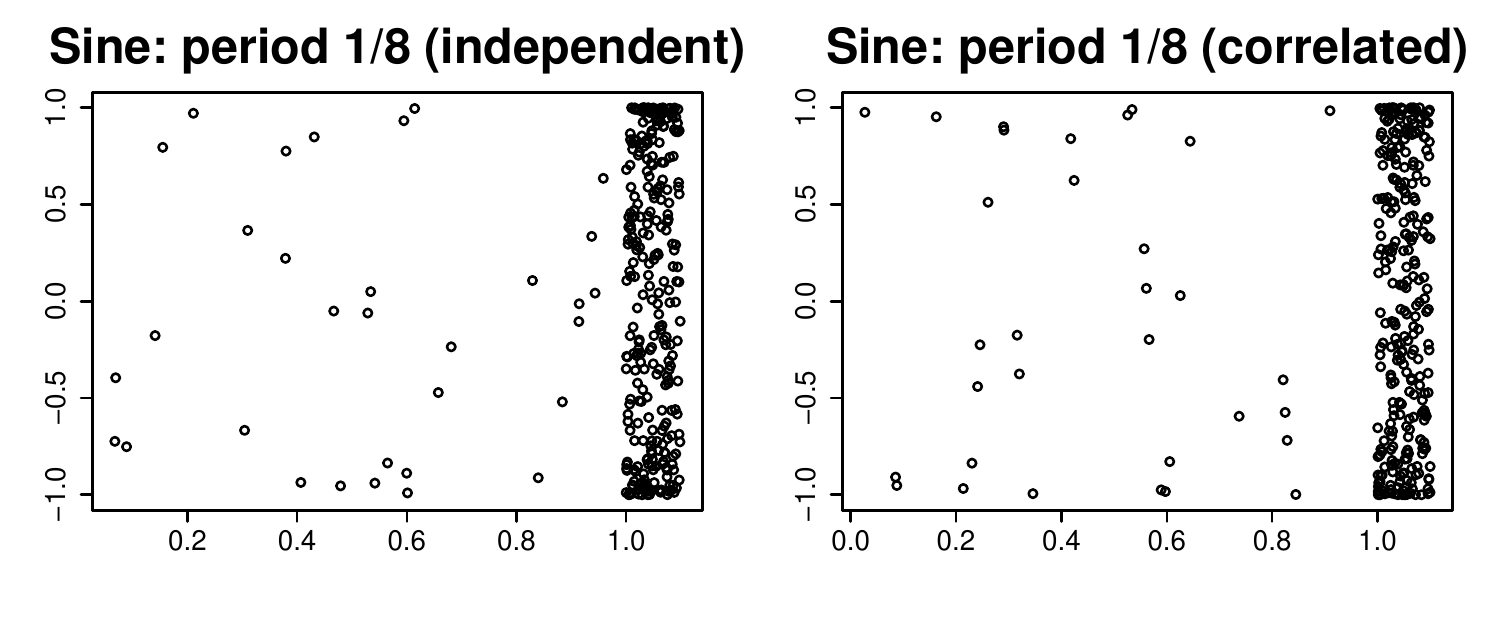}\ \ \ \includegraphics[scale=0.45]{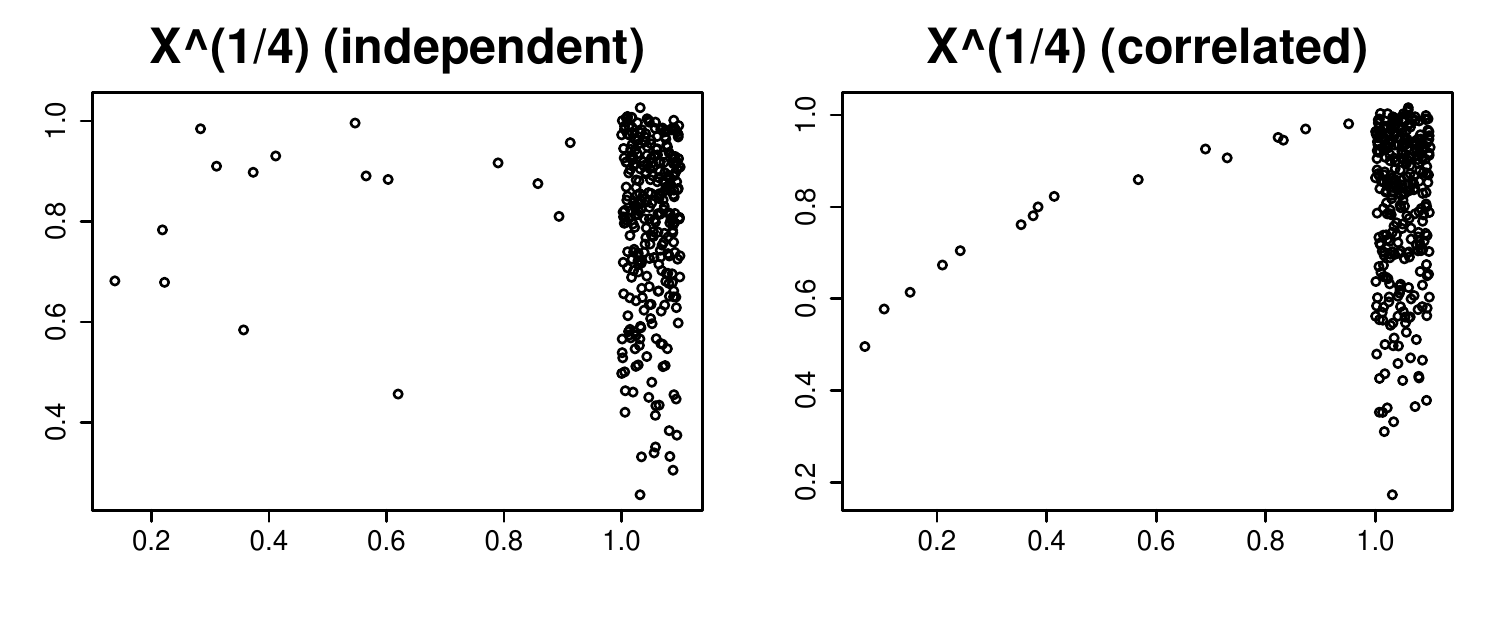} 
\put(-50,4){\tiny $x$}
\put(-150,4){\tiny $x$}
\put(-250,4){\tiny $x$}
\put(-350,4){\tiny $x$}
\put(-198,45){\tiny $y$}
\put(-398,45){\tiny $y$}\\
\includegraphics[scale=0.45]{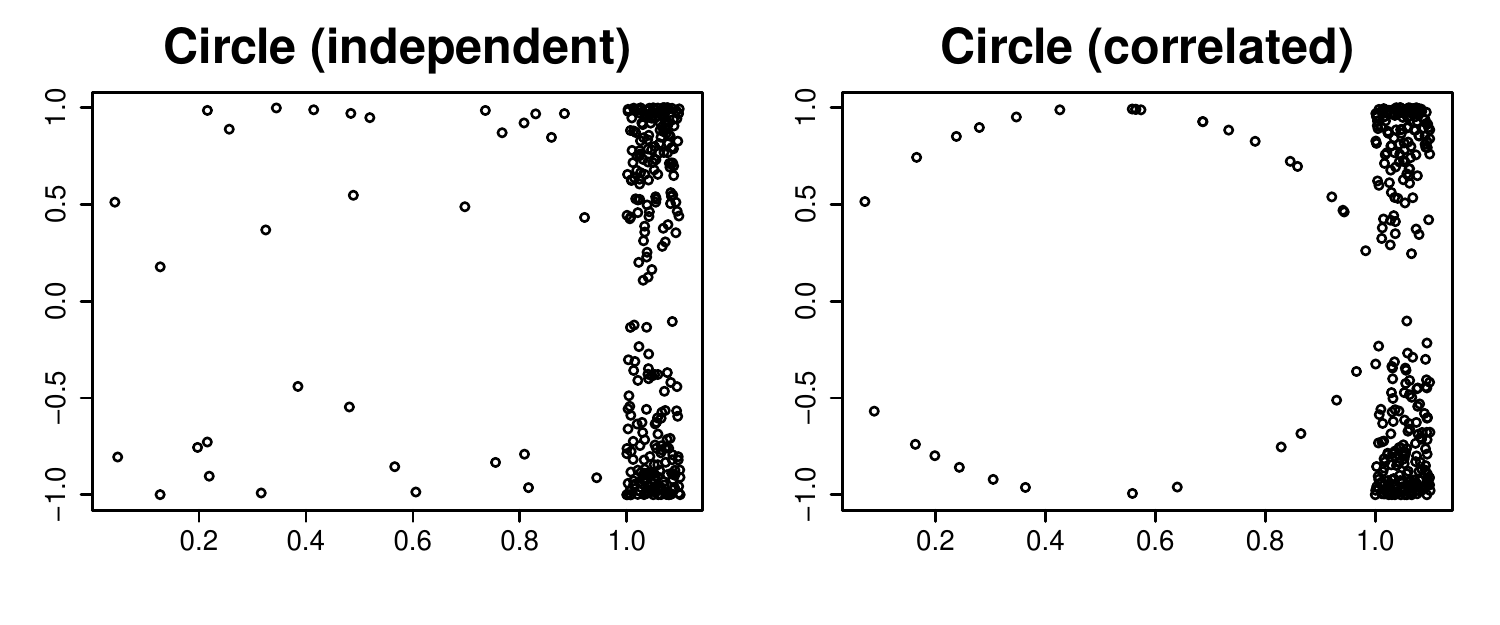}\ \ \ \includegraphics[scale=0.45]{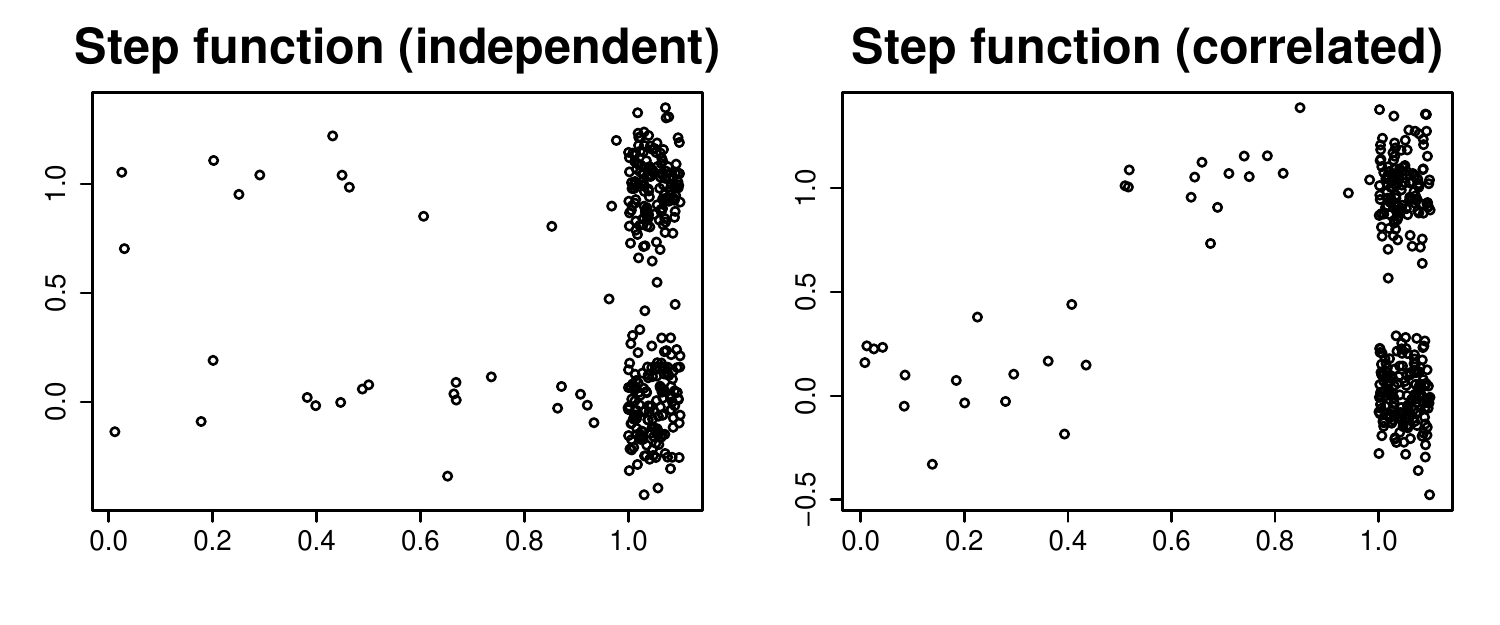}
\put(-50,4){\tiny $x$}
\put(-150,4){\tiny $x$}
\put(-250,4){\tiny $x$}
\put(-350,4){\tiny $x$}
\put(-198,45){\tiny $y$}
\put(-398,45){\tiny $y$}\\
\end{center}
\caption{Sample data points for eight functions with/without a potential correlation for $n=320$.}
\label{sample-points}
\end{figure}

\begin{table}
\caption{Comparison of correlation coefficients for independent and correlated samples 
from Figure~\ref{sample-points}
}\label{comp-table} 
\begin{center}
\setlength{\tabcolsep}{3.5pt}
\def\arraystretch{1.3}
\begin{tabular}{|c|c|c|c|c|c|c|c|c|c|c|c|c|c|}
  \hline
 \multicolumn{4}{|c|}{}  & \multicolumn{2}{|c|}{Cor} & \multicolumn{2}{c|}{dCor} & \multicolumn{2}{c|}{mCor} & \multicolumn{2}{c|}{MIC} & \multicolumn{2}{c|}{HC} \\ 
  \hline
 \# & Function & $\a$ & $\sigma^2$& dep & indep & dep & indep &  dep & indep &  dep & indep &  dep & indep \\ 
  \hline
1 & Linear & 0.05 & 0.03 & 0.03 & 0.00 & 0.19 & 0.11 & 0.06 & 0.04 & 0.21 & 0.17 & \textbf{0.18} & \textbf{0.08} \\ 
  2 & Quadratic & 0.10 & 0.10 & 0.00 & 0.01 & 0.09 & 0.10 & \textbf{0.07} & \textbf{0.02} & 0.21 & 0.18 & 0.08 & 0.04 \\ 
  3 & Cubic & 0.10 & 0.00 & 0.02 & 0.00 & 0.16 & 0.08 & 0.09 & 0.03 & \textbf{0.26} & \textbf{0.17} & 0.11 & 0.04 \\ 
  4 & sin($4\pi X$) & 0.05 & 0.03 & 0.00 & 0.00 & 0.10 & 0.06 & 0.03 & 0.01 & 0.20 & 0.18 & \textbf{0.10} & \textbf{0.04} \\ 
  5 & sin($16\pi X$) & 0.10 & 0.00 & 0.00 & 0.00 & 0.07 & 0.08 & \textbf{0.03} & \textbf{0.03} & 0.18 & 0.22 & \textbf{0.03} & \textbf{0.03} \\ 
  6 & $X^{1/4}$ & 0.05 & 0.01 & 0.01 & 0.00 & 0.12 & 0.07 & 0.02 & 0.01 & 0.20 & 0.20 & \textbf{0.12} & \textbf{0.04} \\ 
  7 & Circle & 0.10 & 0.00 & 0.00 & 0.00 & 0.09 & 0.05 & 0.01 & 0.03 & 0.16 & 0.17 & \textbf{0.06} & \textbf{0.01} \\ 
  8 & Step func. & 0.10 & 0.03 & 0.00 & 0.00 & 0.13 & 0.07 & 0.04 & 0.02 & 0.20 & 0.17 & \textbf{0.11} & \textbf{0.04} \\ 
   \hline
\end{tabular}
\end{center}
\end{table}


A formal statistical approach to test the robustness as well as accuracy is to run {\em power tests}:
testing for the power of the estimator in binary hypothesis tests. 
To compute the power of each estimator, we compare the false negative rate at a fixed false positive rate of, say, 5\%. We generate 500 independent datasets and 500 correlated datasets. We compute the correlation estimates on 500 independent samples, and take the top 5\% as a threshold. We compute the correlation estimates on 500 correlated samples. Power is defined as the fraction of correlated datasets for which the correlation estimate is larger than the threshold. 
 
 We show empirically that for linear, quadratic, sine with period 1/2, and the step function,
 the hypercontractivity coefficient is more powerful as compared to other measures.
For a given setting, a larger power means a larger successful detection rate for a fixed false alarm rate.
Figure~\ref{fig:005} shows the power of correlation estimators as a function of the additive noise level, $\sigma^2$, for $\a = 0.05$ and $n=320$.
The hypercontractivity coefficient is more powerful than other correlation estimators for most functions. The power of all the estimators are very small for sine (period 1/8) and circle functions. This is not surprising given that it is very hard to discern the correlated and independent cases even visually, as shown in Figure~\ref{sample-points}. 

\begin{figure}[H]
\begin{center}
\includegraphics[scale=0.63]{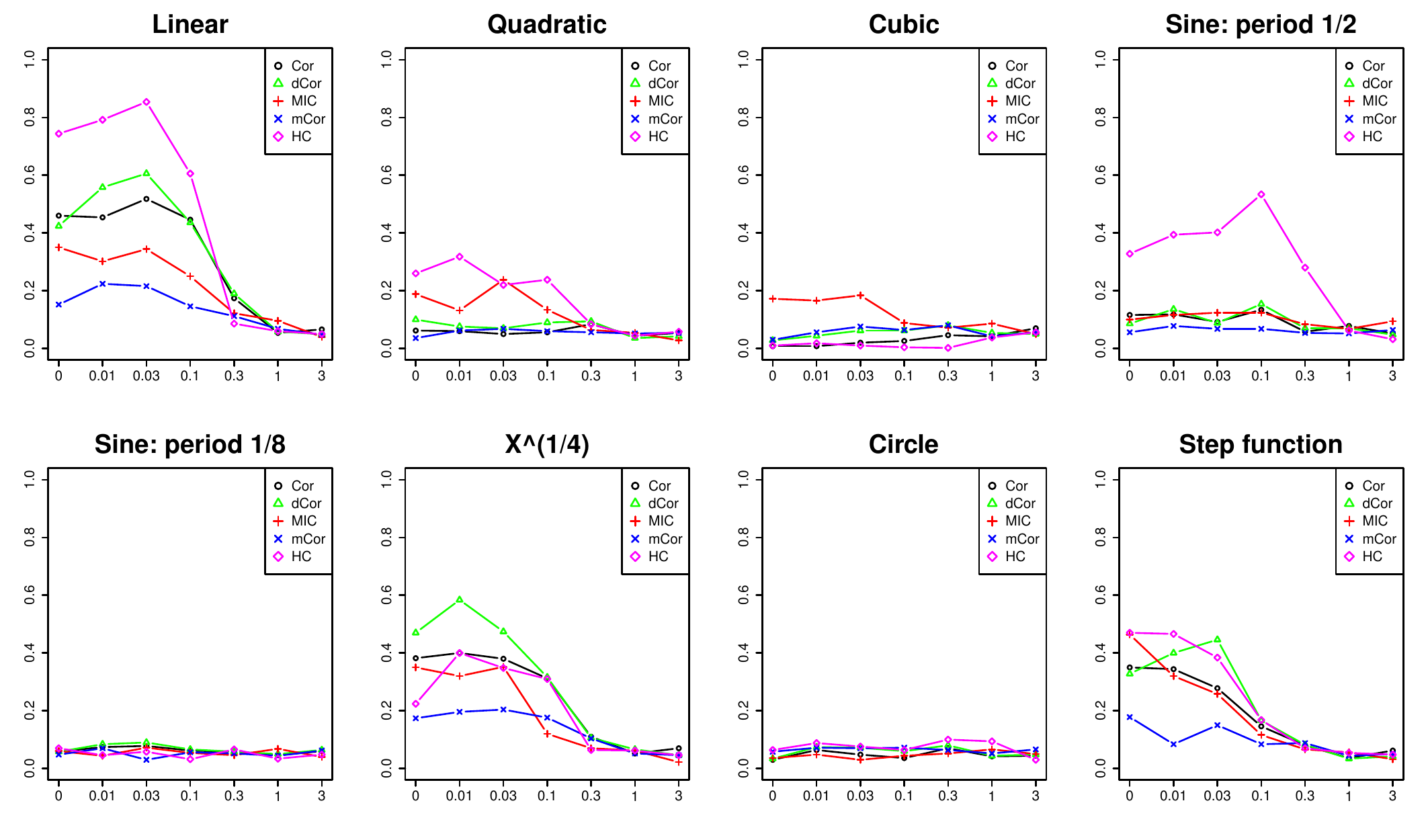}
\put(-64,0){\small Noise level}
\put(-158,0){\small Noise level}
\put(-256,0){\small Noise level}
\put(-352,0){\small Noise level}
\put(-64,115){\small Noise level}
\put(-158,115){\small Noise level}
\put(-256,115){\small Noise level}
\put(-352,115){\small Noise level}
\put(-390,50){\small \rotatebox{90}{Power}}
\put(-390,160){\small \rotatebox{90}{Power}}
\end{center}
\caption{Power vs. noise level for $\alpha = 0.05$, $n=320$}
\label{fig:005}
\end{figure}
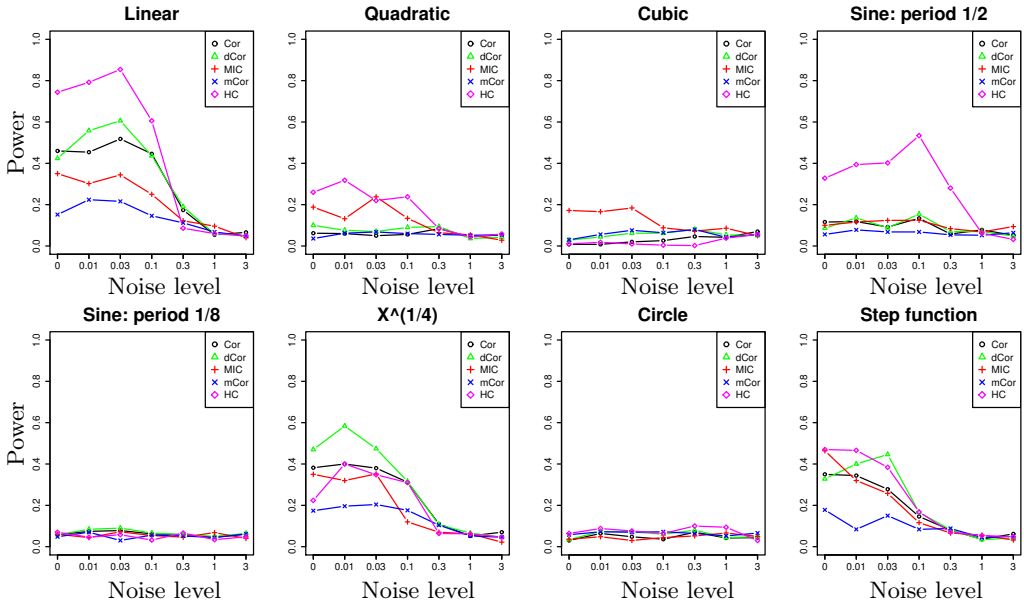

Figure~\ref{fig:01} plots the power of correlation estimators as a function of noise level for $\a = 0.1$ and $n=320$. As we can see from these figures, hypercontractivity estimator is more powerful than other correlation estimators for most functions. For circle function, the gap between the power of hypercontractivity estimator and the powers of other estimators is significantly large. 

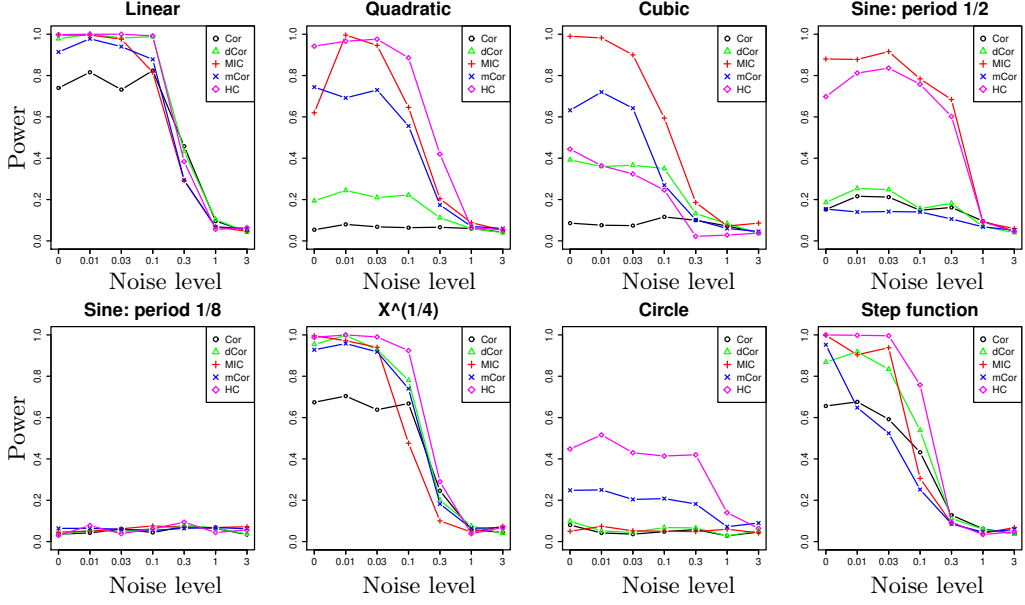
\begin{figure}[ht]
\begin{center}
\includegraphics[scale=0.63]{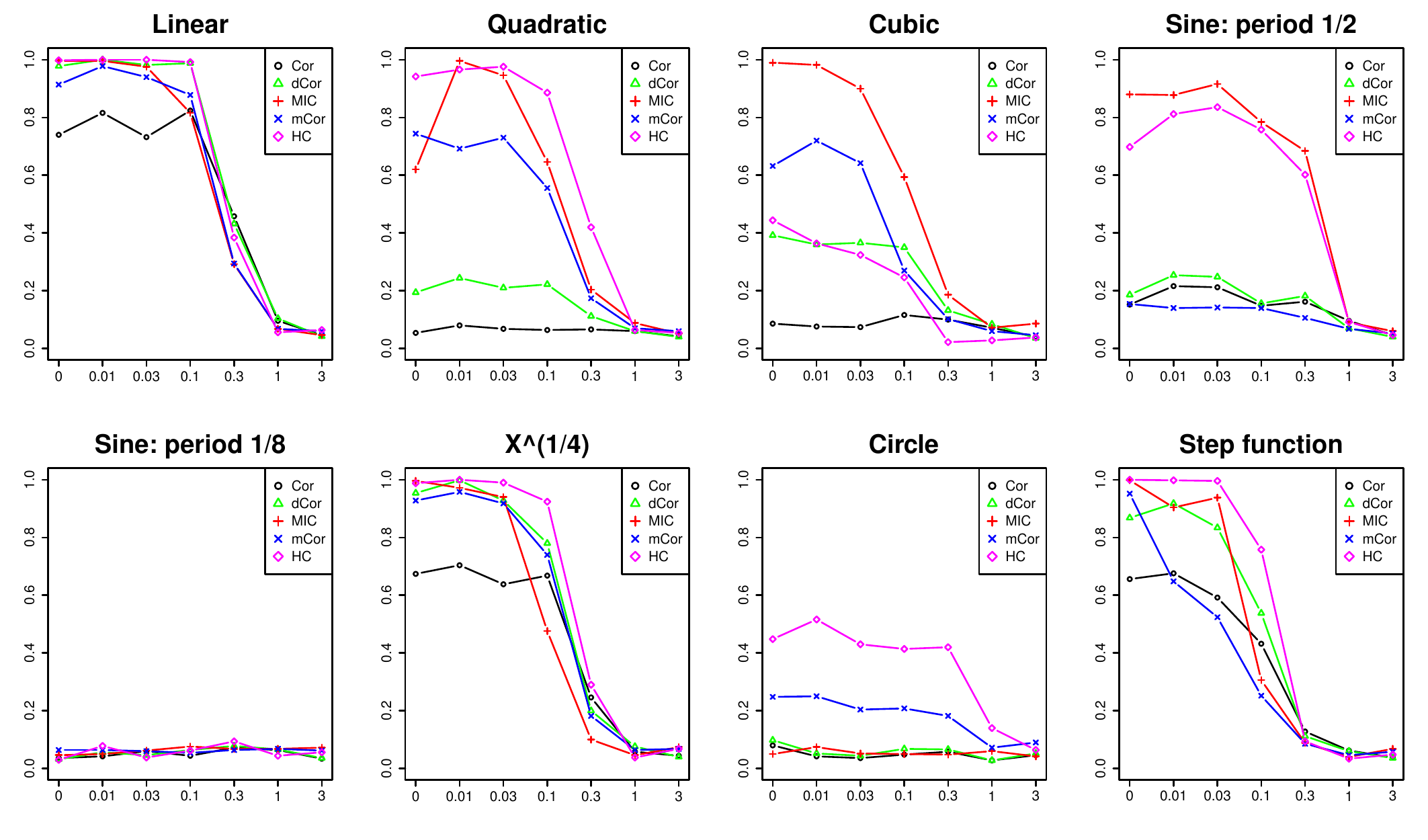}
\put(-64,0){\small Noise level}
\put(-158,0){\small Noise level}
\put(-256,0){\small Noise level}
\put(-352,0){\small Noise level}
\put(-64,115){\small Noise level}
\put(-158,115){\small Noise level}
\put(-256,115){\small Noise level}
\put(-352,115){\small Noise level}
\put(-390,50){\small \rotatebox{90}{Power}}
\put(-390,160){\small \rotatebox{90}{Power}}
\end{center}
\caption{Power vs. noise level for $\alpha = 0.1$, $n = 320$}
\label{fig:01}
\end{figure}

On the other hand, hypercontractivity estimator is power deficient for the cubic function. This is because in estimating hypercontractivity coefficient, we estimate $p(y_j|x_i)/p(y_j)$ using the kernel density estimator (KDE), which gives a smooth estimate of $p(y_j|x_i)/p(y_j)$, i.e., for $x_i$ and $x_j$ close to each other, estimated $p(y|x_i)$ and $p(y|x_j)$ are close to each other. Hence, for a correlated dataset for a cubic function, shown in Figure~\ref{cubic-comp} (A)-(right), the estimated $p(y|x)$ does not vary much for $x$. (Estimated $p(y|x)$ for $x \in [0.8:1]$ and $p(y|x)$ for $x \in [1:1.1]$ are close to each other). This results in a small hypercontracitivy, which in turn results in a low power in the hypothesis testing. 
 To further analyze this effect, we considered the same dataset but with dominant independent samples appear on the left, as shown in Figure~\ref{cubic-comp} (B)-(middle) and (right), and computed the power of hypercontractivity estimator, shown in Figure~\ref{cubic-comp}-(B) (left).  Hypercontractivity estimator is much more powerful than the one for the original dataset. This is because the estimated $p(y|x)$ for $x \in [0.8,1]$ is very different from the estimated $p(y|x)$ for $x \in [-0.1,0]$, which results in a large hypercontractivity coefficient for the correlated dataset.

\begin{figure}[H]
\begin{center}
\begin{subfigure}[t]{\textwidth}
\begin{center}
\includegraphics[scale=0.55]{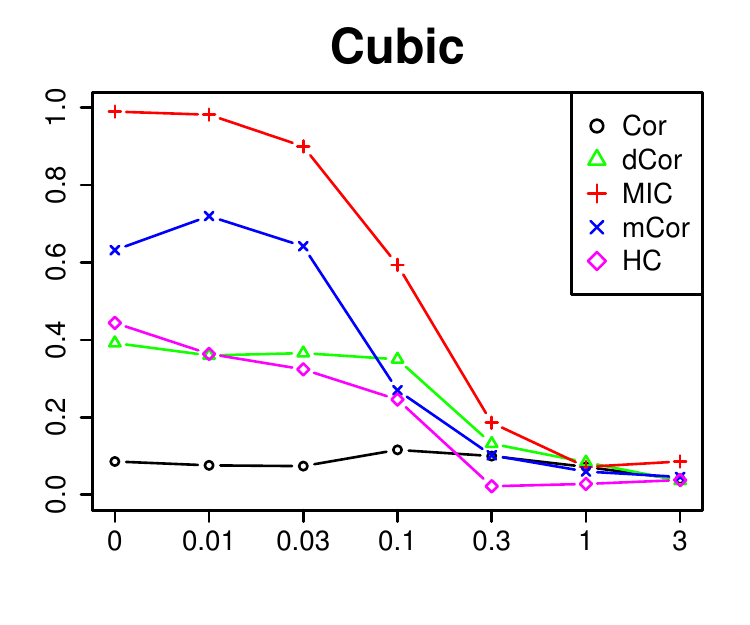}
\includegraphics[scale=0.55]{figs/442_draw_samples3.pdf}
\put(-58,0){$x$}
\put(-175,0){$x$}
\put(-118,50){$y$}
\put(-238,50){$y$}
\put(-318,0){\small Noise level}
\put(-365,42){\small \rotatebox{90}{Power}}
\caption{}
\end{center}
\end{subfigure}
\begin{subfigure}[t]{\textwidth}
\begin{center}
\includegraphics[scale=0.55]{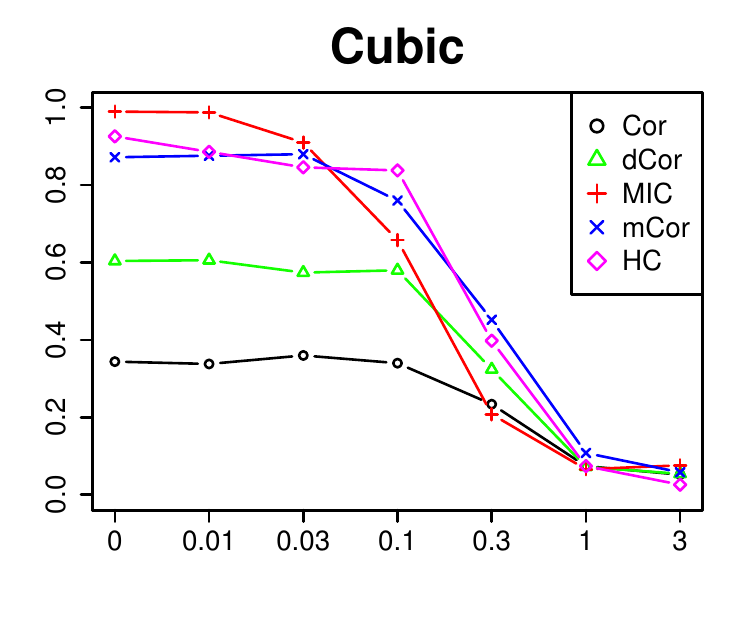}
\includegraphics[scale=0.55]{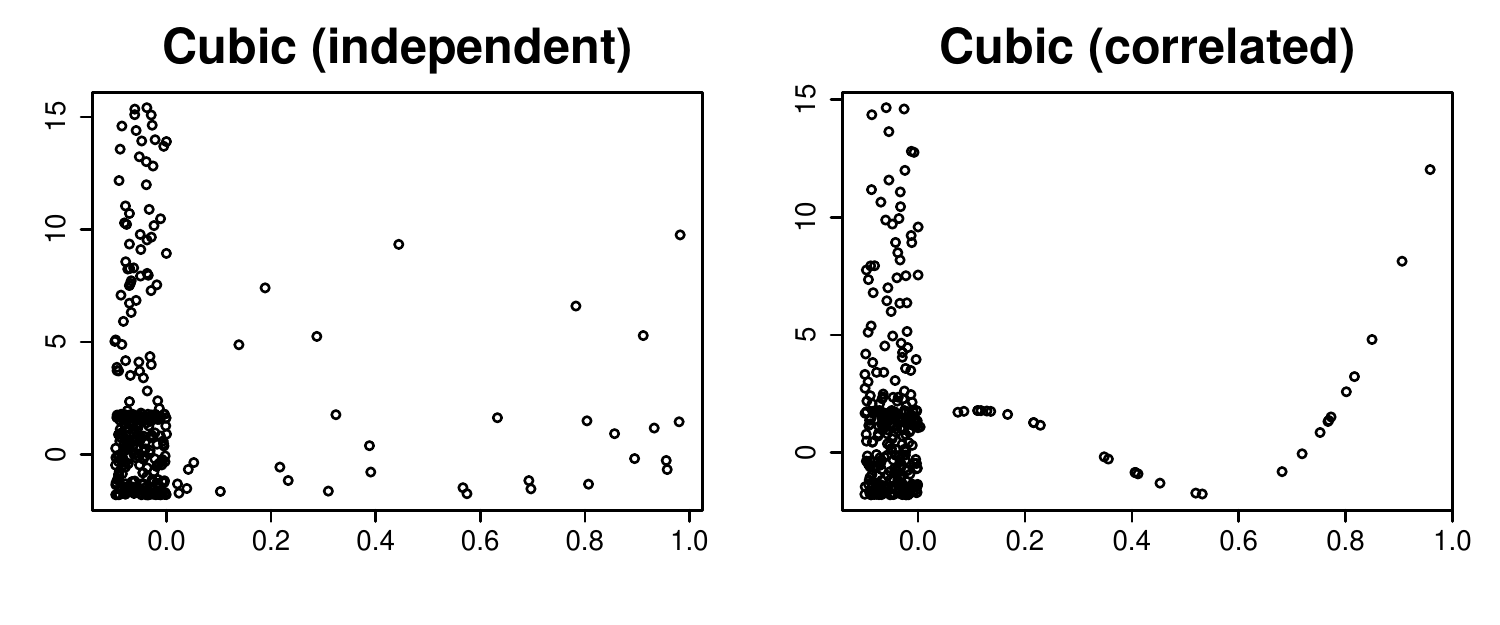}
\put(-58,0){$x$}
\put(-175,0){$x$}
\put(-118,50){$y$}
\put(-238,50){$y$}
\put(-318,0){\small Noise level}
\put(-365,42){\small \rotatebox{90}{Power}}
\caption{}
\end{center}
\end{subfigure}
\end{center}
\caption{Power vs. noise level for $\a = 0.1$ and $n=320$ (left), corresponding examples of an independent dataset (middle) and a correlated dataset (right).
}
\label{cubic-comp} 
\end{figure}

To investigate the dependency of power on $\a$ more closely, in Figure~\ref{noise0.1-sample320-alpha}, we plot the power vs. $\a$ or $n = 320$ and $\sigma^2 = 0.1$. 
Hypercontractivity estimator is more powerful than other estimators for most $\a$, for all functions except for cubic function. For a sine with period 1/8, due to its high frequency, the powers of all the correlation estimators do not increase as $\a$ increases. Figure~\ref{alpha0.05-noise0.1-sample} plots the power vs. sample size $n$ for $\a =  0.05$ and $\sigma^2 = 0.1$. For sine with period 1/2, hypercontractivity estimator is much more powerful than the other estimators for all sample sizes. We can also see that for sine with period 1/8, powers of all correlation estimators do not increase as sample size increases. 
\begin{figure}[H]
\begin{center}
\includegraphics[scale=0.63]{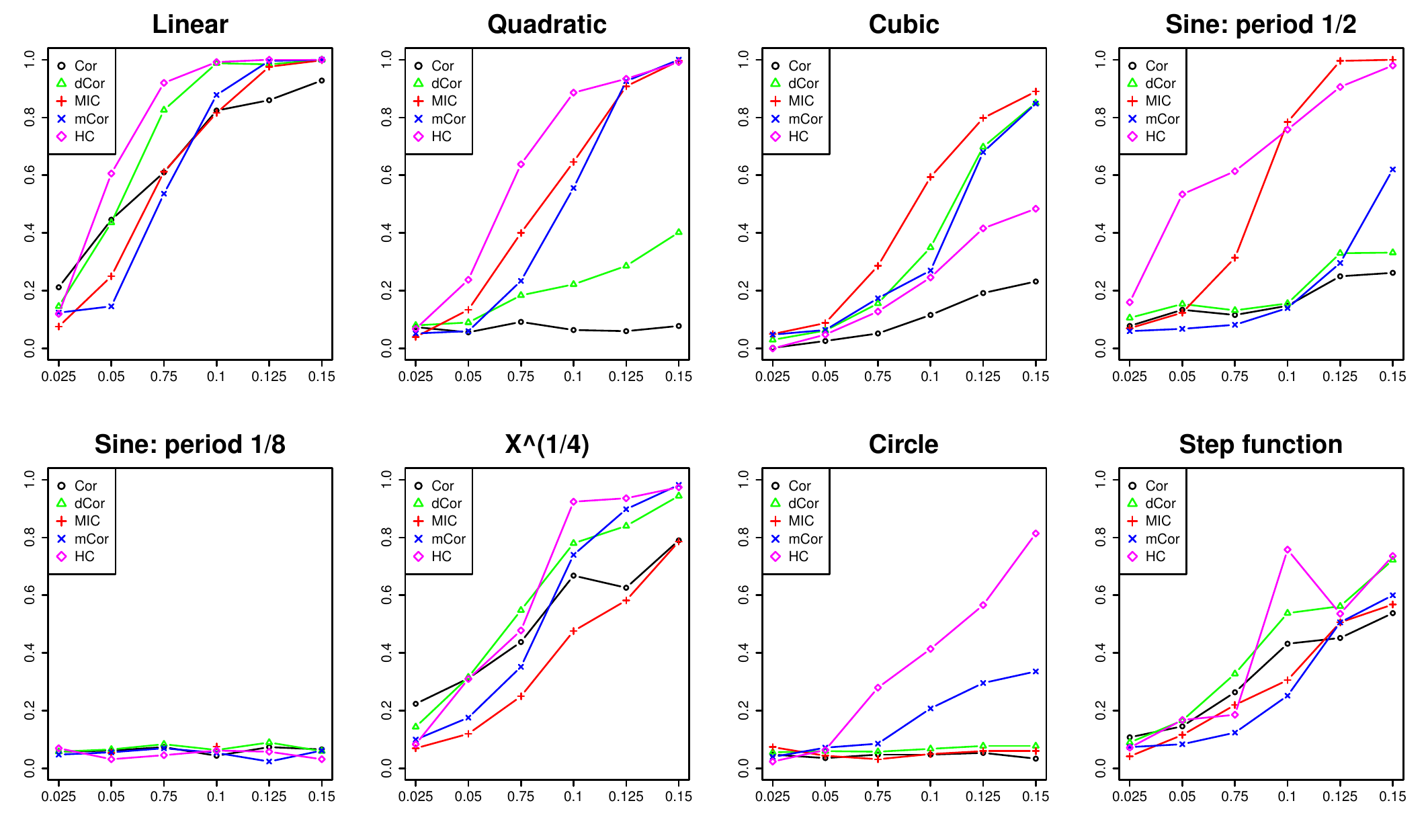}
\put(-50,0){\small $\a$}
\put(-144,0){\small $\a$}
\put(-243,0){\small $\a$}
\put(-338,0){\small $\a$}
\put(-50,115){\small $\a$}
\put(-144,115){\small $\a$}
\put(-243,115){\small $\a$}
\put(-338,115){\small $\a$}
\put(-390,50){\small \rotatebox{90}{Power}}
\put(-390,160){\small \rotatebox{90}{Power}}
\end{center}
\caption{Power vs. $\alpha$ (fraction of correlated samples) for $n = 320$, $\sigma^2$= $0.1$}
\label{noise0.1-sample320-alpha}
\end{figure}
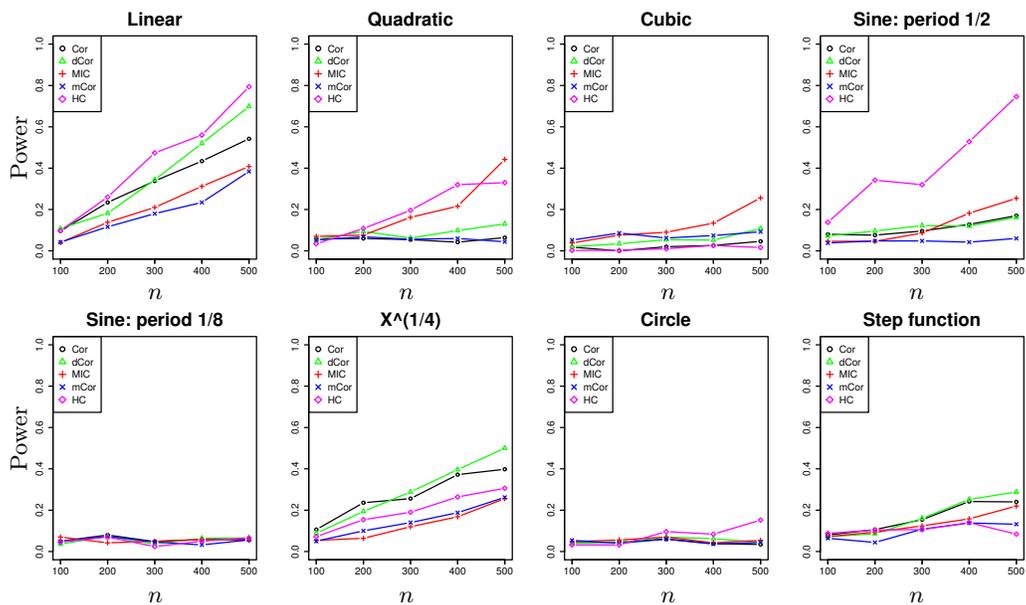
\begin{figure}[H]
\begin{center}
\includegraphics[scale=0.63]{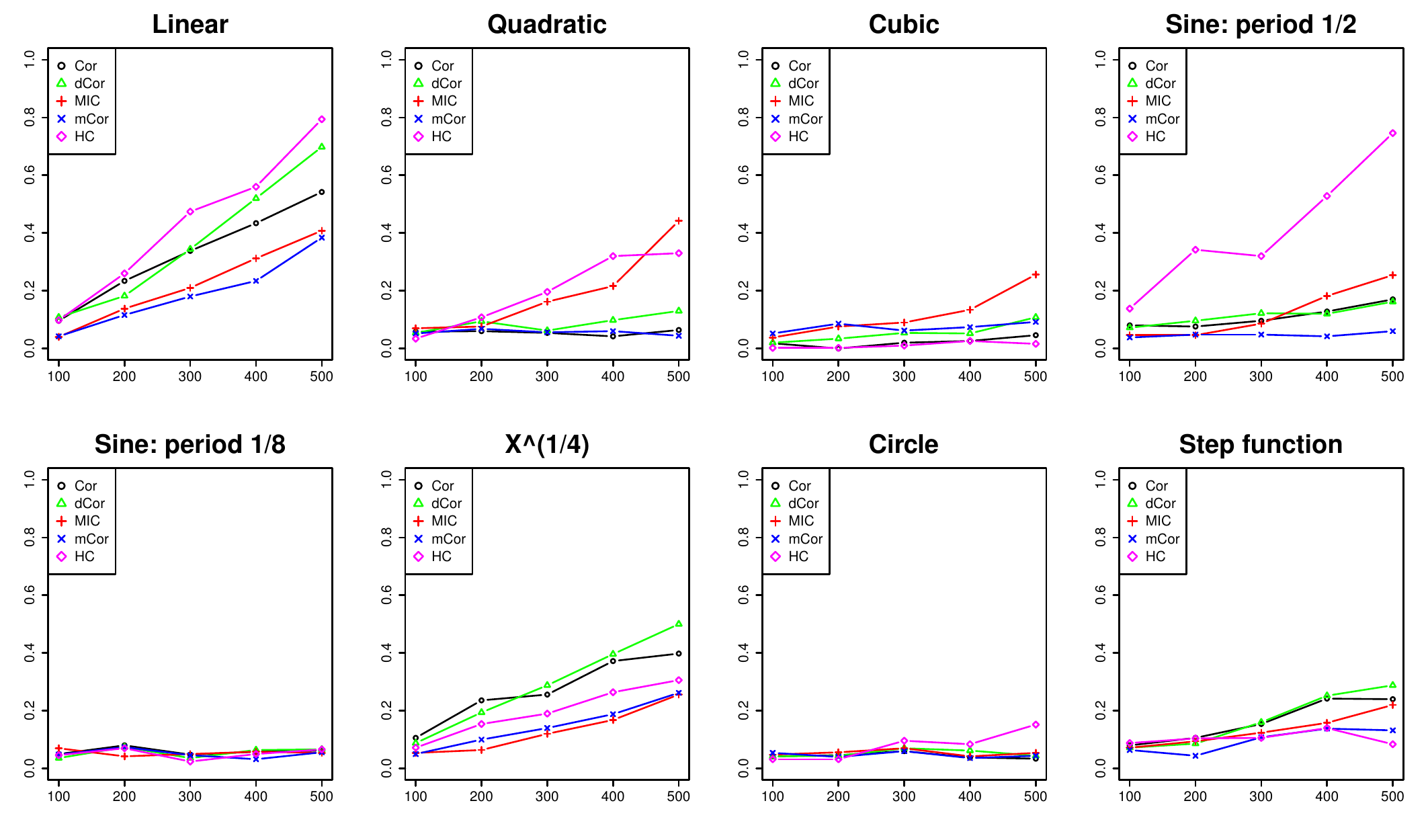}
\put(-50,0){\small $n$}
\put(-144,0){\small $n$}
\put(-243,0){\small $n$}
\put(-338,0){\small $n$}
\put(-50,115){\small $n$}
\put(-144,115){\small $n$}
\put(-243,115){\small $n$}
\put(-338,115){\small $n$}
\put(-390,50){\small \rotatebox{90}{Power}}
\put(-390,160){\small \rotatebox{90}{Power}}
\end{center}
\caption{Power vs. $n$ (number of samples) for $\alpha = 0.05$, $\sigma^2$= $0.1$} \label{alpha0.05-noise0.1-sample}
\end{figure}


\subsection{Real data: correlation between indicators of WHO datasets}
\label{sec:who}

We computed  the hypercontractivity coefficient, MIC, and Pearson correlation of 1600 pairs of indicators for 202 countries in the World Health Organization (WHO) dataset~\cite{ReshefEtAl2011}. 
Figure~\ref{WHO-VS} illustrates that the hypercontractivity coefficient discovers hidden potential correlation (e.g.~in (E) - (H)), whereas other measures fail.
Scatter plots of Pearson correlation vs. the hypercontractivity coefficient and MIC vs. the hypercontractivity coefficient for all pairs are presented in (A) and (D). 
The samples for pairs of indicators corresponding to B,C,E -- J are shown in (B),(C),(E) - (J), respectively. In Figure~\ref{WHO-VS} (B), it is reasonable to assume that the number of bad teeth per child is uncorrelated with the democracy score. 
The hypercontractivity coefficient,  MIC, and Pearson correlation are all small, as expected. In Figure~\ref{WHO-VS} (C), the correlation between CO$_2$ emissions and energy use is clearly visible, and all three correlation estimates are close to one.

However, only the hypercontractivity coefficient discovers the hidden potential correlation in Figure~\ref{WHO-VS} (E) -- (H). In Figure~\ref{WHO-VS} (E), the data is a mixture of two types of countries -- one with small amount of aid received (less than $\$ 5 \times 10^8$), and the other with large amount of aid received (larger than $\$ 5 \times 10^8$). Dominantly many countries (104 out of 146) belong to the first type (small aid), and for those countries, the amount of aid received and the income growth are independent.
For the remaining countries with larger aid received, although those are rare, there is a clear correlation between the amount of aid received and the income growth.

Similarly in Figure~\ref{WHO-VS}  (F), there are two types of countries -- one with small arms exports (less than $\$ 2 \times 10^8$) and the other with large arms exports (larger than $\$ 2\times 10^8$). Dominantly many countries (71 out of 82) belong to the first type, for which the amount of arms exports and the health expenditure are independent. For the remaining countries that belong to the second type, on the other hand, there is a visible correlation between the arms exports and the health expenditure. This is expected as for those countries that export arms the GDP is positively correlated with both arms exports and health expenditure, whereas for those do not have arms industry, these two will be independent.

 \begin{figure}[H]
 \centering
       \begin{subfigure}[t]{0.3\textwidth}
        \centering
        \includegraphics[width=\textwidth]{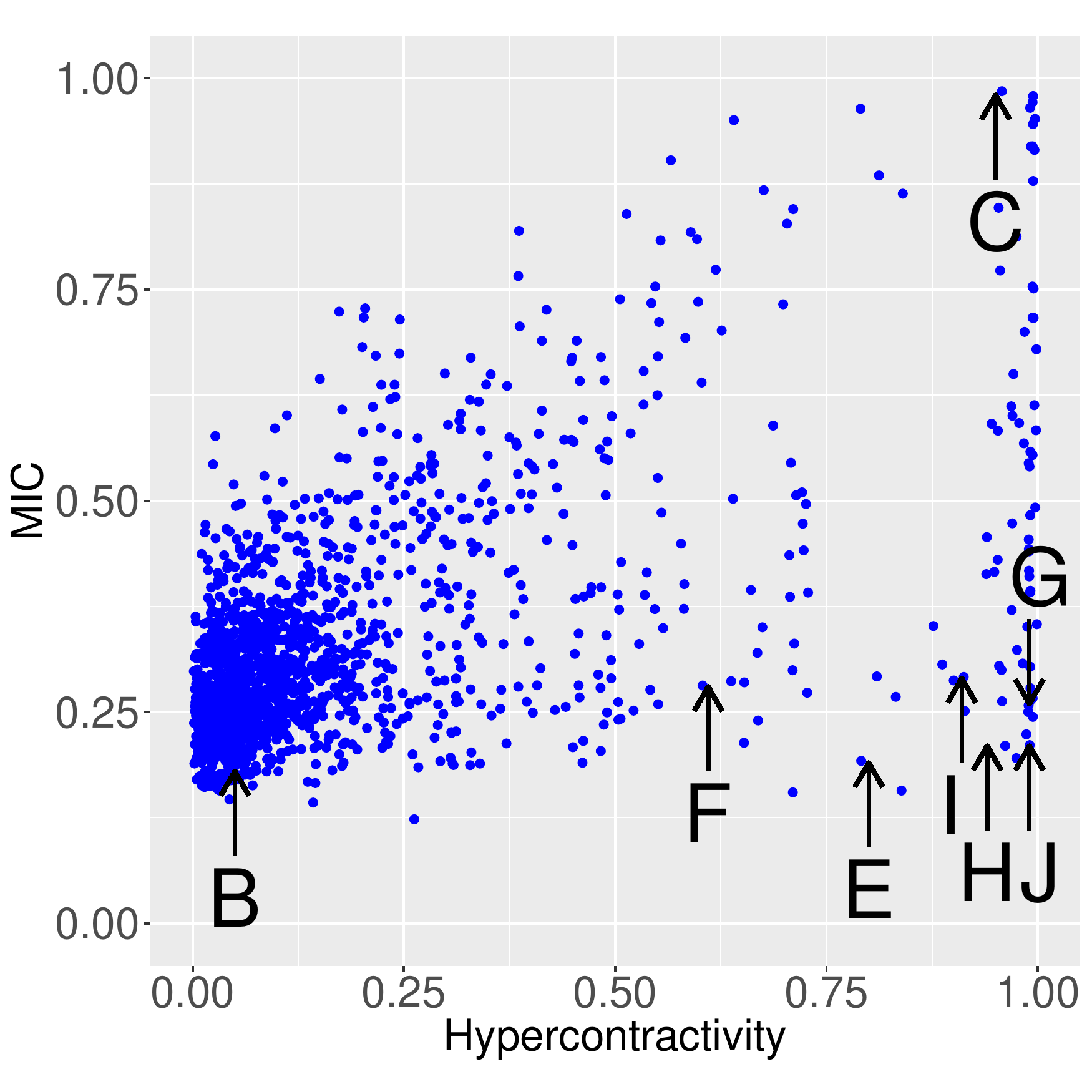}
        \caption{}
        \end{subfigure}\
                \begin{subfigure}[t]{0.3\textwidth}
        \centering
        \includegraphics[width=\textwidth]{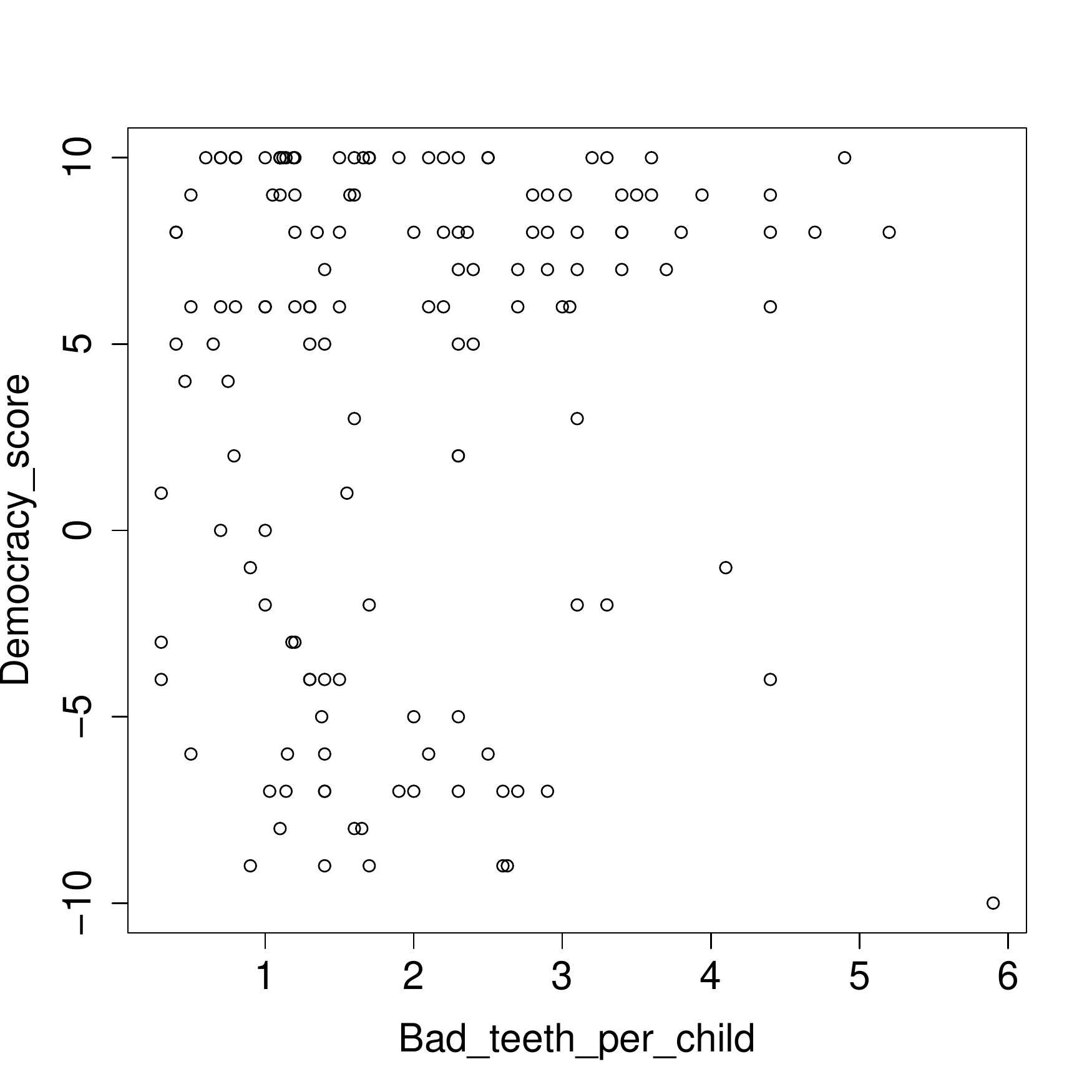}
        \caption{}
    \end{subfigure}\
        \begin{subfigure}[t]{0.3\textwidth}
        \centering
        \includegraphics[width=\textwidth]{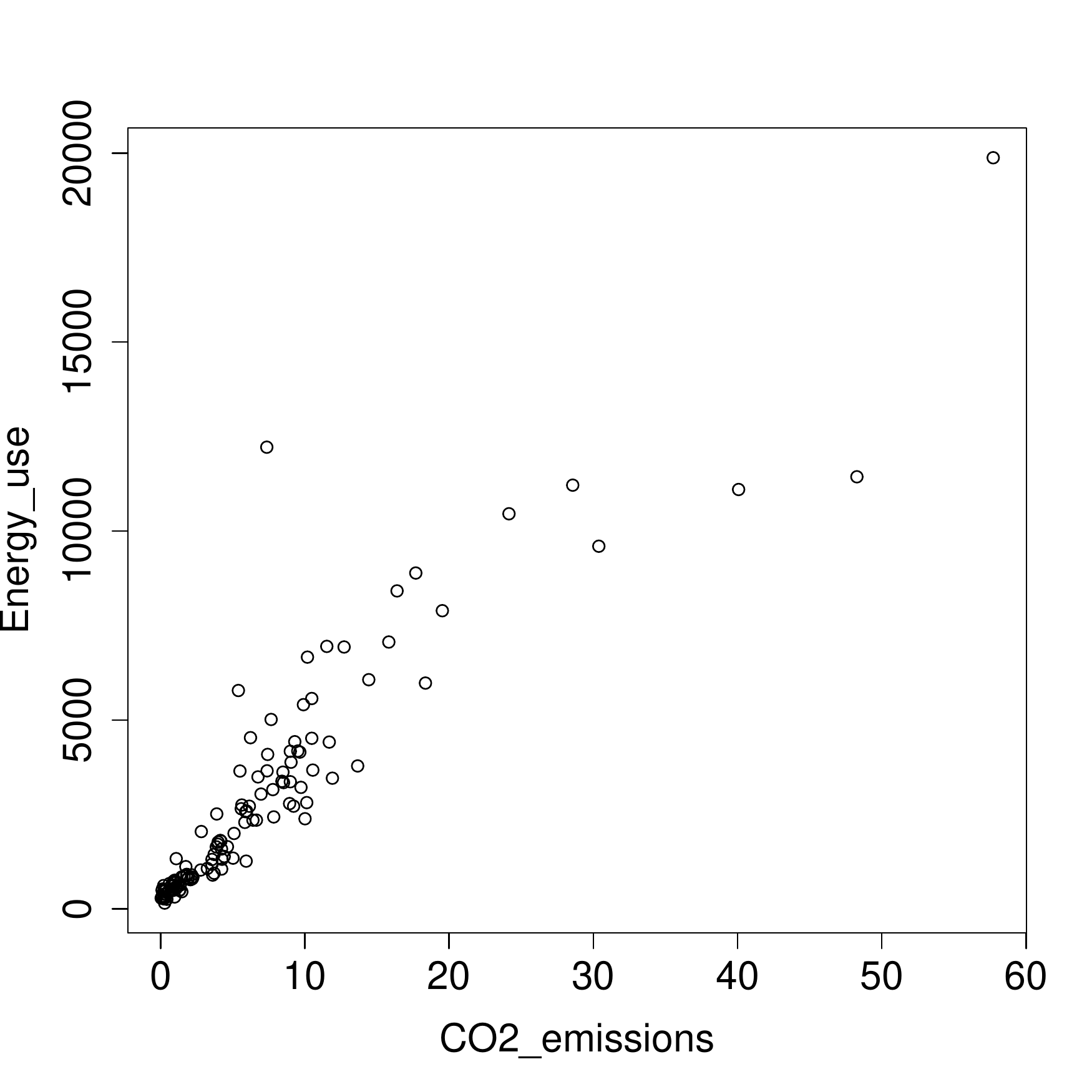}
        \caption{}
         \end{subfigure} 
                 \vspace{-2em}
\end{figure}
\begin{figure}[H]
    \ContinuedFloat 
\centering
                        \begin{subfigure}[t]{0.3\textwidth}
       \centering
        \includegraphics[width=\textwidth]{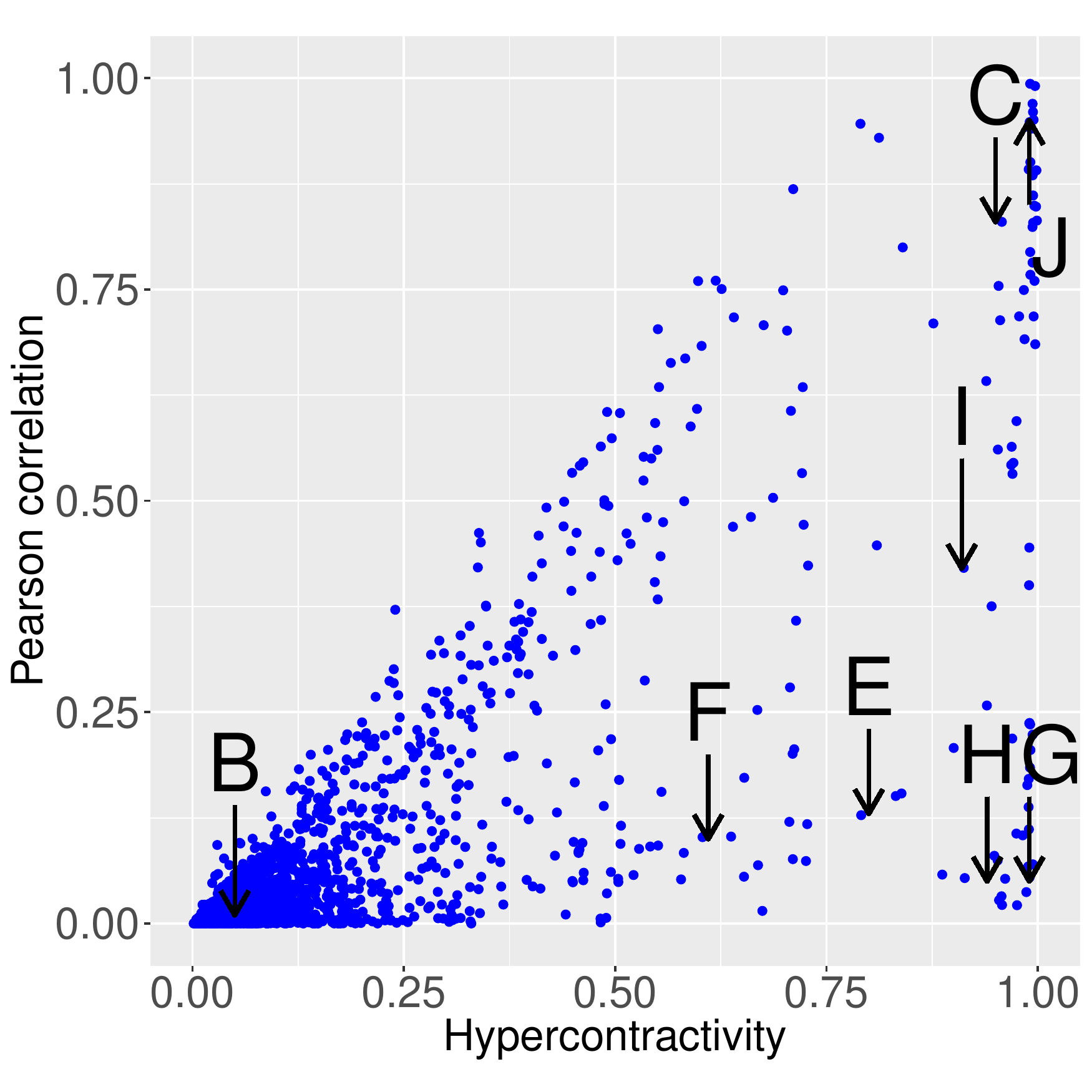}
       \caption{}
      \end{subfigure}\
        \begin{subfigure}[t]{0.3\textwidth}
        \centering
        \includegraphics[width=\textwidth]{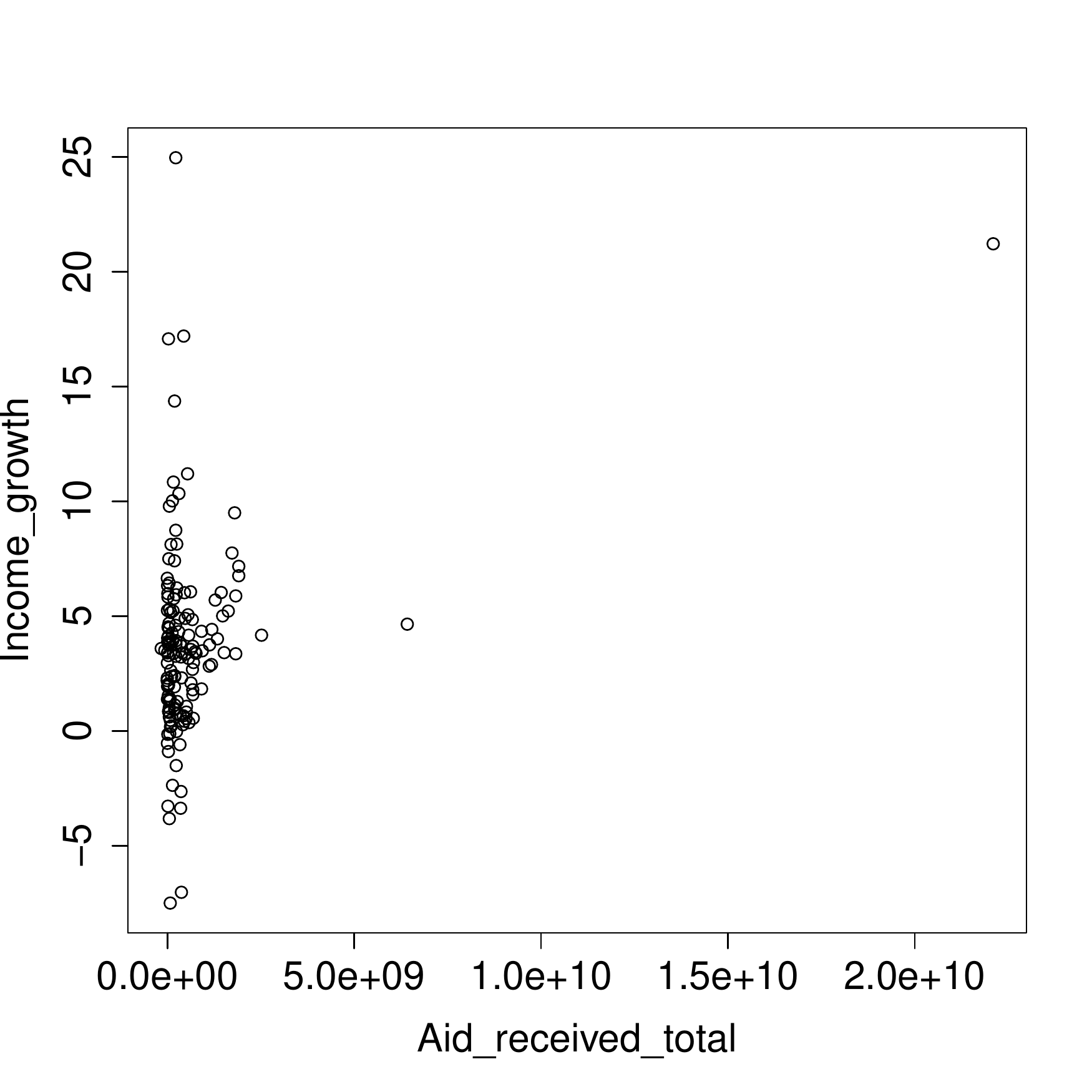} 
        \caption{}
        \end{subfigure}\
         \begin{subfigure}[t]{0.3\textwidth}
        \centering
        \includegraphics[width=\textwidth]{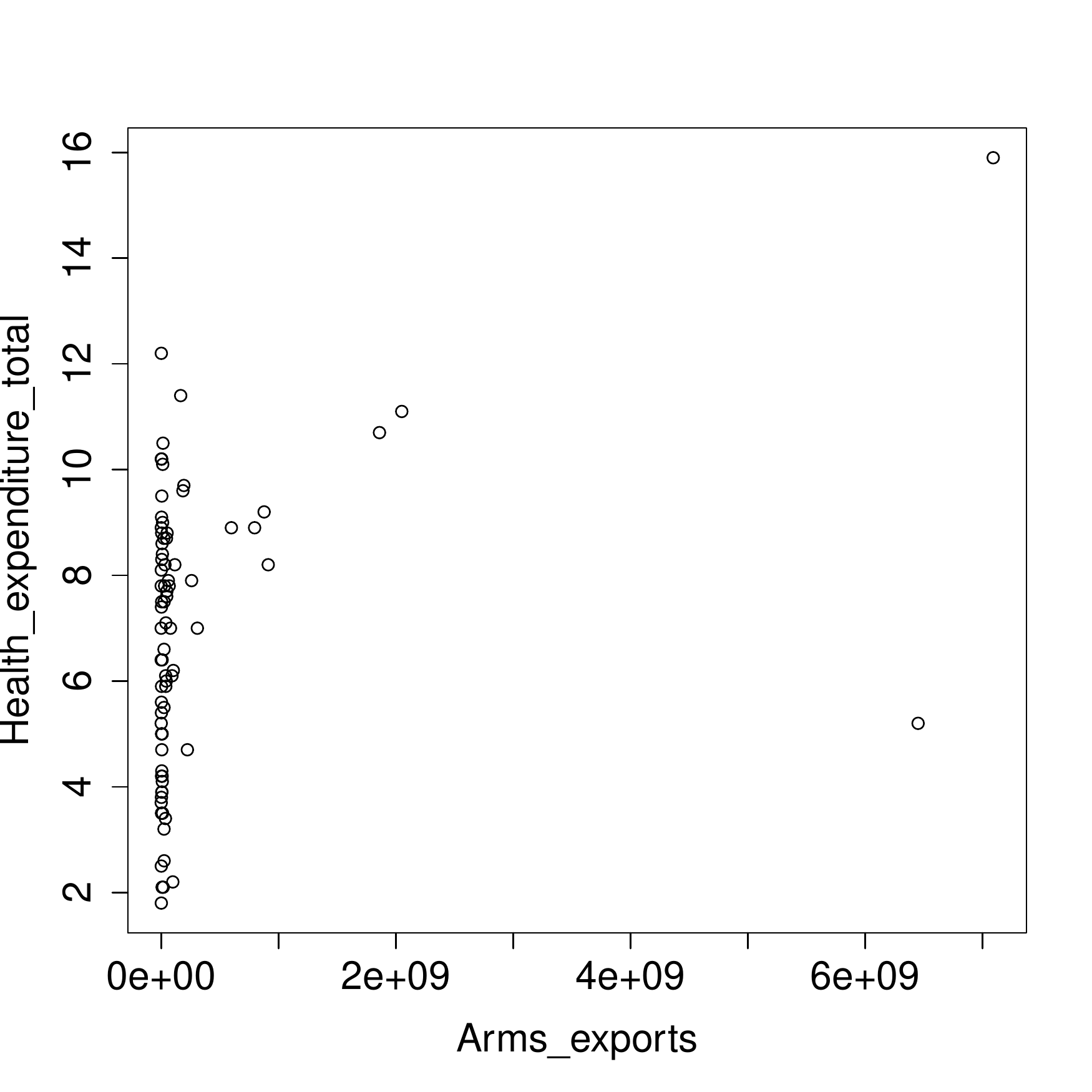}
        \caption{}
        \end{subfigure}
        \vspace{-2em}
\end{figure}
\begin{figure}[H]
    \ContinuedFloat 
\centering        
                \begin{subfigure}[t]{0.3\textwidth}
        \centering
        \includegraphics[width=\textwidth]{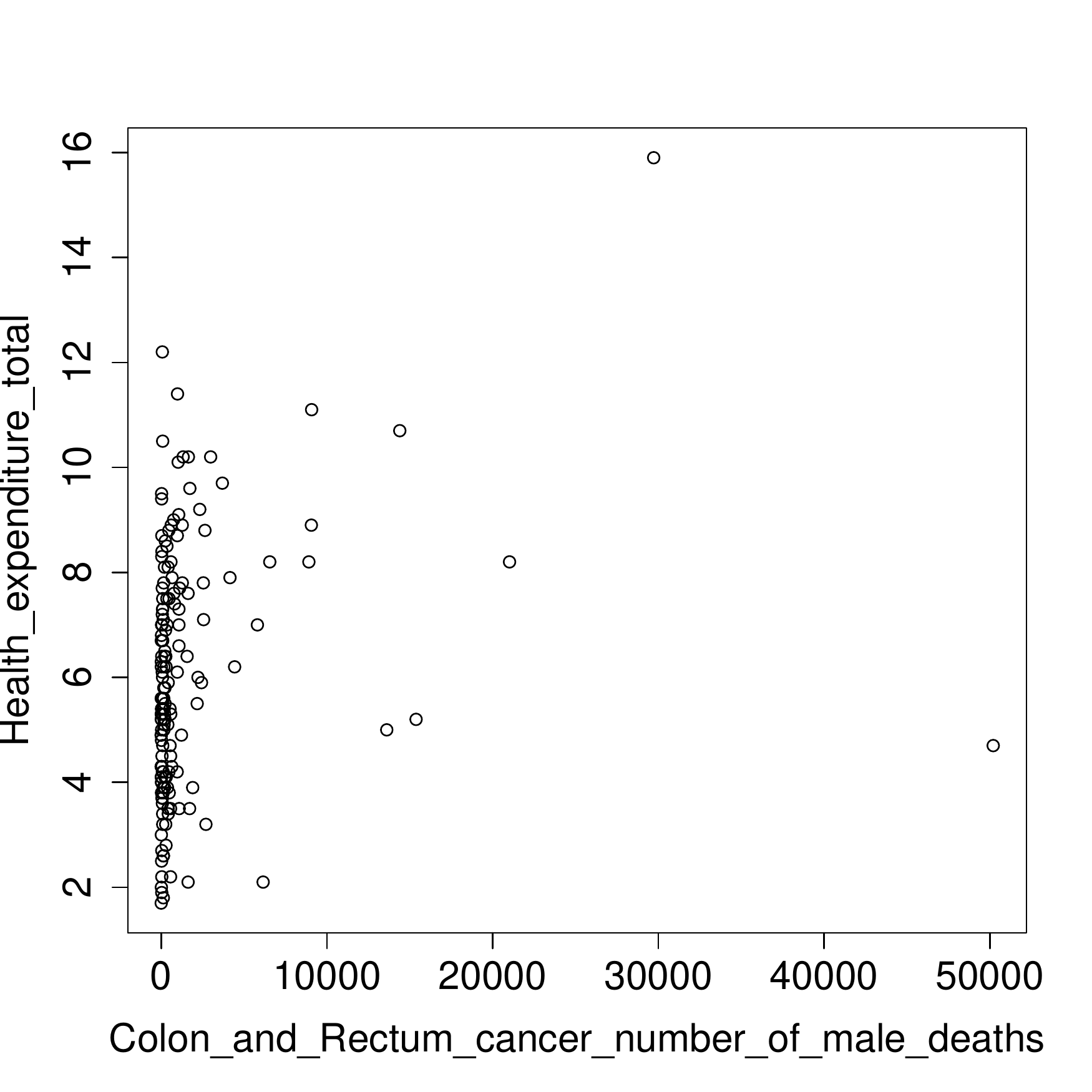}
        \caption{}
    \end{subfigure}\ 
        \begin{subfigure}[t]{0.3\textwidth}
        \centering
        \includegraphics[width=\textwidth]{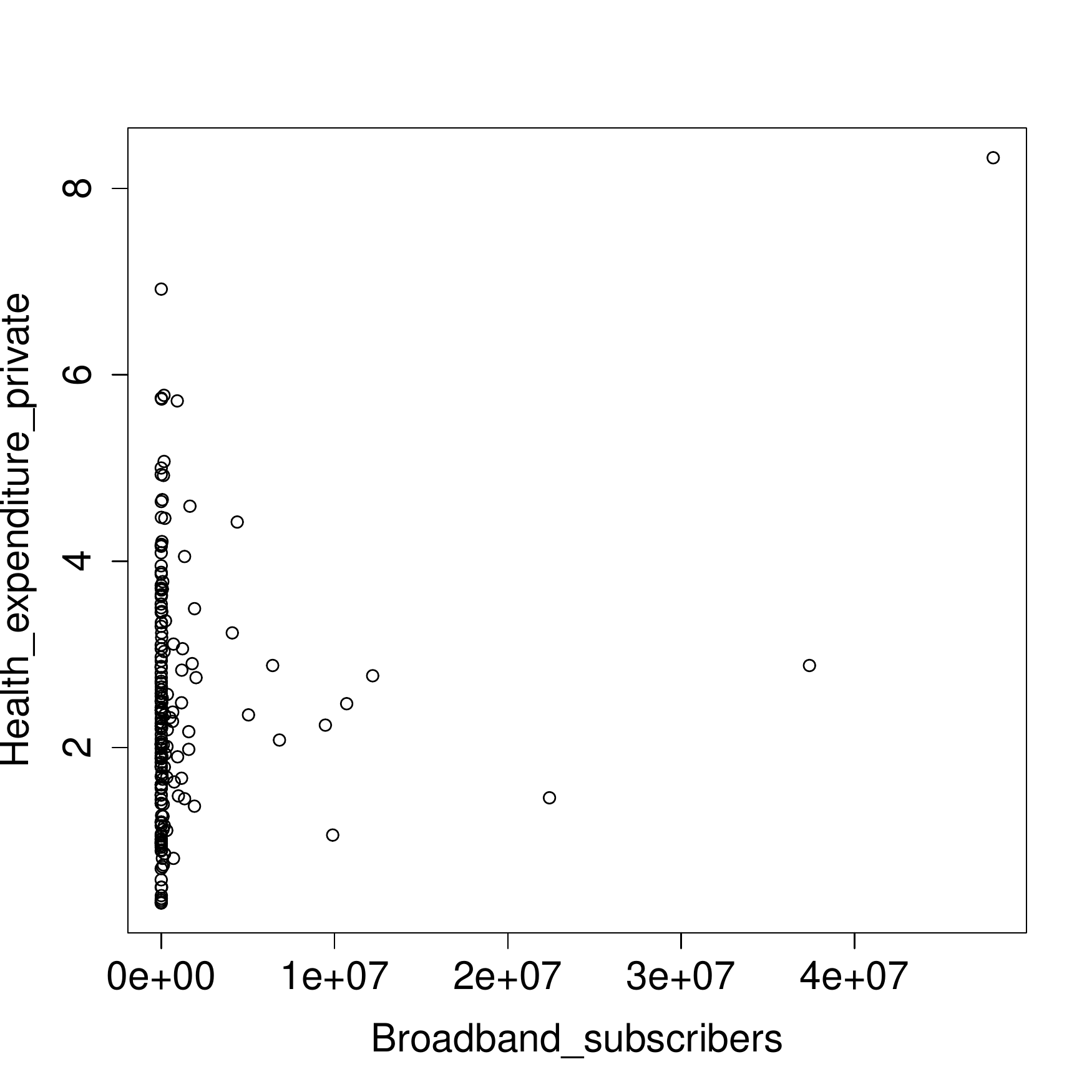}
        \caption{}
         \end{subfigure}
        \vspace{-2em}
         \end{figure}
\begin{figure}[H]
    \ContinuedFloat 
\centering
        \begin{subfigure}[t]{0.3\textwidth}
        \centering
        \includegraphics[width=\textwidth]{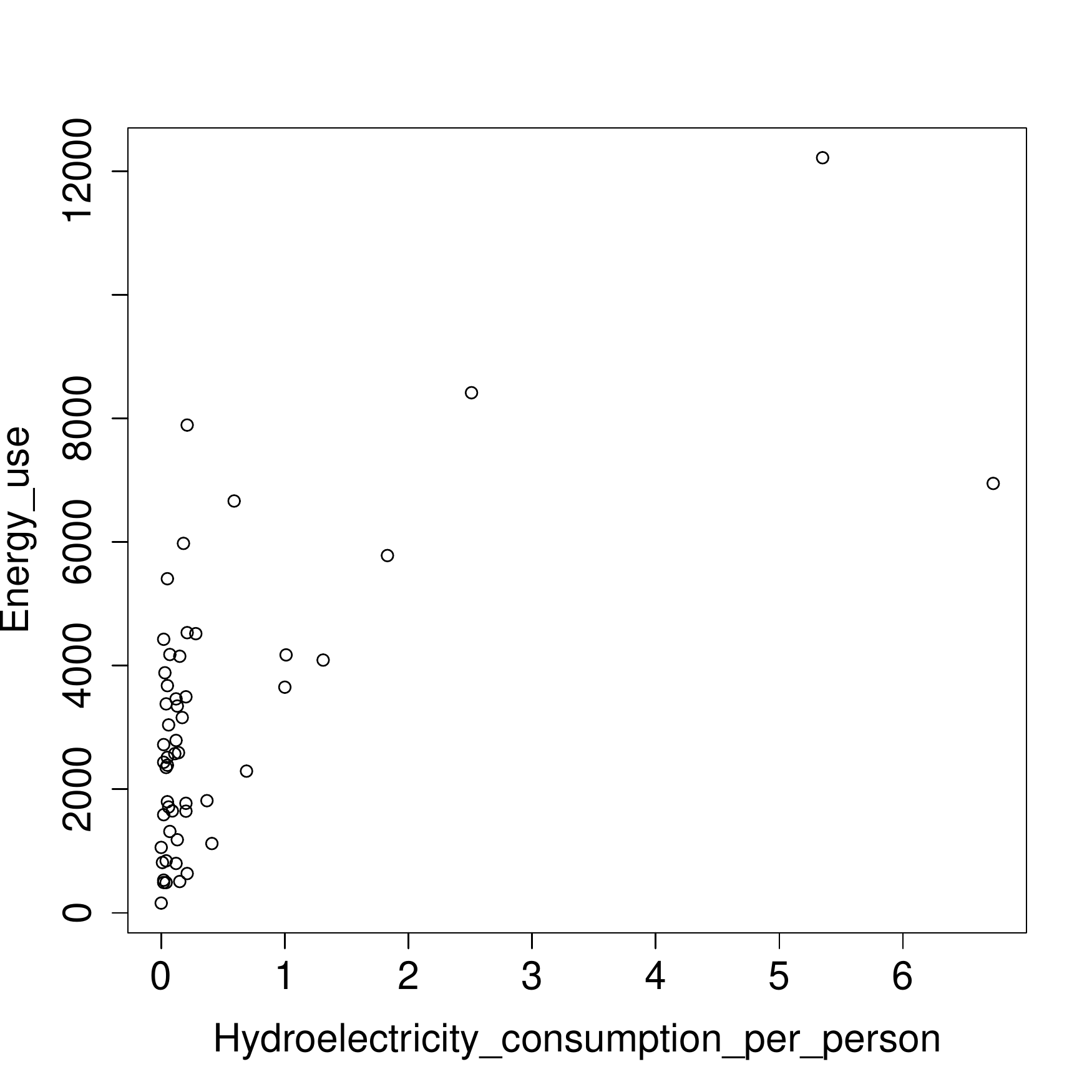} 
        \caption{}
        \end{subfigure}\ 
         \begin{subfigure}[t]{0.3\textwidth}
        \centering
        \includegraphics[width=\textwidth]{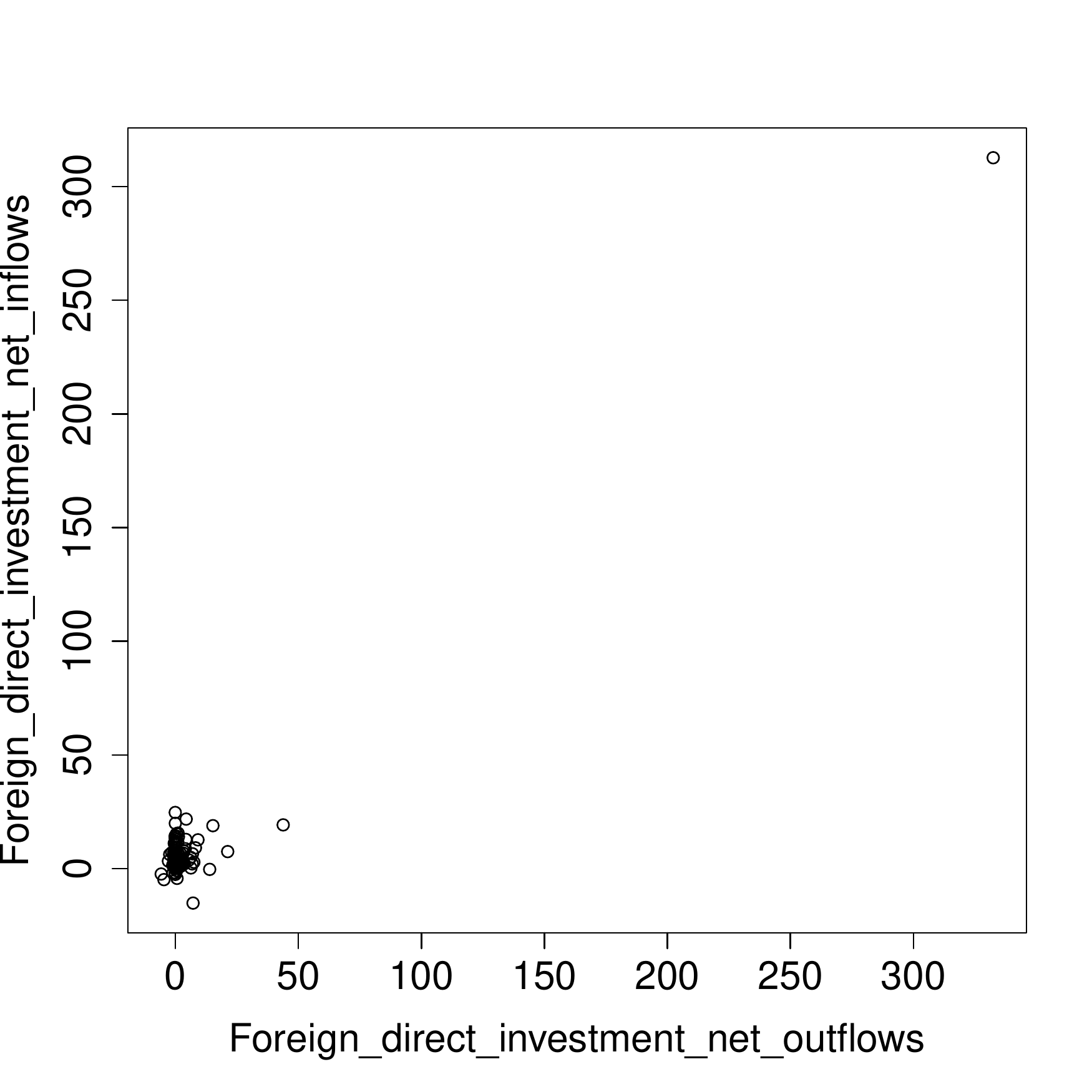}
        \caption{}
        \end{subfigure}
        \vspace{-0em}
\caption{(A) and (D): Scatter plot of correlation measures. (B): Correlations are small. (C): Correlations are large. (E),(F),(G),(H): Only the hypercontractivity coefficient discovers potential correlation. (I): Hypercontractivity discovers potential correlation. (J): Hypercontractivity and Pearson correlation are large because of an outlier. }\label{abcd}\label{WHO-VS}
    \end{figure}

In Figure~\ref{WHO-VS} (G), for dominant number of countries, the number of male deaths from the colon and rectum cancer is small (145 out of 169 countries have it less than 2000), and it is independent of the amount of health expenditure. On the other hand, for the remaining countries with larger number of male deaths from colon and rectum cancer, the two indicators are positively associated. This is expected as both indicators are positively correlated with the population. Only hypercontractivity discovers this hidden potential correlation. MIC and Pearson correlation are small. 

In Figure~\ref{WHO-VS} (H), for dominant number of countries, the number of broadband subscribers is very small and is independent of the private health expenditure; 155 out of 180 countries have broadband subscribers less than $10^6$. On the other hand, for the remaining countries, the number of broadband subscribers is positively correlated with the private health expenditure. This is as expected because both indicators are positively correlated with the population. Hypercontractivity is large for this dataset, discovering the hidden correlation, whereas all other correlations all small. 

In Figure~\ref{WHO-VS} (I), most countries do not have large hydroelectricity facilities, and for those countries, energy use and hydroelectricity consumption are independent (41 out of 53 countries have hydroelectricity $\le$ 0.25). On the other hand, for the countries which have hydroelectrocity facilities, the amount of total energy use and the amount of hydroelectricity consumption are positively correlated. Hypercontractivity discovers this hidden potential correlation.  Unlike in (G) and (H) for which the fraction of correlated samples was only about $14\%$, in (I), the fraction of correlated samples is about $23\%$. Hence, Pearson correlation is larger compared to Pearson correlation values for (G) and (H).


In Figure~\ref{WHO-VS} (J), there is one country (Luxembourg) with very large amounts of foreign direct investment net inflow and outflow. Due to this outlier, Pearson correlation is close to 1. Hypercontractivity is also close to 1, whereas MIC is small. To analyze the effect of the outlier in correlation measures, in the following, we compute the correlation measures for samples without an outlier.

\subsubsection{How hypercontractivity changes as we remove outliers}
 Figure~\ref{fig-repeat1}--\ref{more-good2}, on the left, are shown samples from Figure~\ref{WHO-VS} (E),(F),(G),(H),(I),(J) respectively. On the middle and on the right are shown all samples but one outlier and all samples but two outliers, respectively. By comparing the hypercontractivity coefficients for the three datasets for each pair of indicators, we can analyze the effect of outliers on hypercontractivity. For a comparison, on the top of each figure, we show the estimated hypercontractivity (HC), MIC, Pearson correlation (Cor), distance correlation (dCor), maximal correlation (mCor), and the hypercontractivity for reversed direction (HCR). 
In Figures~\ref{fig-repeat1} and~\ref{fig-repeat2}, we can see that hypercontractivity is more sensitive to an outlier than other correlation measures. 

 \begin{figure}[h]
        \begin{subfigure}[t]{\textwidth}
        \centering
        \includegraphics[width=0.3\textwidth]{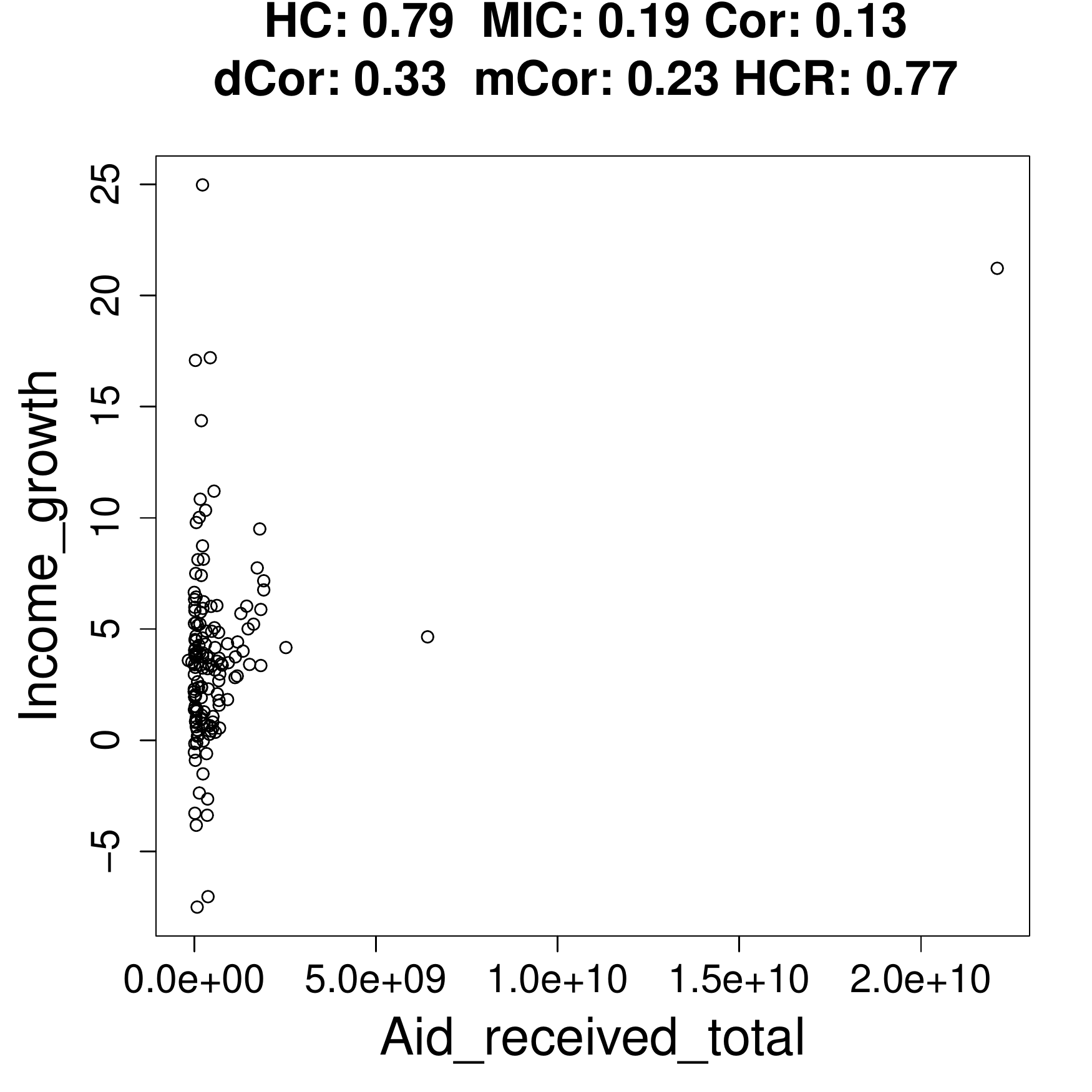} 
        \includegraphics[width=0.3\textwidth]{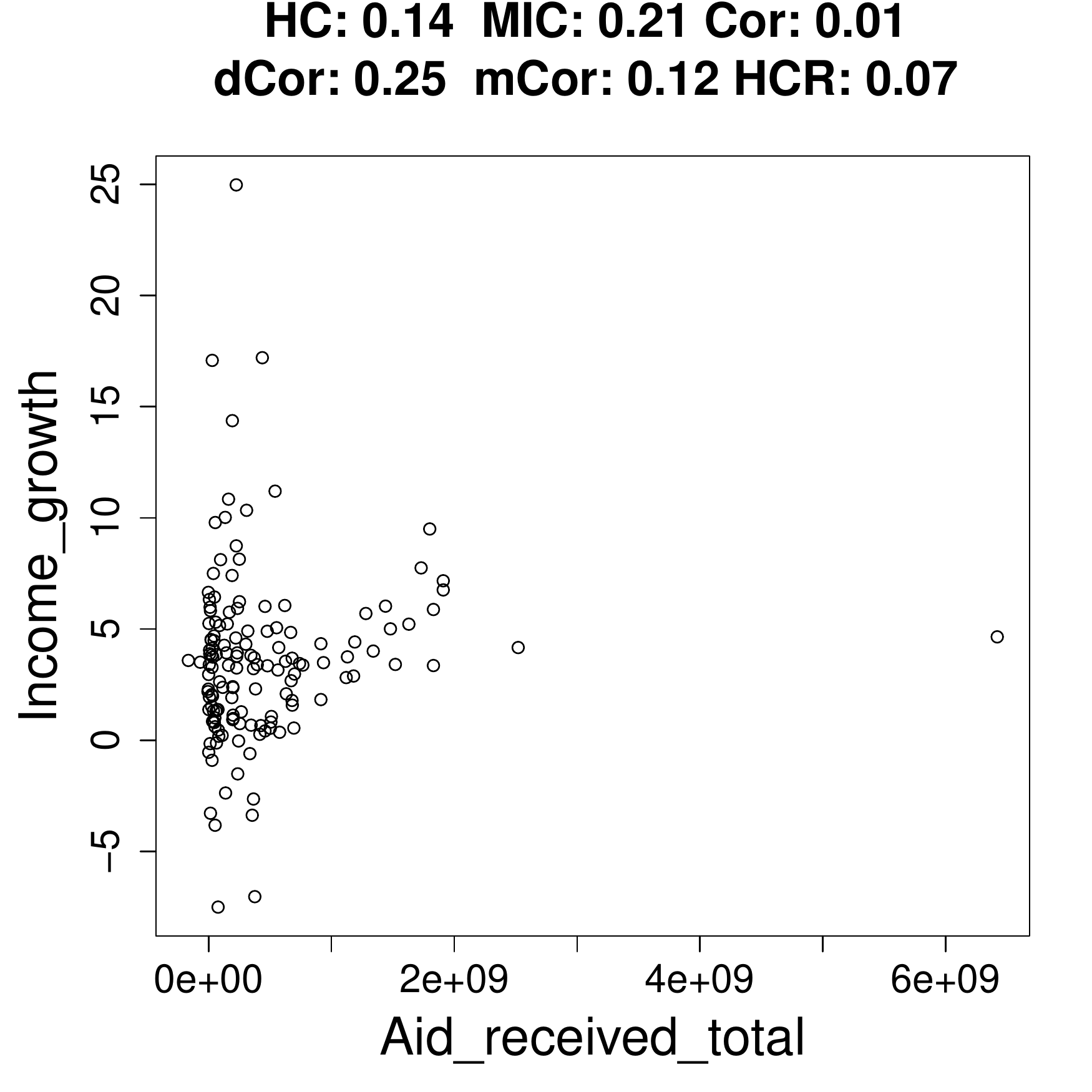}
        \includegraphics[width=0.3\textwidth]{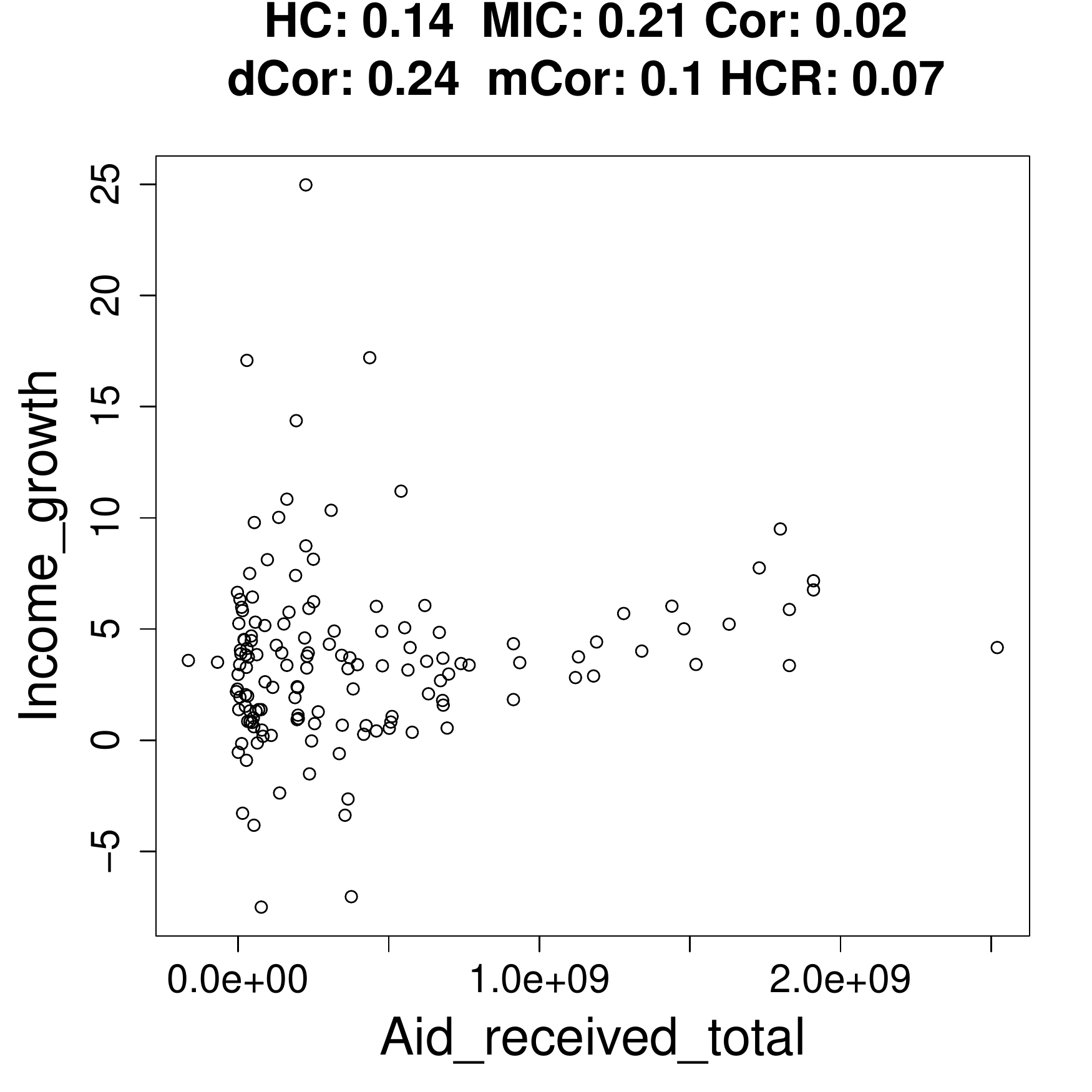}
     \caption{}
     \end{subfigure} 
     \caption{Samples for the pair of indicators shown in Figure~\ref{WHO-VS}-(E) from the entire WHO dataset (left), without one outlier (middle), and without two outliers (right).}\label{fig-repeat1}
\end{figure}

\begin{figure}[h]
        \begin{subfigure}[t]{\textwidth}
        \centering
        \includegraphics[width=0.3\textwidth]{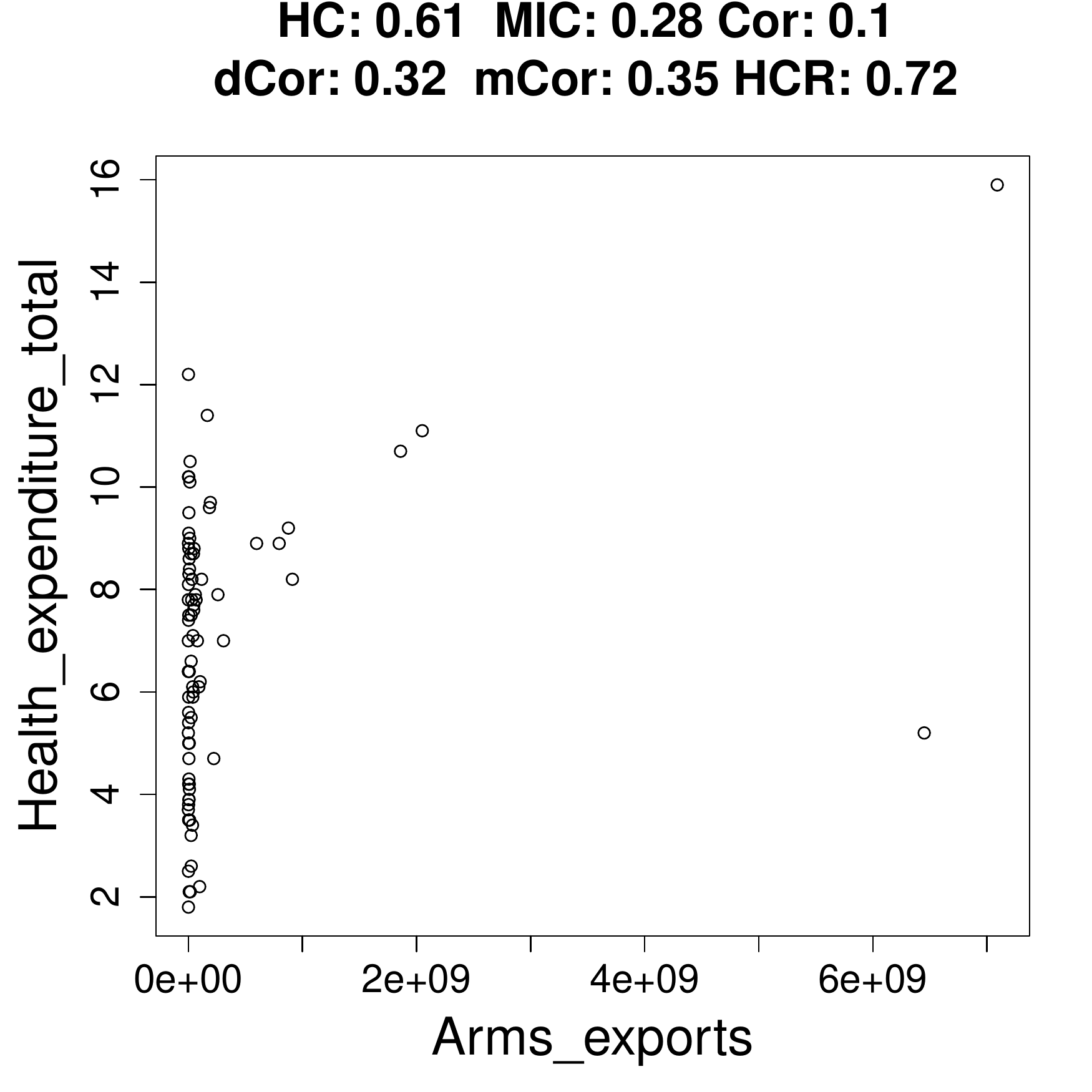} 
        \includegraphics[width=0.3\textwidth]{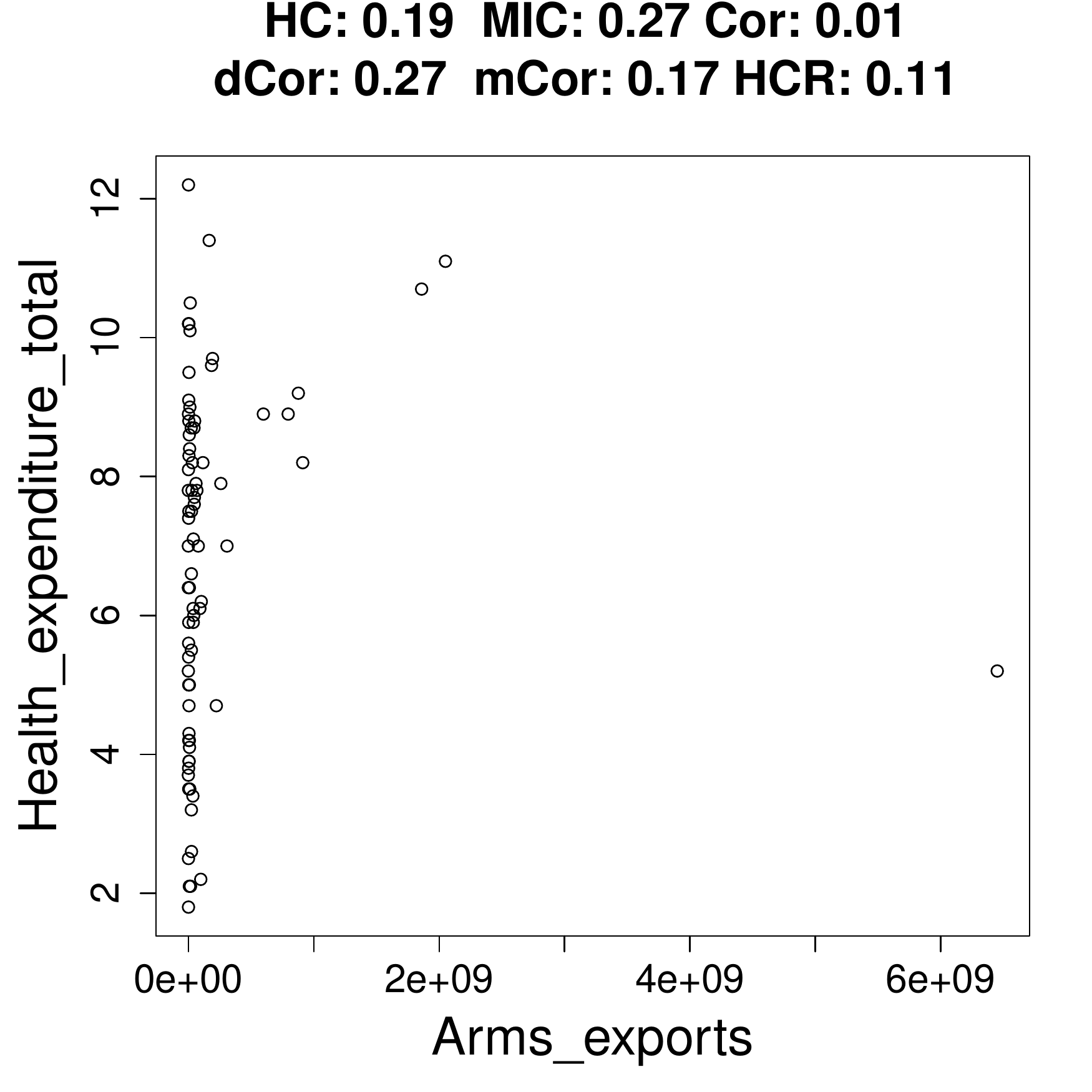}
        \includegraphics[width=0.3\textwidth]{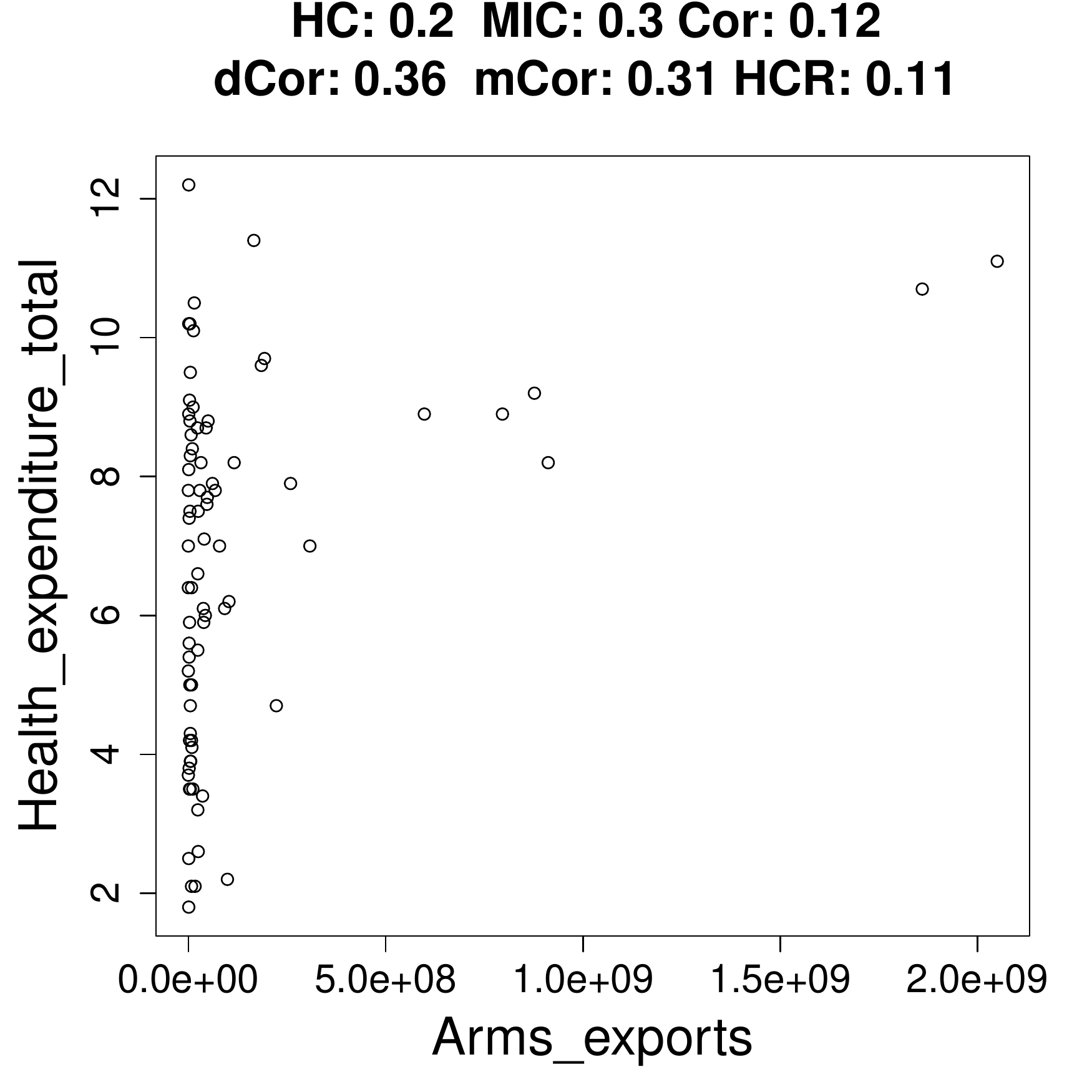}
        \caption{}
     \end{subfigure} 
  \caption{Samples for the pair of indicators shown in Figure~\ref{WHO-VS}-(F) from the entire WHO dataset (left), without one outlier (middle), and without two outliers (right).}\label{fig-repeat2}
\end{figure}

In Figure~\ref{more-good3} (left), 
the two countries with the largest number of male deaths from the colon and rectum cancer are China and United States. As China is removed from the dataset, in (middle), hypercontractivity remains unchanged. As we also remove United States, in (right), hypercontractivity becomes small, 0.17. This value is still larger than the typical coefficient for two independent indicators ($\approx 0.05$), we can see that hypercontractivity is more sensitive to the outlier than other correlation measures.  

 \begin{figure}[H]
        \centering
        \includegraphics[width=0.3\textwidth]{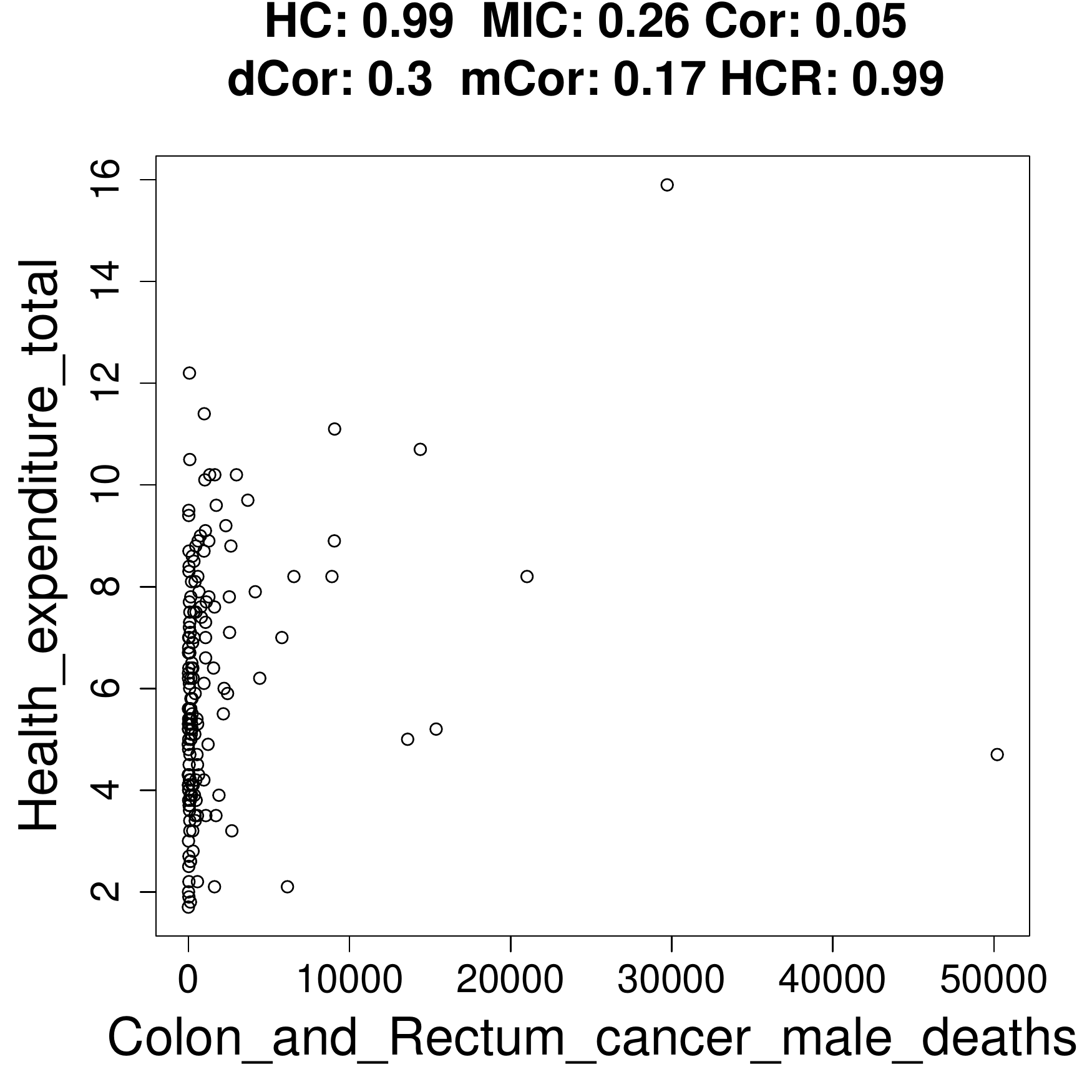}
        \includegraphics[width=0.3\textwidth]{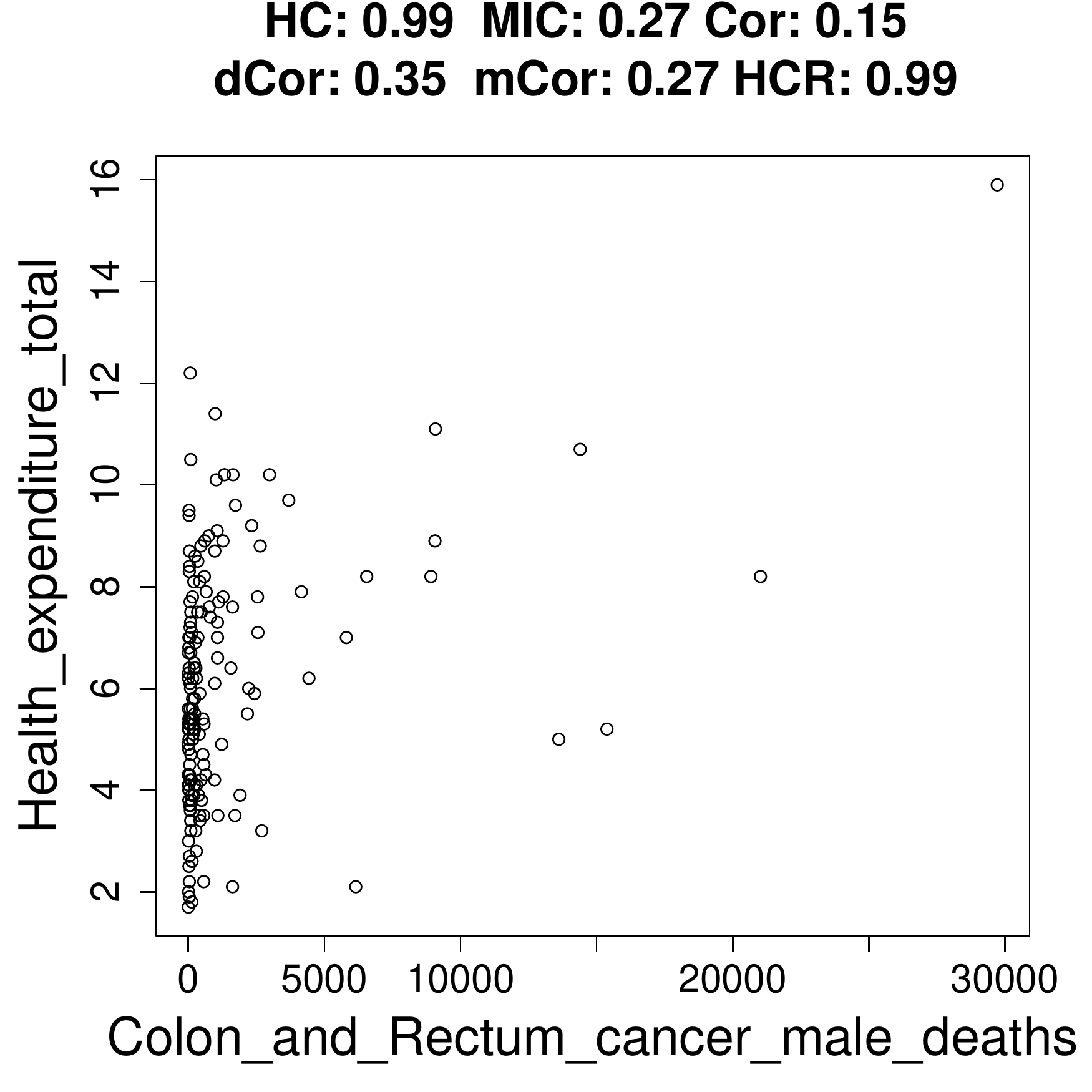}
        \includegraphics[width=0.3\textwidth]{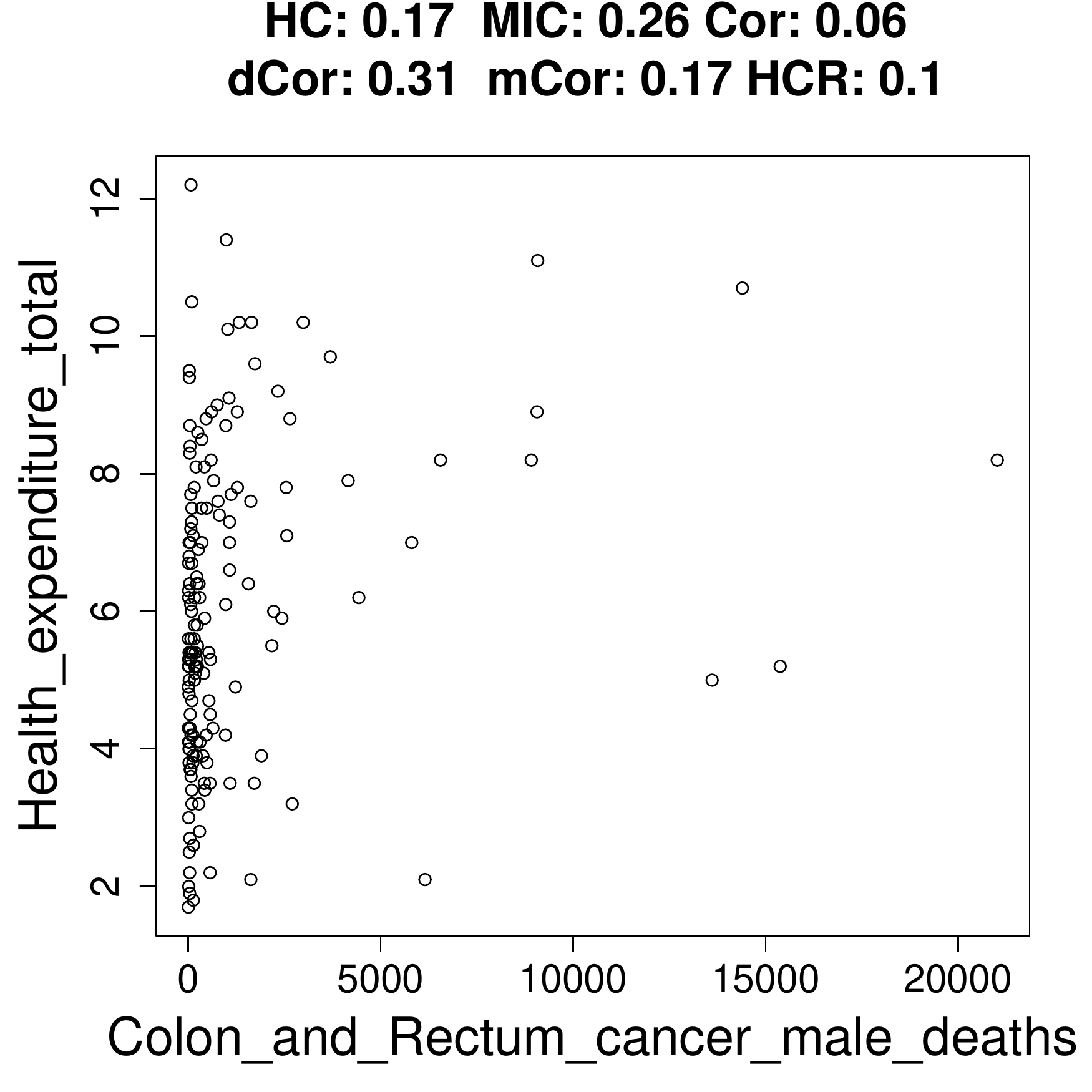}
                \caption{Samples for the pair of indicators shown in Figure~\ref{WHO-VS}-(G) from the entire WHO dataset (left), without one outlier (middle), and without two outliers (right).} \label{more-good3} 
    \end{figure}
    \vspace{-1em}

In Figure~\ref{more-good4},  
the two countries with the largest number of broadband subscribers are United States and China. When we remove United States from the samples, hypercontractivity becomes close to zero, which also shows hypercontractivity is sensitive to the outliers. 
 \begin{figure}[H]
        \centering
        \includegraphics[width=0.3\textwidth]{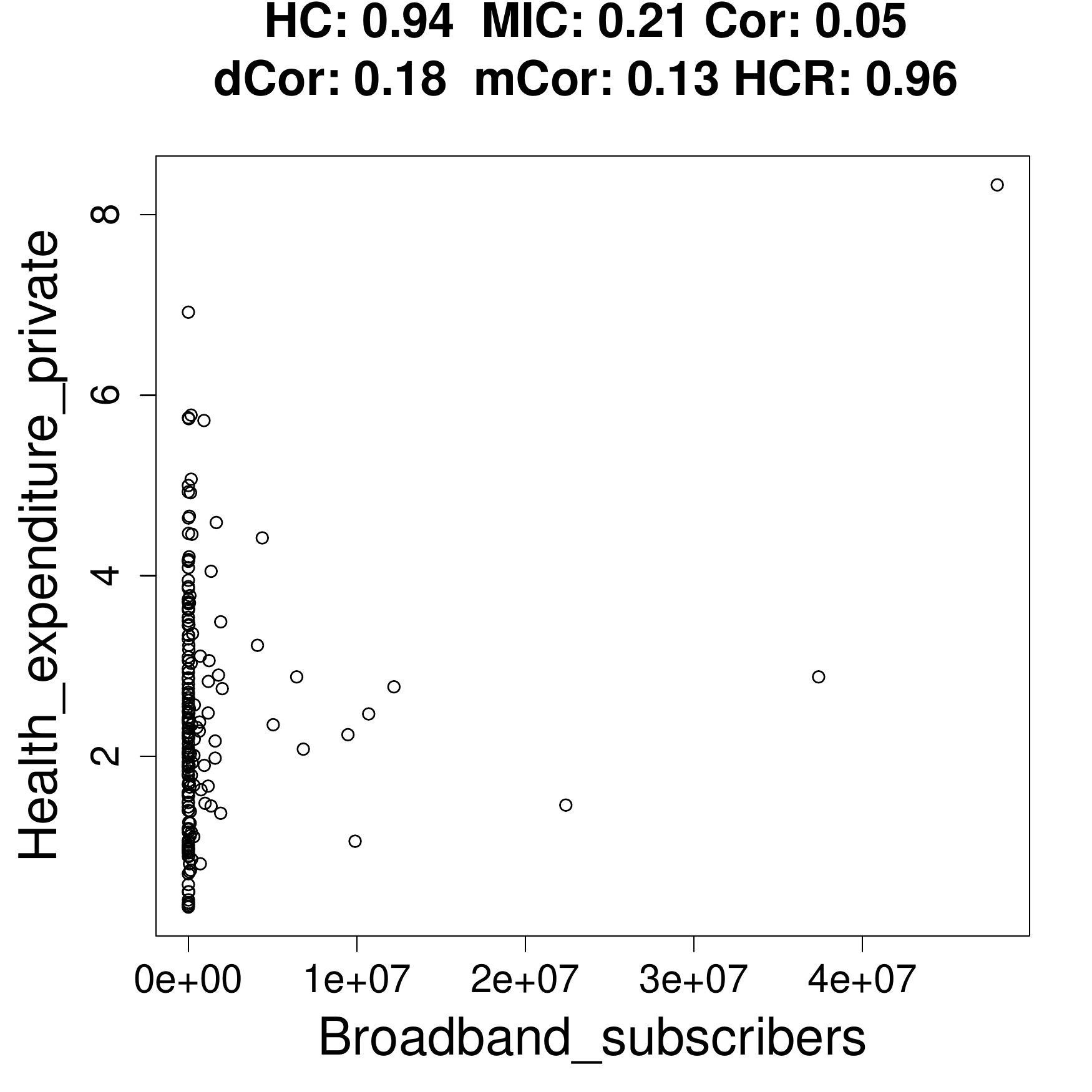}
        \includegraphics[width=0.3\textwidth]{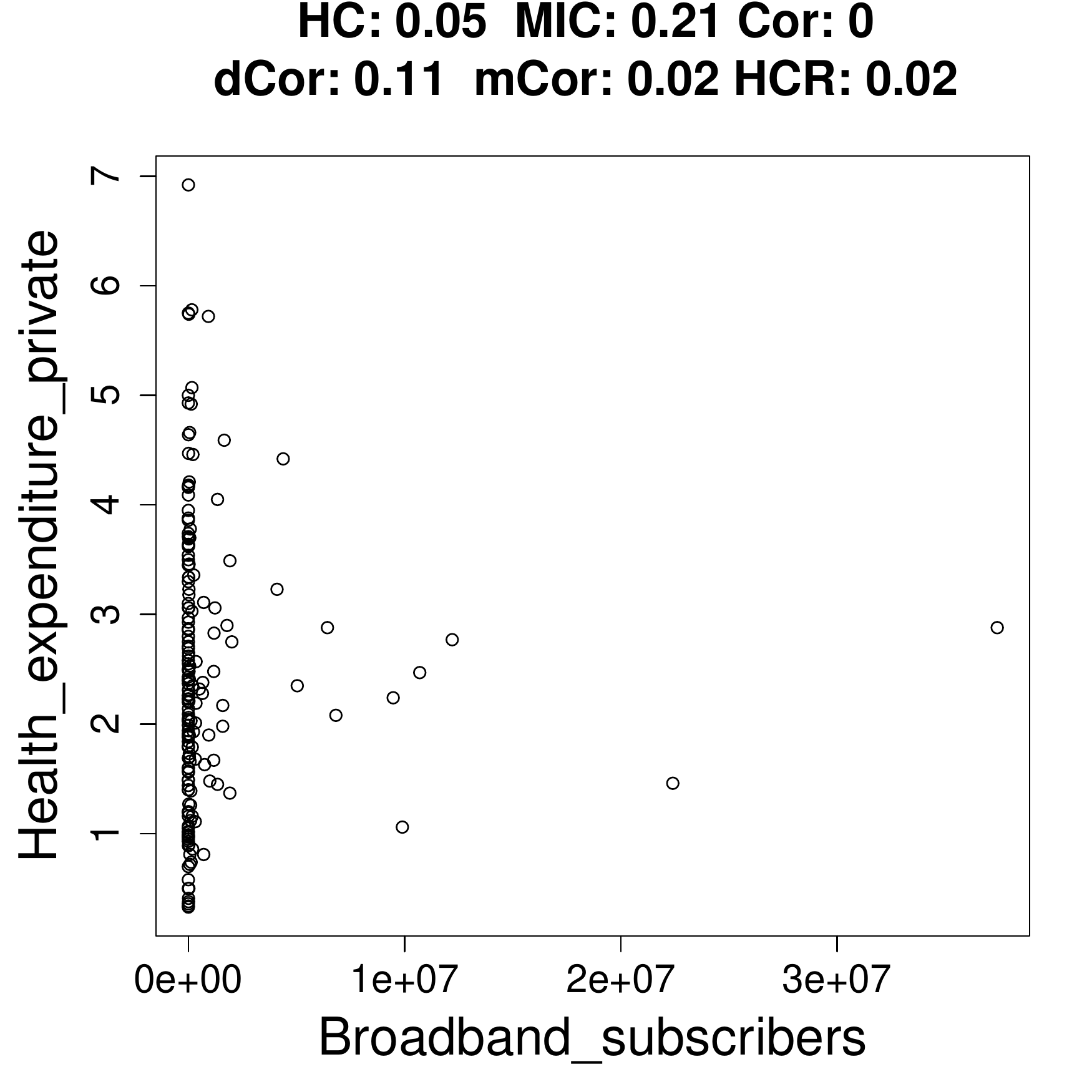}
        \includegraphics[width=0.3\textwidth]{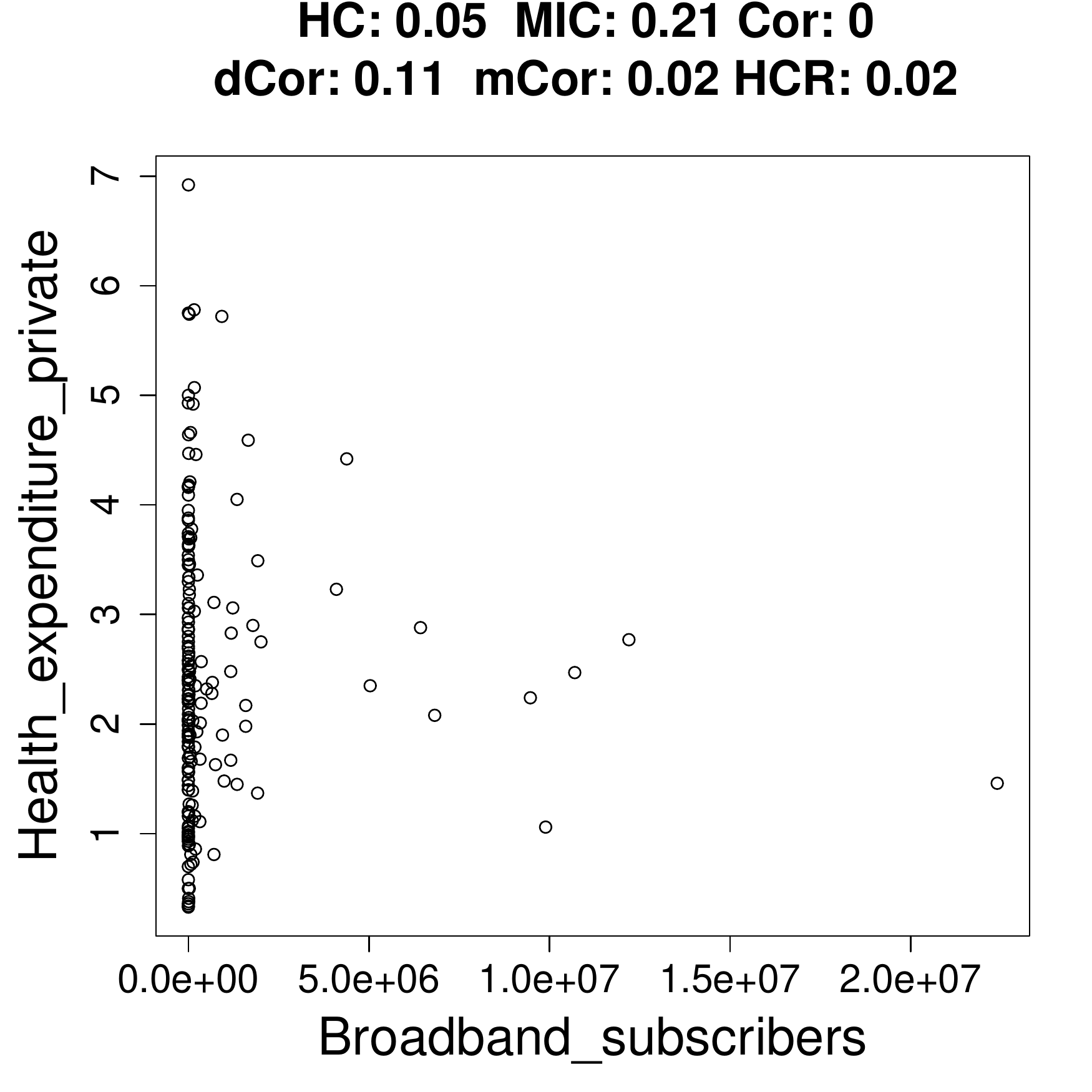}
                \caption{Samples for the pair of indicators shown in Figure~\ref{WHO-VS}-(H) from the entire WHO dataset (left), without one outlier (middle), and without two outliers (right).} \label{more-good4} 
    \end{figure}
\vspace{-1em}
    
In Figure~\ref{more-good1}, hypercontractivity remains large even after we remove outliers. 
The two countries with the largest amount of hydroelectricity consumption are Norway and Iceland. Even after we remove Norway from the samples, as shown in (middle), hypercontractivity remains large. As we further remove one outlier (Iceland) from the samples, as shown in (right), hypercontractivity becomes 0.49.  
\begin{figure}[H]     
        \centering
        \includegraphics[width=0.3\textwidth]{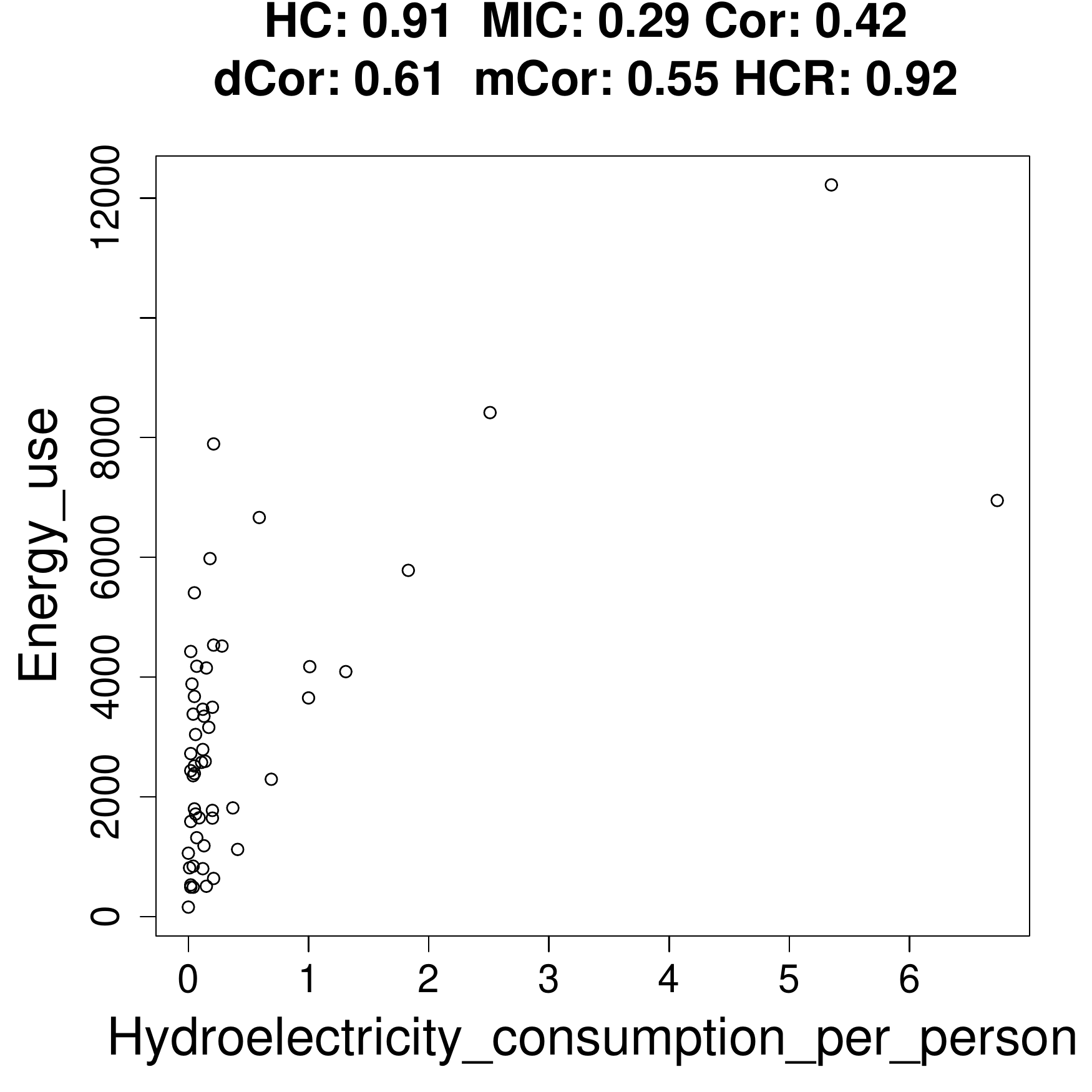} 
        \includegraphics[width=0.3\textwidth]{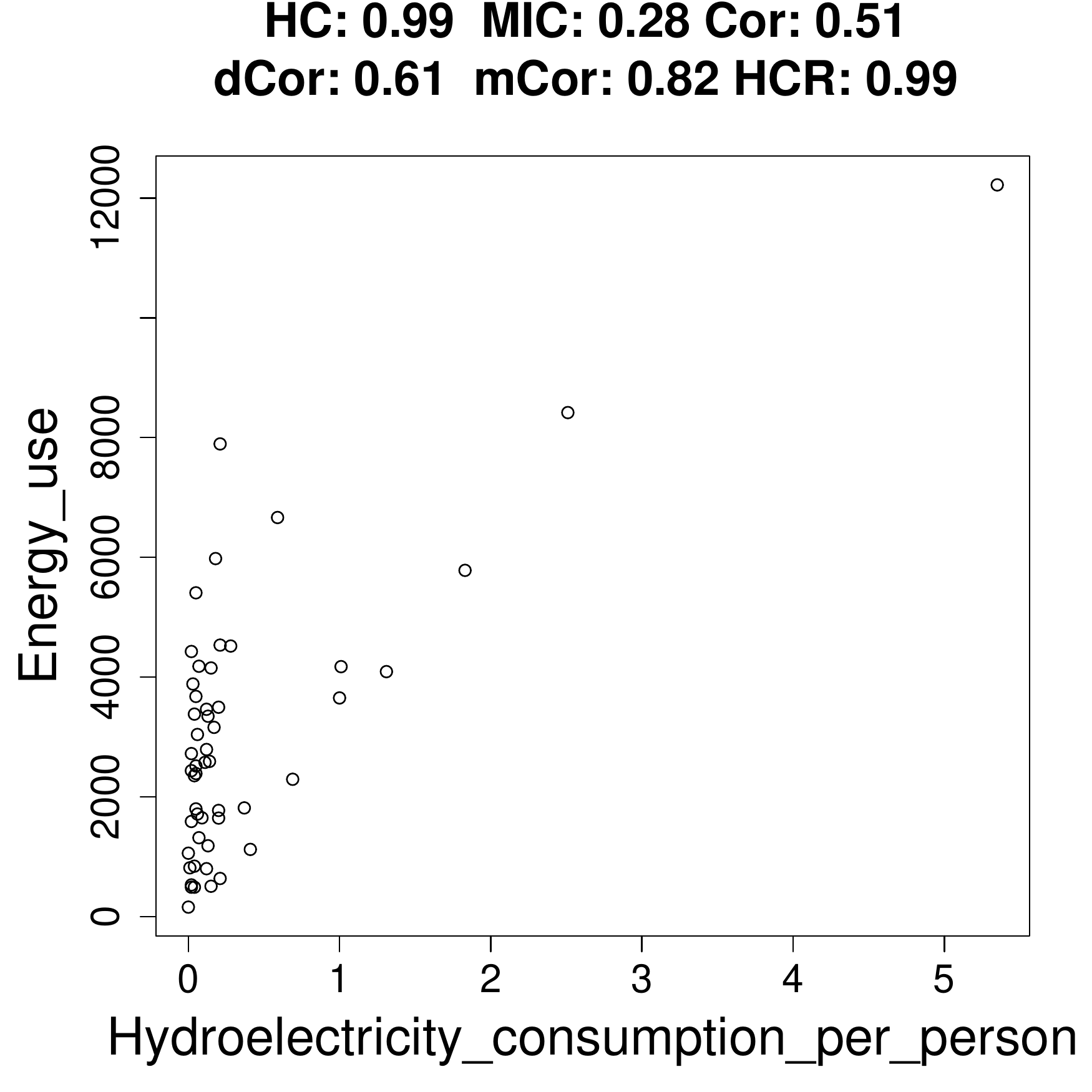} 
        \includegraphics[width=0.3\textwidth]{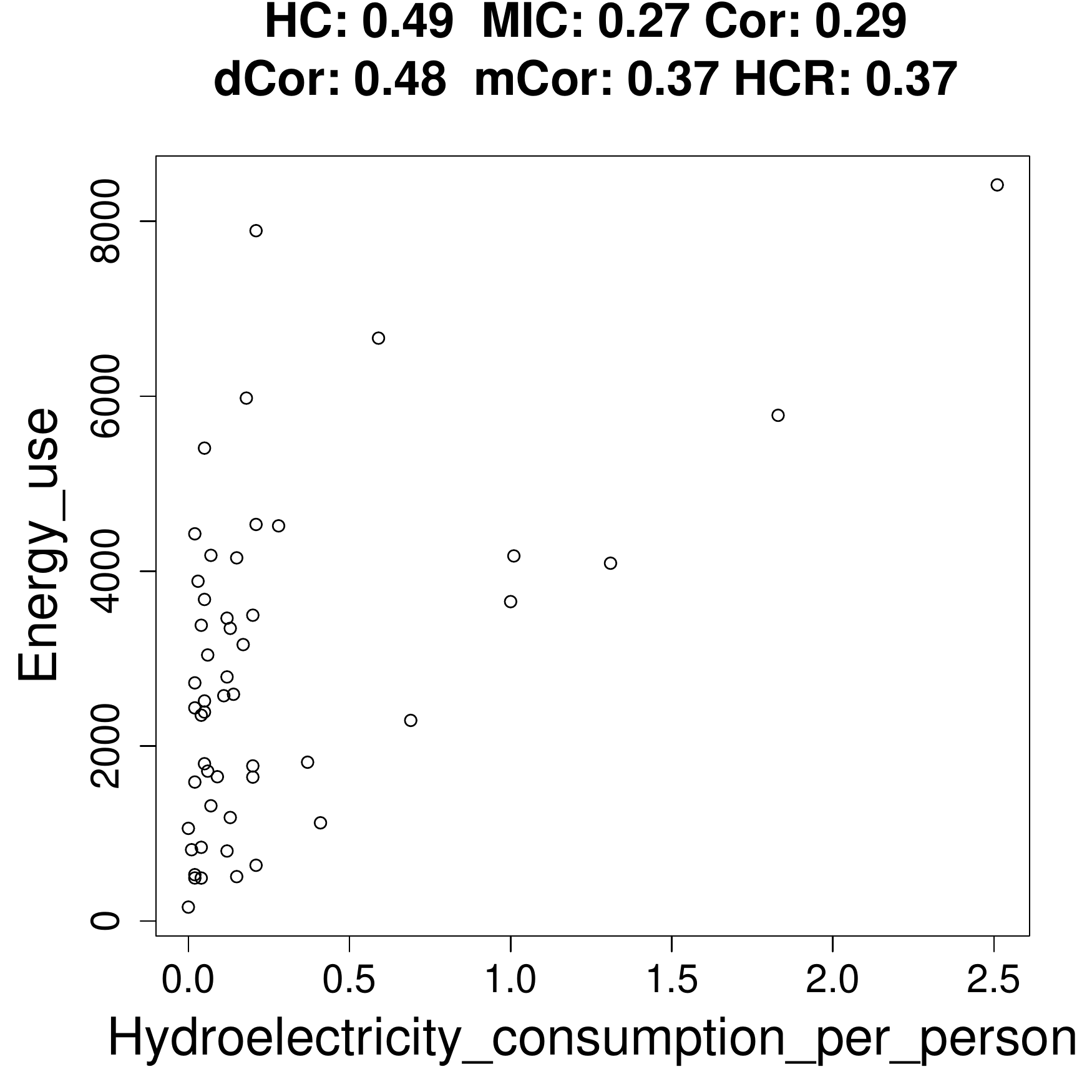} 
        \caption{Samples for the pair of indicators shown in Figure~\ref{WHO-VS}-(I) from the entire WHO dataset (left), without one outlier (middle), and without two outliers (right).} \label{more-good1}
\end{figure}

In Figure~\ref{more-good2}, (middle), all samples but Luxembourg is shown.  
We can see that most countries have a very small absolute amount of foreign direct investment net outflows (For 126 out of 157 countries, it is between $[-2,2]$), and for those countries, the foreign direct investment net outflow is independent of foreign direct investment net inflows. For the remaining countries, there is a positive association between the outflow and the inflow. Hypercontractivity captures this hidden correlation better than other correlations; hypercontractivity is 0.47, whereas MIC and Pearson correlation are small. If we further remove the rightmost sample, as shown in (right), hypercontractivity becomes small. 
 \begin{figure}[H]    
        \centering
        \includegraphics[width=0.3\textwidth]{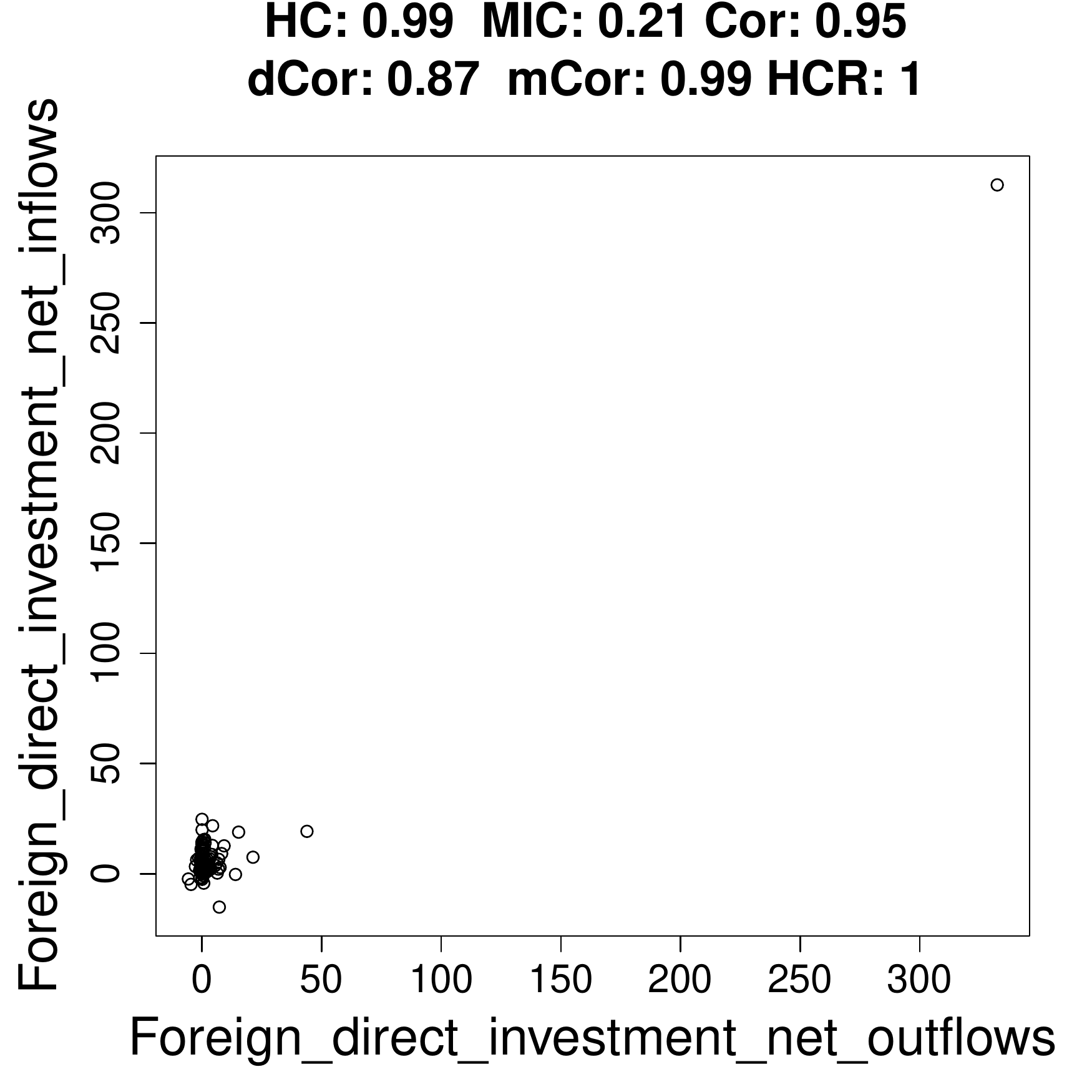}
        \includegraphics[width=0.3\textwidth]{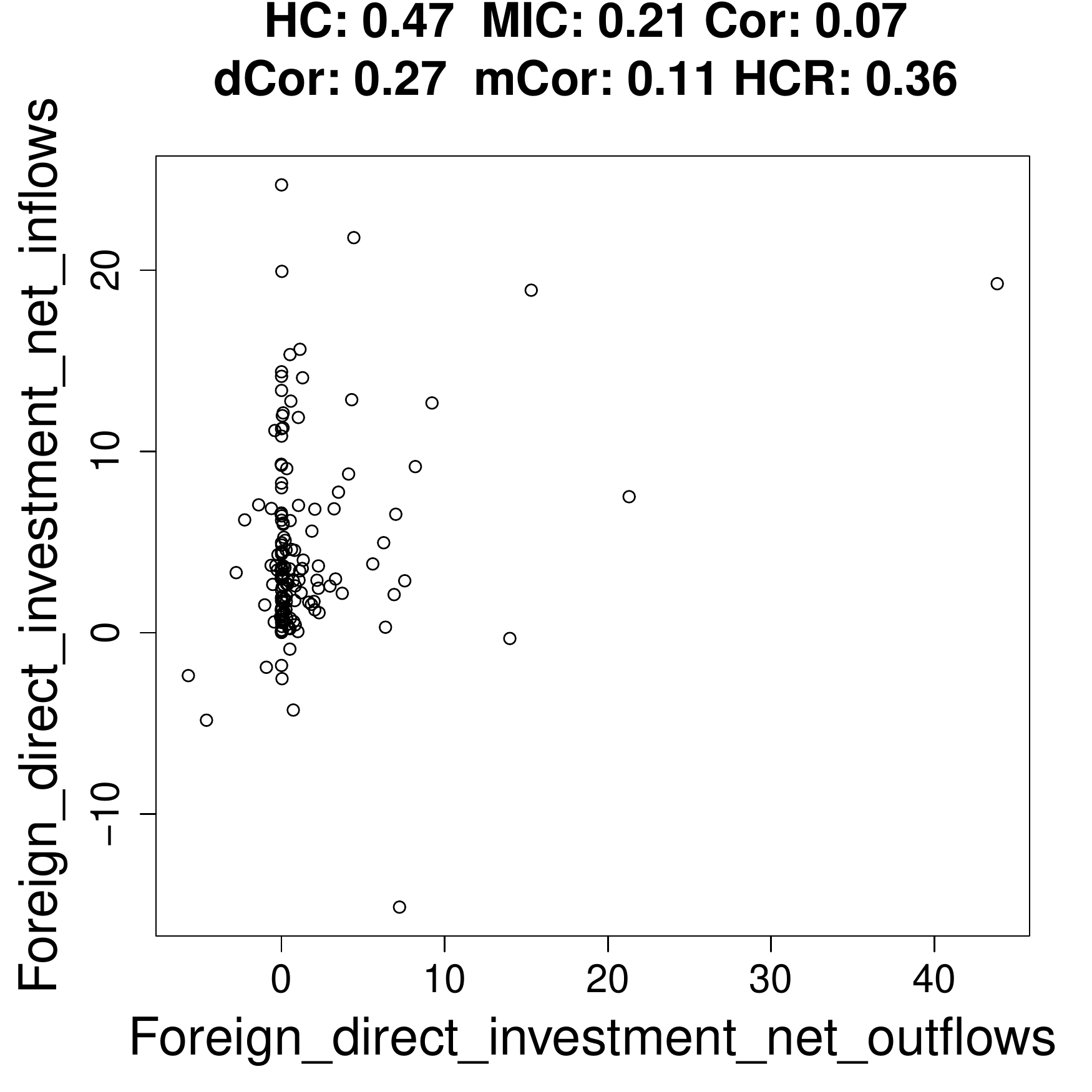}
        \includegraphics[width=0.3\textwidth]{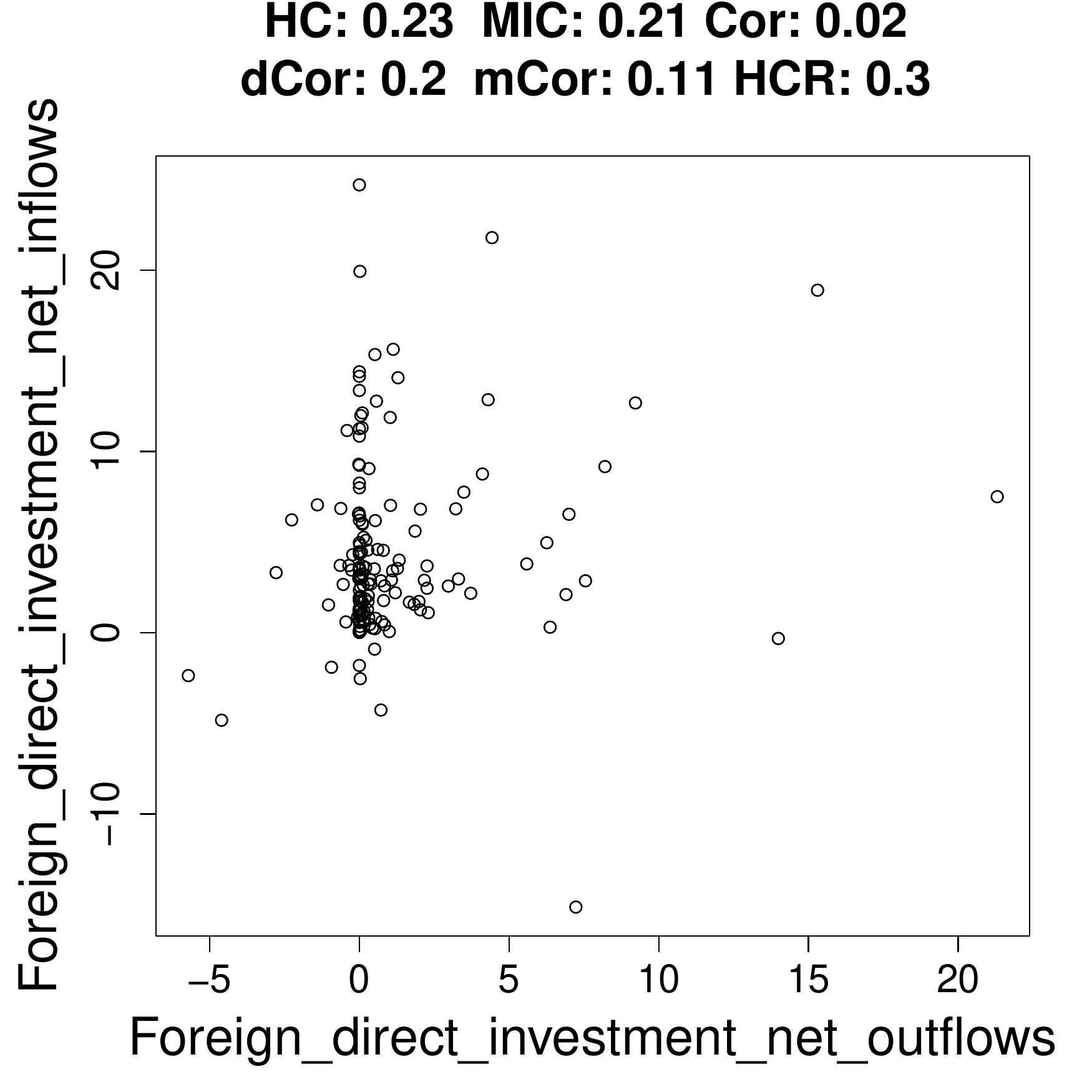}
                \caption{Samples for the pair of indicators shown in Figure~\ref{WHO-VS}-(J) from the entire WHO dataset (left), without one outlier (middle), and without two outliers (right).} \label{more-good2}
    \end{figure}

Whether we should consider a sample in a rare type as a meaningful sample or as an outlier depends on the application. 
If we use hypercontractivity to discover a pair of measures for which one variable can be potentially correlated with the other, then we would expect to discover that 
an aid for a country has potential correlation in the income growth. 
Other measures will fail.  
It is possible that hypercontractivity might have larger false positive rate, 
and depending on the application, one might prefer to error on the side of having more 
positive cases to be screened by further experiments, surveys, or human judgements.

\subsubsection{Hypercontractivity detecting an outlier}
In Figure~\ref{hc-outlier} (A) and (B), we show examples of pairs of indicators for which there is one  outlier and the remaining samples are independent, but hypercontractivity is large. 
As shown in Figure~\ref{hc-outlier} (A) and (B) (left), hypercontractivity is close to 1, when there is an outlier. As shown in (right), hypercontractivity is close to 0, when the outlier is removed. This implies that one single outlier can make the hypercontractivity large. We can see similar patterns for other correlation measures, such as for Pearson correlation, distance correlation, and maximal correlation for both (A) and (B), and MIC for (B), but are less sensitive than hypercontractivity. 


    \begin{figure} [!ht]
            \begin{subfigure}[t]{\textwidth}
                    \centering
        \includegraphics[width=0.3\textwidth]{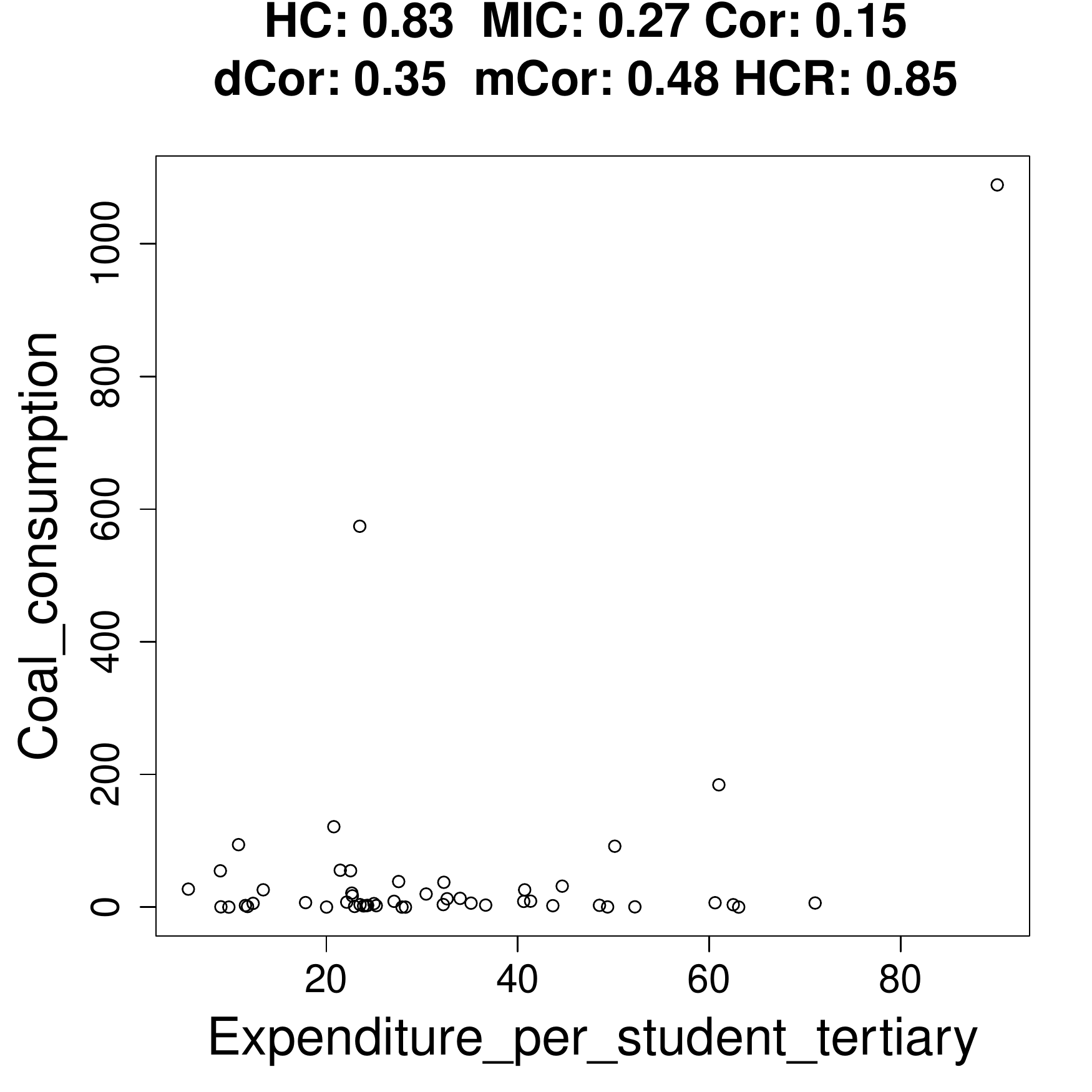} 
        \includegraphics[width=0.3\textwidth]{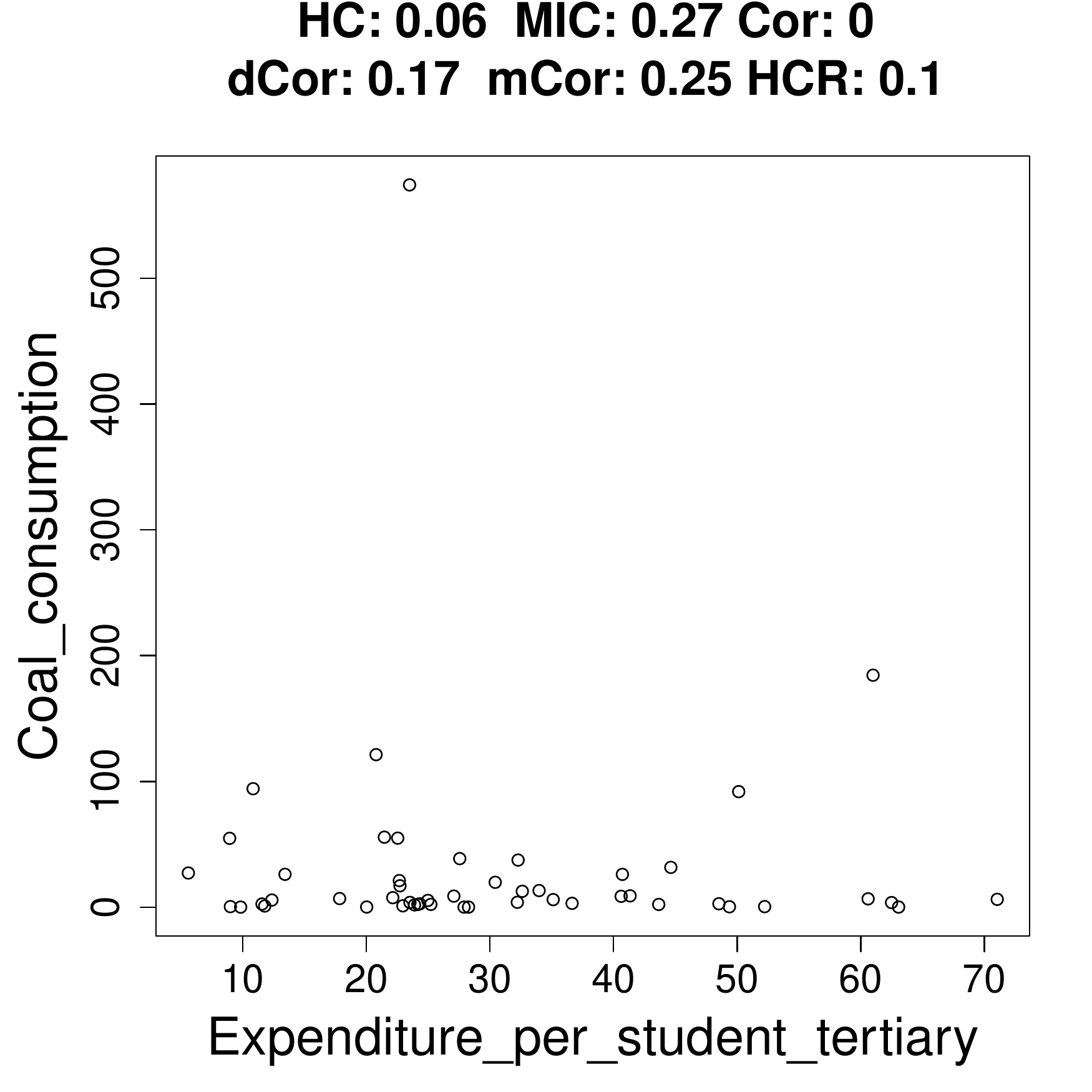} 
        \caption{}
    \end{subfigure} 
        \begin{subfigure}[t]{\textwidth}
        \centering
        \includegraphics[width=0.3\textwidth]{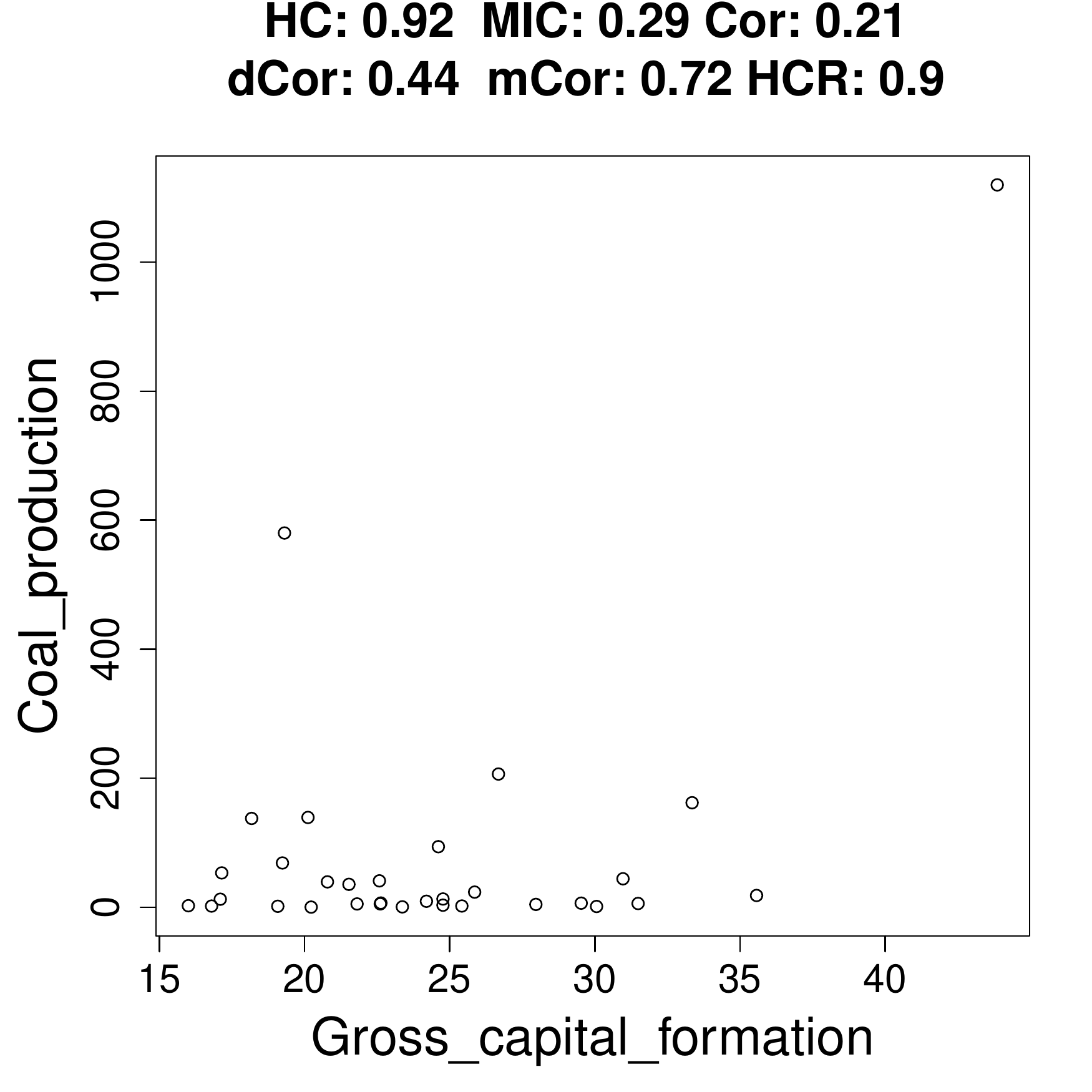}
        \includegraphics[width=0.3\textwidth]{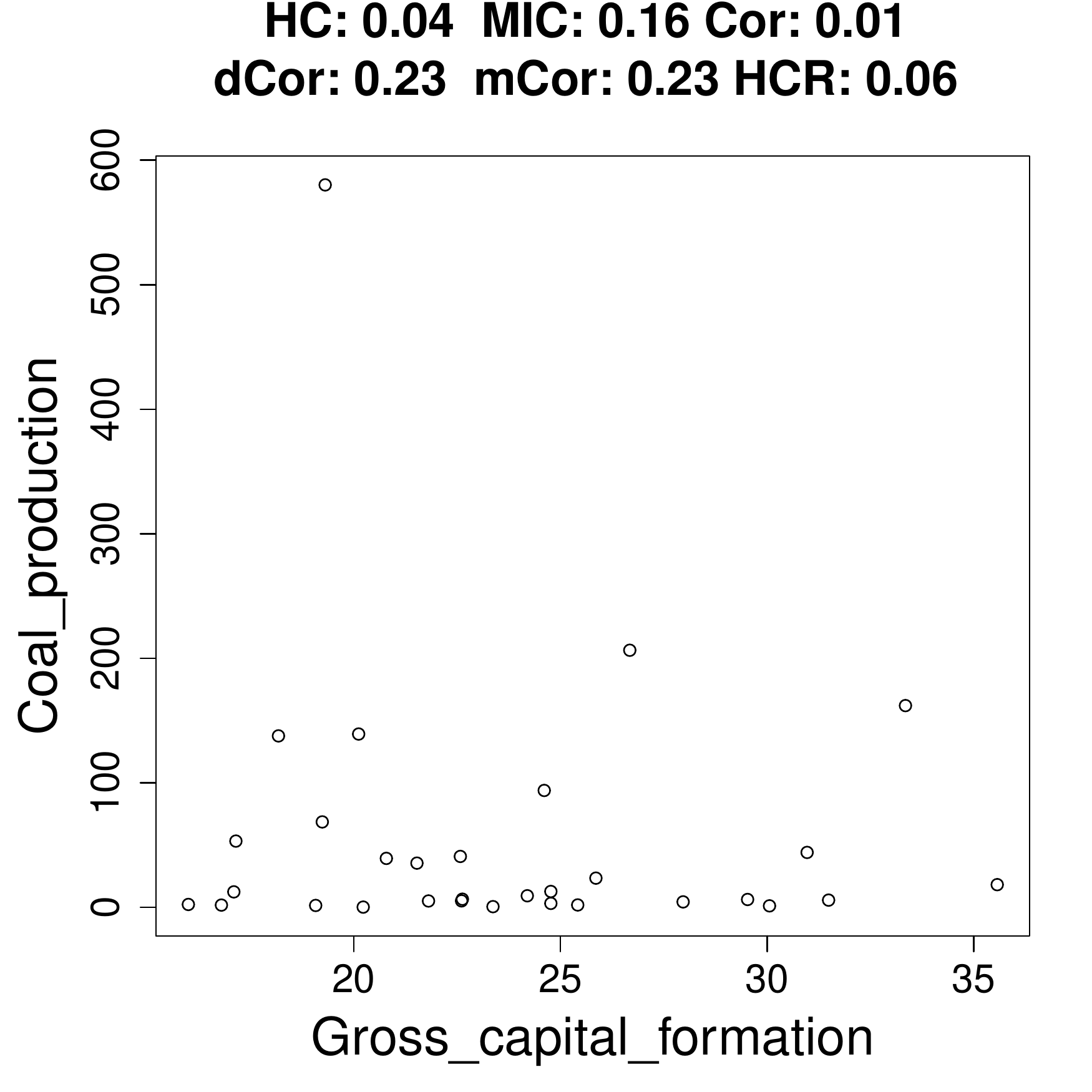}
        \caption{}
        \end{subfigure}
        \caption{ Hypercontractivity and other correlation measures become smaller as we remove an outlier.}\label{hc-outlier}
\end{figure}

To further study how hypercontractivity estimator is affected by outliers, we ran simulations on synthetic data. We generated three sets of synthetic data shown in Figure~\ref{outlier} and computed hypercontractivity coefficients. In Figure~\ref{outlier} (left), an outlier is located far from the rest of samples, and~the estimated hypercontractivity coefficient is 0.99. In Figure~\ref{outlier} (middle), an outlier is located close to the rest of samples, and  the estimated hypercontractivity coefficient is 0.04. In Figure~\ref{outlier}~(right), $X$~and $Y$ are potentially correlated, and the hypercontractivity estimate is 0.17. As can bee seen from this simulation and experimental results on WHO dataset, our hypercontractivity estimator is sensitive to outliers. If one wants to filter out the effect of outliers, one can combine methods for robust estimation, such as trimming and winsorizing~\cite{hastings1947,McBean1998,Rustum2007}, along with the hypercontractivity estimator. This is an~interesting future research direction.

 \begin{figure}[!ht]    
        \centering
        \includegraphics[width=0.32\textwidth]{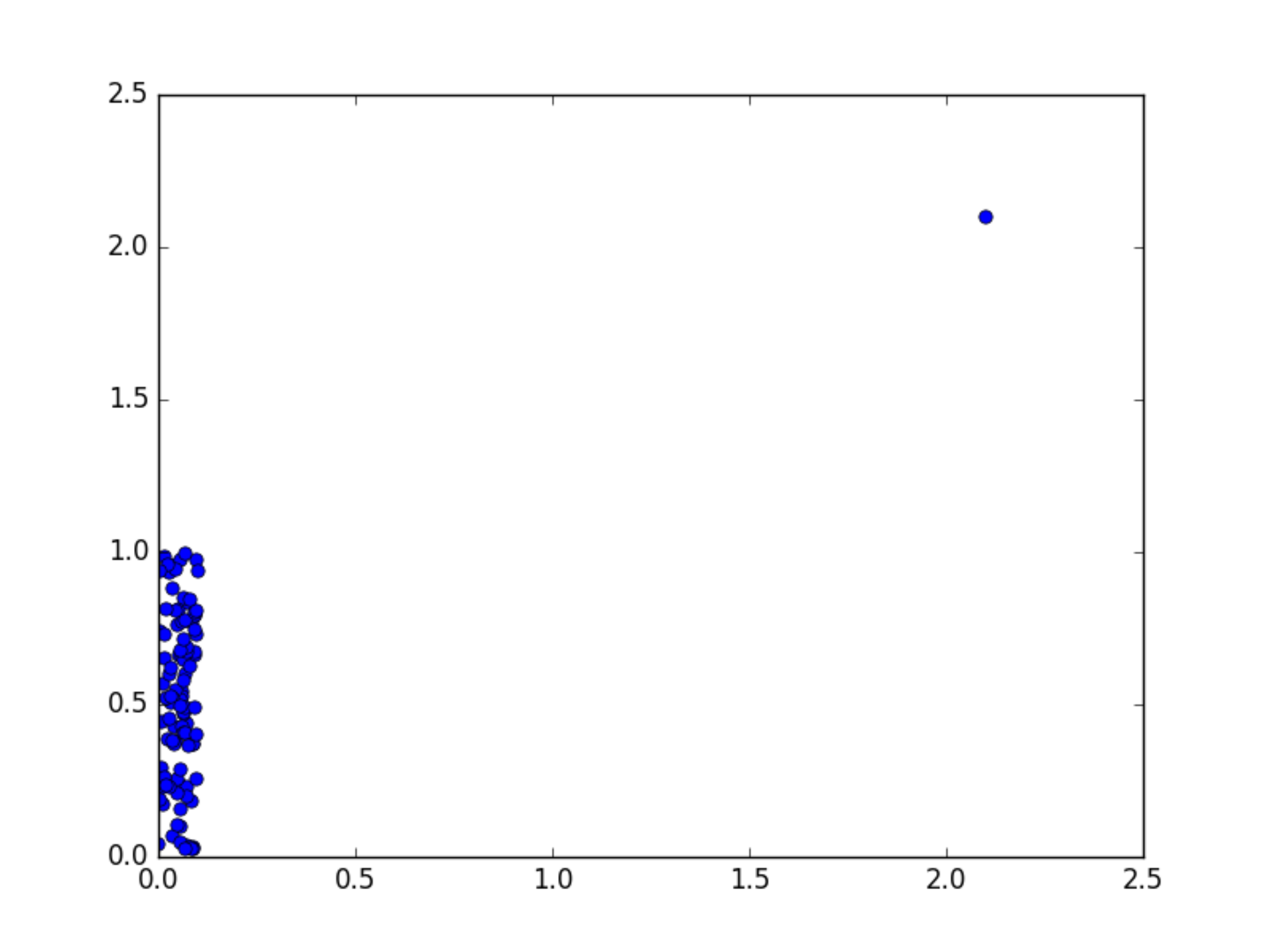}
        \includegraphics[width=0.32\textwidth]{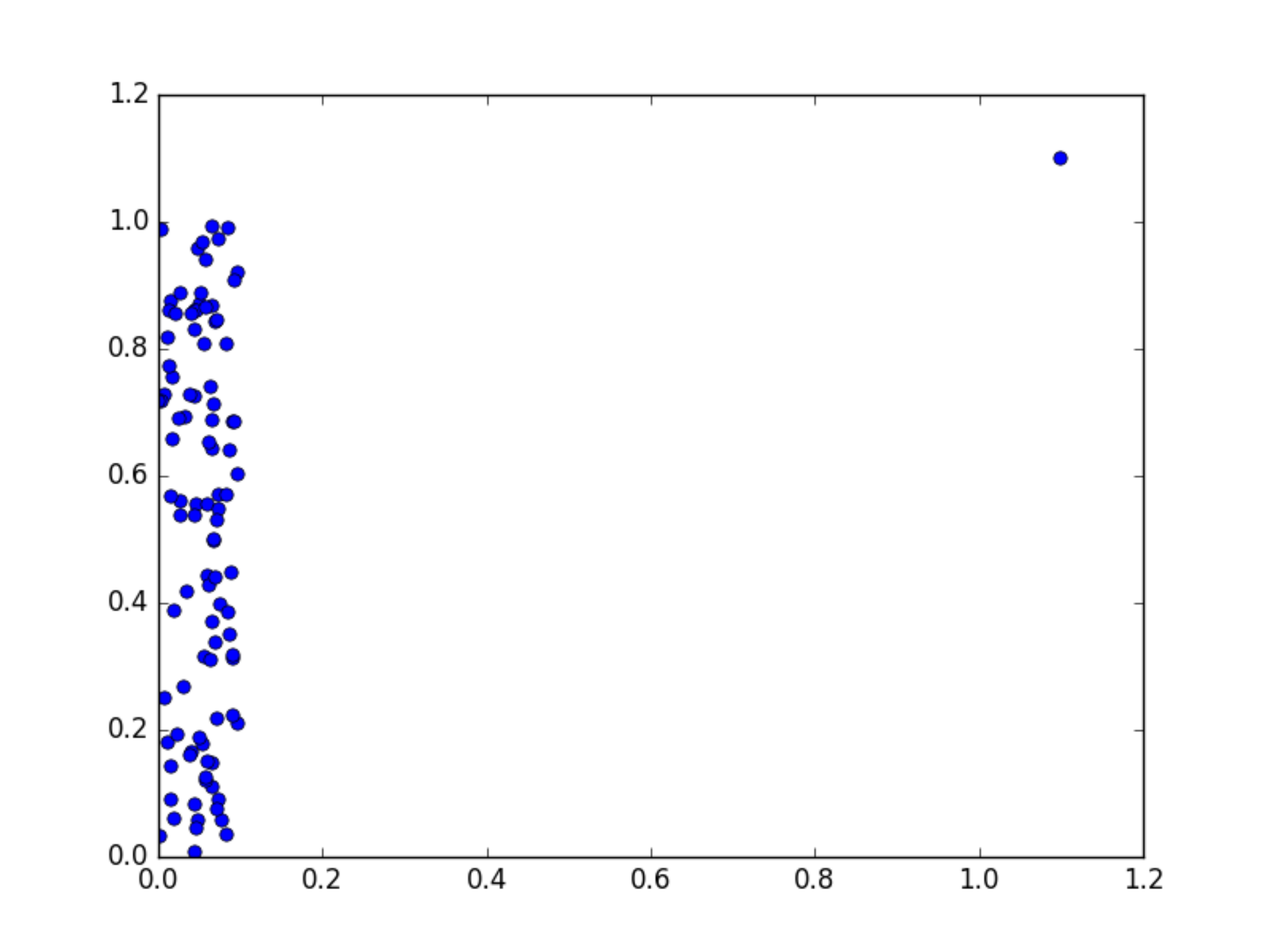}
        \includegraphics[width=0.32\textwidth]{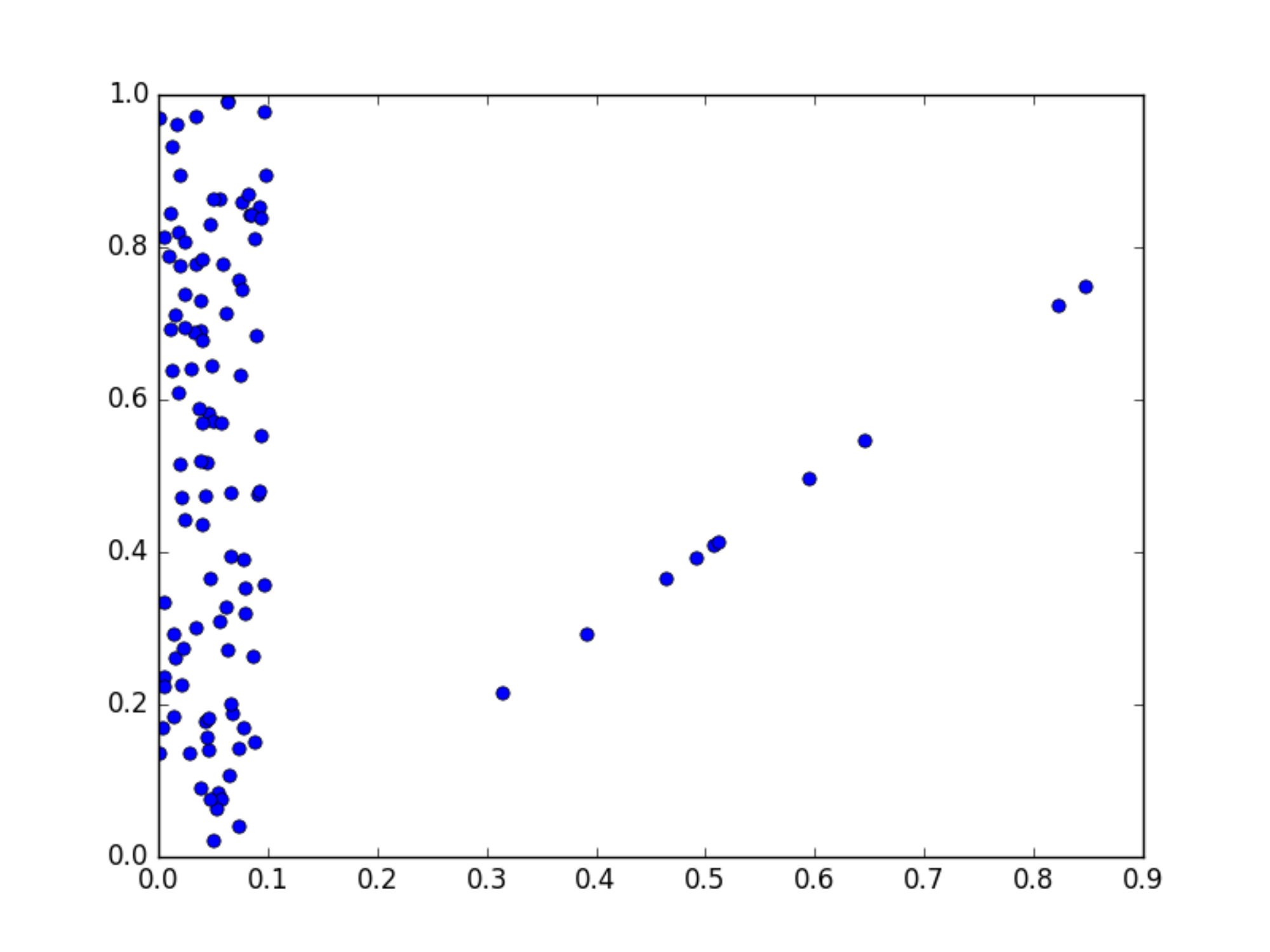}
        \put(-218,-5){\small X}
	\put(-283,45){\small Y}
        \put(-365,-5){\small X}
	\put(-425,45){\small Y}
        \put(-73,-5){\small X}
	\put(-138,45){\small Y}
             
                \caption{{Synthetic data}: (Left) an outlier is located far from other samples; (Middle) an outlier is located close to the rest of samples; (Right) potential correlation exists.} \label{outlier}
    \end{figure}
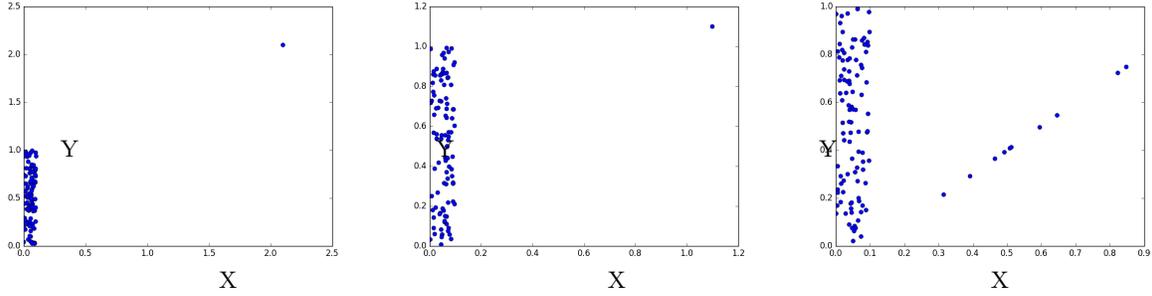

\subsection{Gene pathway recovery from single cell data}
\label{sec:dremi}

We replicate the genetic pathway detection experiment from~\cite{Krishnaswamy14}, 
and show that  hypercontractivity correctly discovers the genetic pathways from smaller number of samples. 
A genetic pathway  is a series of genes interacting with each other as a chain. 
Consider the following setup where four genes whose expression values in a single cell are modeled by random processes $X_t$, $Y_t$, $Z_t$ and $W_t$ respectively. 
These 4 genes interact with each other following a pathway $X_t \rightarrow Y_t \rightarrow Z_t \rightarrow W_t$; 
 it is biologically known that $X_t$ causes $Y_t$ with a negligible delay, and later at time $t'$, $Y_{t'}$ causes $Z_{t'}$, and so on. Our goal is to recover this known gene pathway from sampled datapoints. 
For a sequence  of time points $\{t_i\}_{i=0}^m$, 
we observe $n_i$ i.i.d. samples $\{X_{t_i}^{(j)}, Y_{t_i}^{(j)}, Z_{t_i}^{(j)}, W_{t_i}^{(j)}\}_{j=1}^{n_i}$ generated from the random process $P(X_{t_i}, Y_{t_i}, Z_{t_i}, W_{t_i})$. 
We use the real data obtained by the single-cell mass flow cytometry technique~\cite{Krishnaswamy14}.

Given these samples from  time series, the goal of \cite{Krishnaswamy14} is to recover the direction of the interaction along the known pathway using correlation measures as follows, where they proposed a new measure called DREMI. 
 The DREMI correlation measure is evaluated on each pairs on the pathway,  
 $\tau(X_{t_i}, Y_{t_i})$, $\tau(Y_{t_i}, Z_{t_i})$ and $\tau(Z_{t_i}, W_{t_i})$, at each time points $t_i$. 
 It is declared that a genetic pathway is correctly recovered if 
 the peak of correlation follows the expected trend: 
 $\arg\max_{t_i} \tau(X_{t_i}, Y_{t_i}) \le \arg\max_{t_i}\tau(Y_{t_i}, Z_{t_i})  \le \arg\max_{t_i} \tau(Z_{t_i}, W_{t_i})$. 
In~\cite{Gao16}, the same experiment has been done with $\tau$ evaluated by UMI and CMI estimators. In this paper, we evaluate $\tau$ using our proposed estimator of hypercontractivity.



The Figure~\ref{fig:fcs_data} shows the scatter plots pCD3$\zeta$-pSLP76-pERK-pS6 chain at different time points after TCR activation. The data comes from CD4+ na{\"i}ve T lymphocytes from B6 mice with CD3, CD28, and CD4 cross-linking. Each row represents a pair of data in the chain, and each column stands for a time point after TCR activation. Estimate of hypercontractivity is shown below the scatter plot for each pair of data and each time point and we highlight the time point where each pair of data is maximally correlated. We can see that the peak of the correlation of pCD3$\zeta$-pSLP76, pSLP76-pERK and pERK-pS6 appears at 0.5 min, 1 min and 2 min respectively, hence the pathway is correctly identified. In Figure~\ref{fig:fcs_effmem}, the similar plots was shown for T-cells exposed with an antigen. Similarly, hypercontractivity is able to capture the trend.

We subsample the raw data from~\cite{Krishnaswamy14} to evaluate the ability to find the trend from smaller samples. Precisely, given a resampling rate $\gamma \in (0,1]$, we randomly select a subset of indices $S_i \subseteq [n_i]$ with ${\rm card}(S_i) = \lceil \gamma n_i \rceil$, compute $\tau(X_{t_i}, Y_{t_i})$, $\tau(Y_{t_i}, Z_{t_i})$ and $\tau(Z_{t_i}, W_{t_i})$ from subsamples $\{X_{t_i}^{(j)}, Y_{t_i}^{(j)}, Z_{t_i}^{(j)}, W_{t_i}^{(j)}\}_{j \in S_i}$, and determine whether we can recover the trend successfully, i.e., whether $\arg\max_{t_i} \tau(X_{t_i}, Y_{t_i}) \le \arg\max_{t_i}\tau(Y_{t_i}, Z_{t_i})  \le \arg\max_{t_i} \tau(Z_{t_i}, W_{t_i})$.  We repeat the experiment several times with independent subsamples and compute the probability of successfully recovering the trend.
Figure~\ref{fig:fcs_experiment} illustrates that when the entire dataset is available, all methods are able to recover the trend correctly. When only fewer samples are available,  hypercontractivity improves upon other competing measures in recovering the hidden chronological order of interactions of the pathway. For completeness, we run datasets for both regular T-cells (shown in left figure) and T-cells exposed with an antigen (shown right figure), for which we expect distinct biological trends. Hypercontractivity method can capture the trend for both datasets correctly and sample-efficiently.

\begin{figure}[!ht]
\begin{center}
\includegraphics{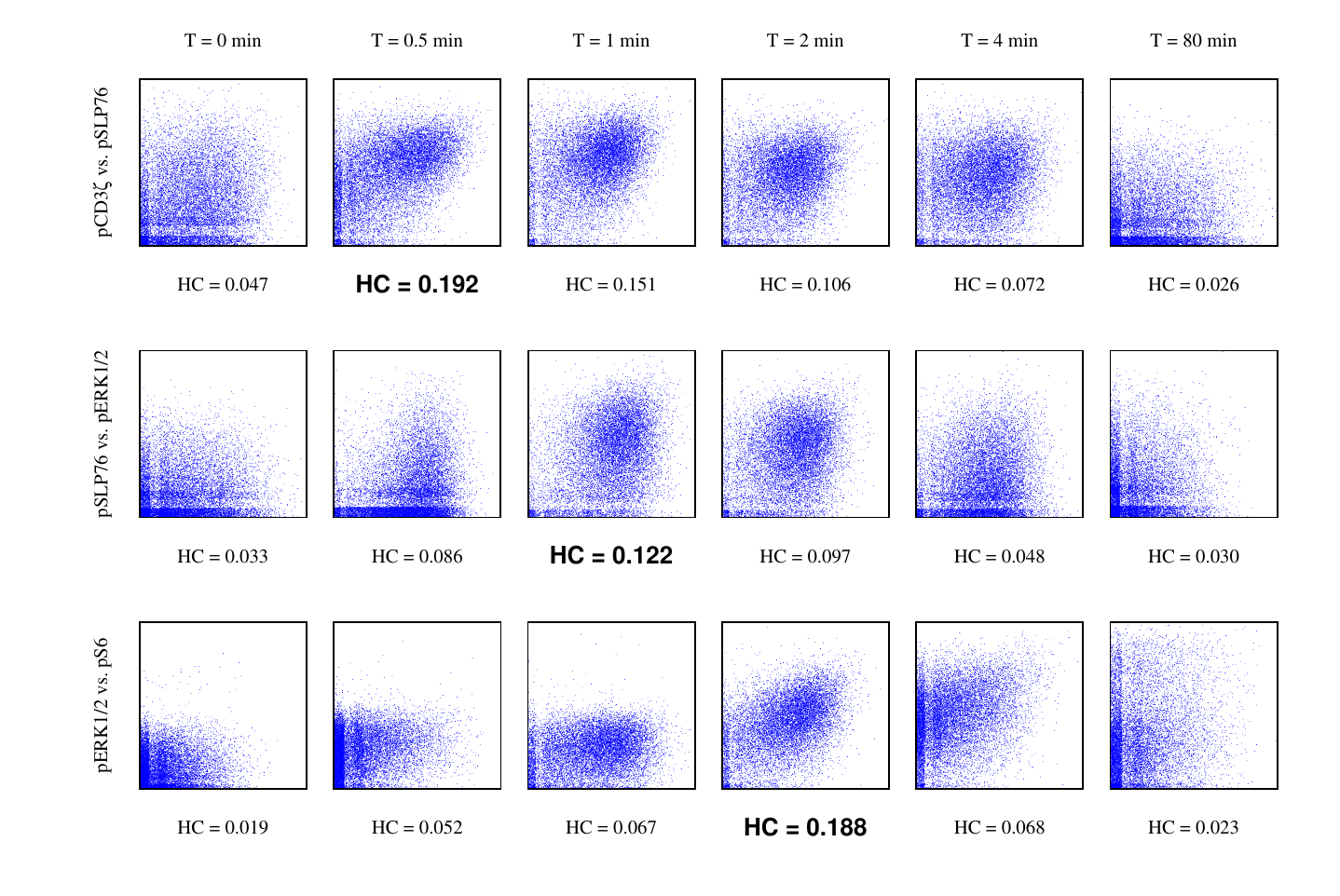}
\end{center}
\caption{Scatter plots of gene pathway data for various pair of data and various time points (regular T-cells).}\label{fig:fcs_data}
\vspace{-1em}
\end{figure}

\begin{figure}[H]
\begin{center}
\includegraphics{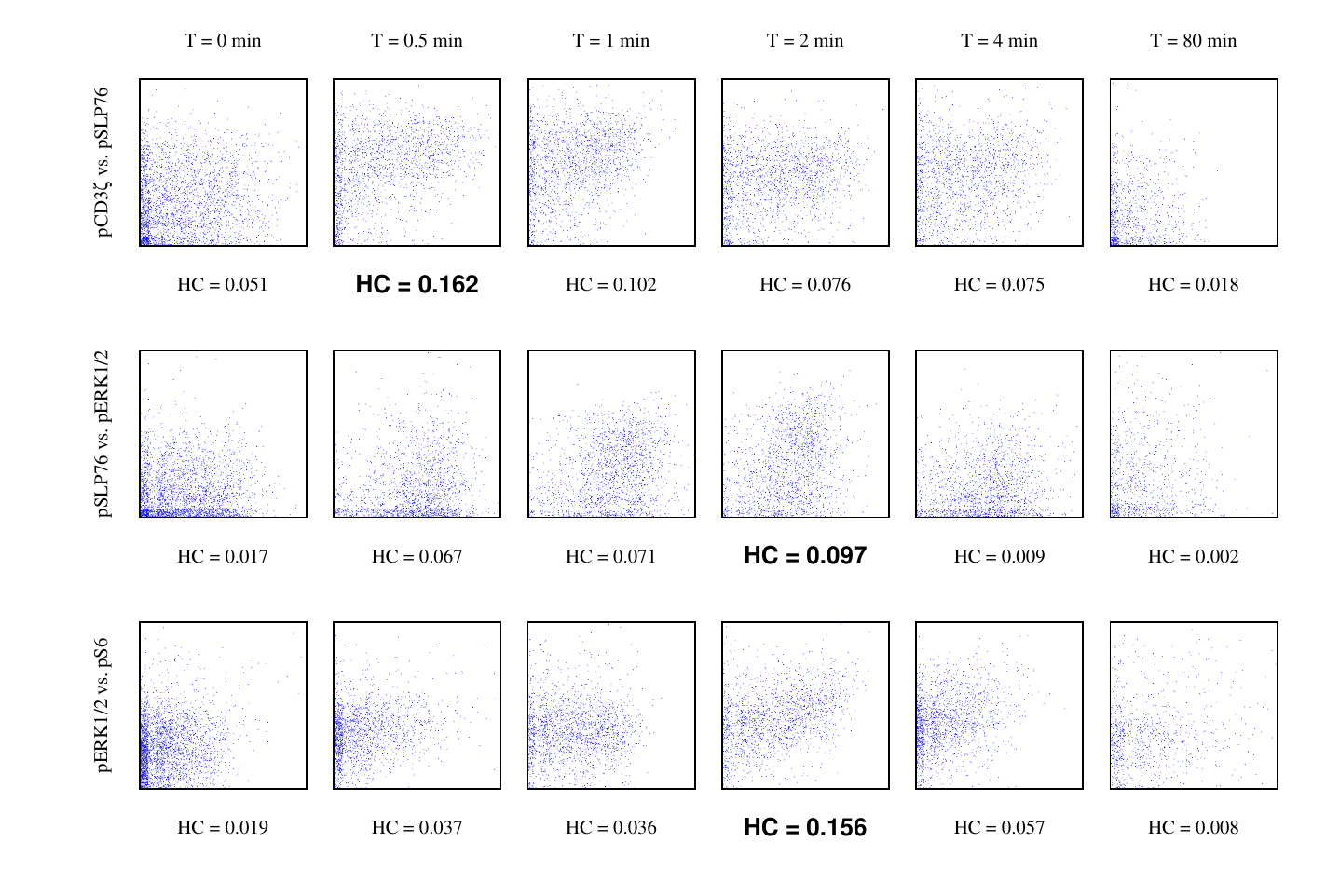}
\end{center}
\caption{Scatter plots of gene pathway data for various pair of data and various time points (T-cells exposed with an antigen).}\label{fig:fcs_effmem}
\vspace{-1em}
\end{figure}

\begin{figure}[!ht]
\begin{center}
\includegraphics[width=2.2in]{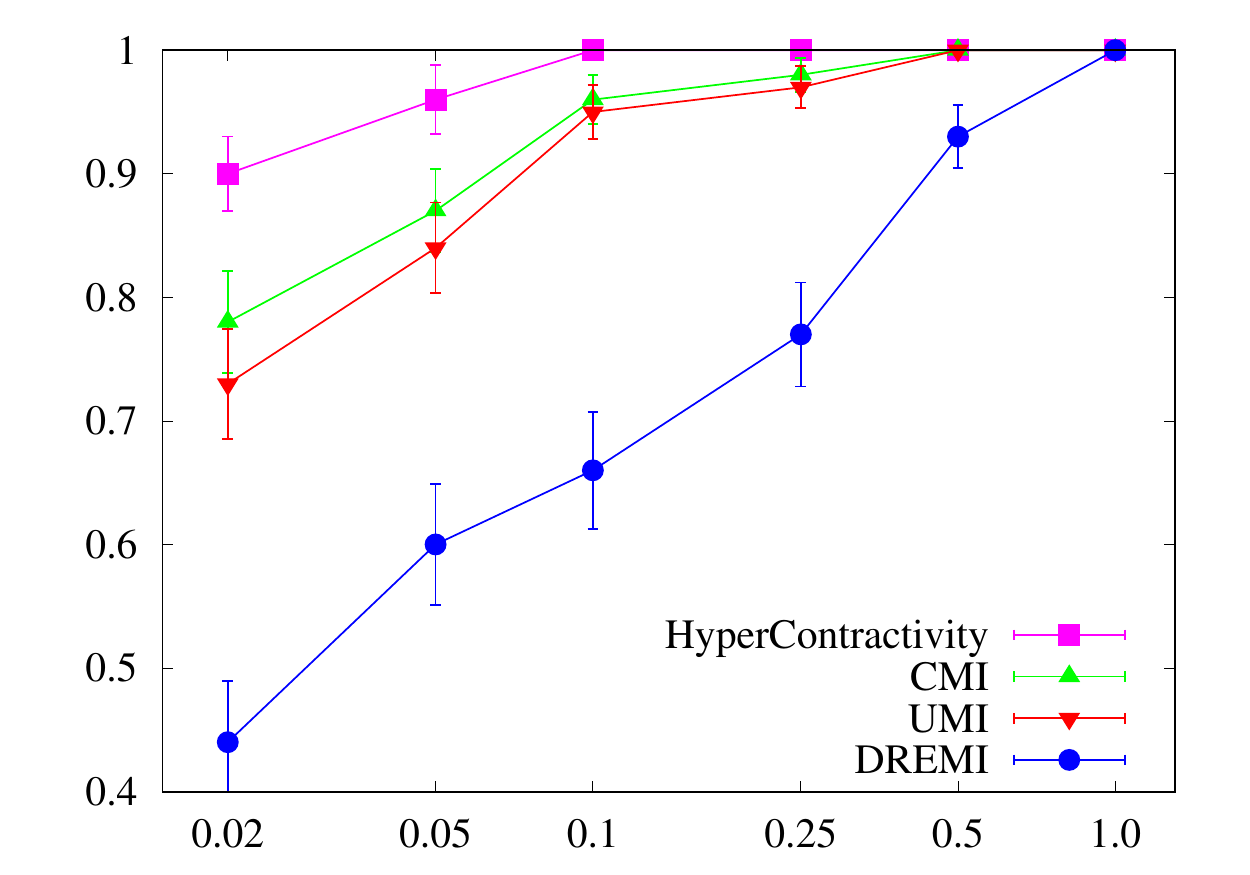}
\includegraphics[width=2.2in]{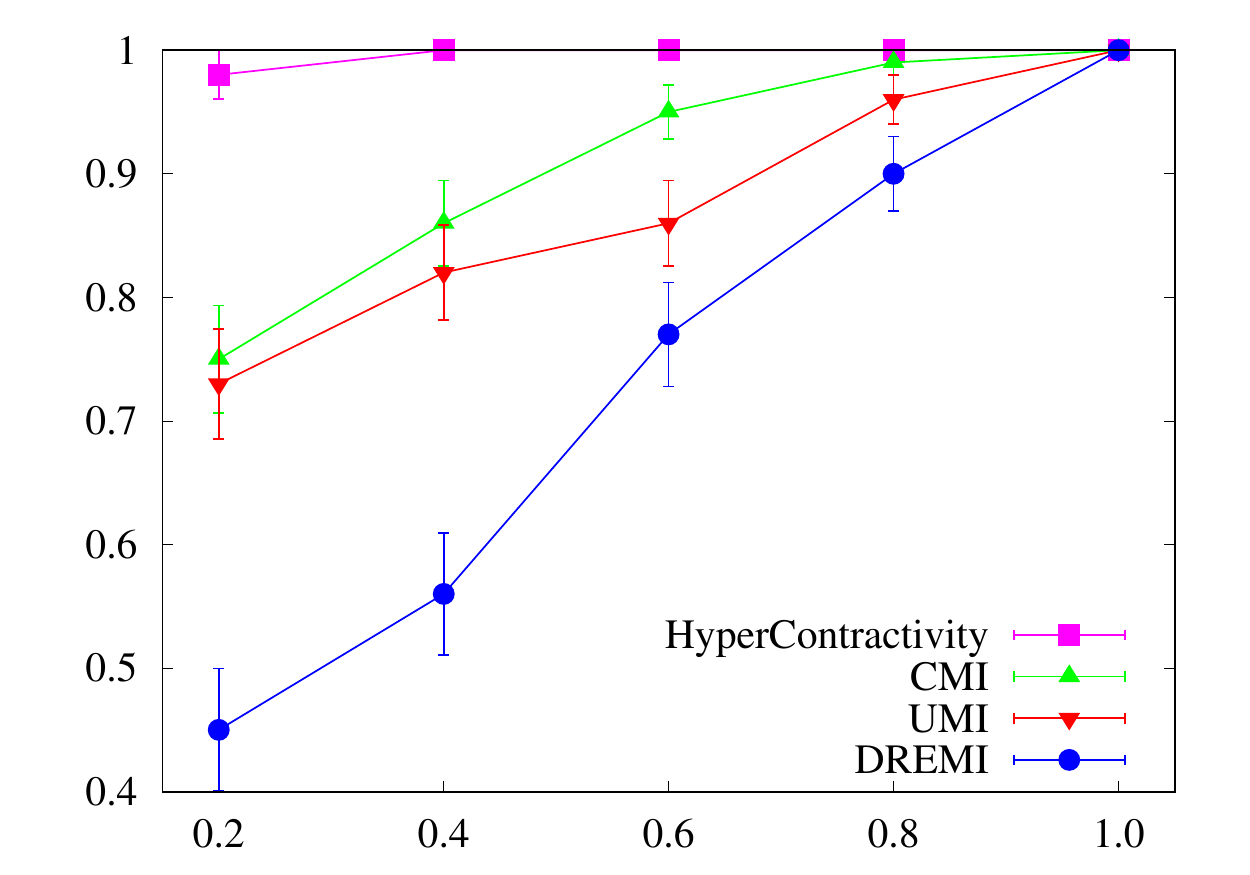}
\put(-110,-5){\small resampling rate}
\put(-270,-5){\small resampling rate}
\put(-325,15){\small \rotatebox{90}{probability of success}}
\end{center}
\caption{Accuracy vs. subsampling rate. Hypercontractivity method has higher probability to recover the trend when data size is smaller compared to other methods. Left: regular T-cells. Right: T-cells exposed with an antigen~\cite{Krishnaswamy14}.}\label{fig:fcs_experiment}
\end{figure}




\section{Proofs}
In this section, we provide proofs for our main results and technical lemmas.

\subsection{Proof of Proposition~\ref{pro:renyi}} 
\label{sec:renyi_proof}

Let $S_F(X,Y) = F(\sqrt{s(X;Y)}, \sqrt{s(Y;X)})$ for $F$ satisfying conditions in Proposition~\ref{pro:renyi}.  We show that $S_F(X,Y)$ satisfies all R{\'e}nyi's axioms, i.e., Axioms 1-5 and 6' and 7'. 

\begin{enumerate}
\item $S_F(X,Y)$ is defined for any pair of non-constant random variables $X, Y$ because $s(X;Y) \in [0,1]$ and $s(Y;X) \in [0,1]$ are defined for any random variables $X$, $Y$ by  Theorem~\ref{thm:axiom}. 
\item $S_F(X,Y) \in [0,1]$ because the output of a function $F$ is in $[0,1]$ by the condition on $F$. 

\item If $X$ and $Y$ are statistically independent, $s(X;Y) = s(Y;X) = 0$. By the condition on $F$, it follows that $S_F(X,Y) =0$.  If $S_F(X,Y)=0$, by the condition on $F$,   $s(X;Y)s(Y;X) = 0$, which implies that $X$ and $Y$ are statistically independent. 

\item $S_F(f(X),g(Y)) = S_F(X,Y)$ for any bijective Borel-measurable functions $f,g$ because
$\sqrt{s(f(X);g(Y))} = \sqrt{s(X;Y)}$ and $\sqrt{s(g(Y);f(X))}  = \sqrt{s(Y;X)}$ by Theorem~\ref{thm:axiom}.

\item For $(X,Y) \sim \mathcal{N}(\mu, \Sigma)$ with Pearson correlation $\rho$, $s(X;Y) = s(Y;X) =  \rho^2$. Hence, $S_F(X,Y) = F(|\rho|,|\rho|) = |\rho|$.

\item[6'] If $Y=f(X)$ for a non-constant function $f$, it follows that $I(f(X);f(X)) = I(f(X);X)$ because if $f(X)$ is discrete, $I(f(X);f(X)) = I(f(X);X) = H(f(X))$ and otherwise, $I(f(X);f(X)) = I(f(X);X) = \infty$. Hence  
\[
s(X;f(X)) = \sup_{U-X-f(X)} I(U;f(X))/I(U;X) = I(f(X);f(X))/I(f(X);X) = 1.
\]
Similarly, $s(f(X);X) = \sup_{U-f(X)-X}I(U;X)/I(U;f(X)) = 1$. Hence, $S_F(X;f(X)) = F(1,1) = 1$. Likewise, we can show that $S_F(X;Y) = 1$ if $X = g(Y)$.

\item[7'] $S_F(X,Y)=S_F(Y,X)$ because $F(x,y)=F(y,x)$. 
\end{enumerate}

\subsection{Proof of Theorem~\ref{thm:axiom}}
\label{sec:axiom_proof}

We show that $s(X;Y)$ satisfies Axioms 1-6 in Section~\ref{sec:axiom}.
\begin{enumerate}
\item For any non-constant random variable $X$, $\exists\ U$ s.t. $I(U;X) > 0$. Hence,
$s(X;Y)$ is defined for any pair of non-constant random variables $X$ and $Y$.
\item Since mutual information is non-negative, $s(X;Y) \ge 0$. By data processing inequality, for any $U-X-Y$, $I(U;X) \le I(U;Y)$. Hence, $s(X;Y) \le 1$. 
\item If $X$ and $Y$ are independent, for any $U$, $I(U;Y) \le I(X;Y) = 0$. Hence, $s(X;Y) = 0$. If $X$ and $Y$ are dependent, $I(X;Y)>0$, which implies that $s(X;Y) \ge I(X;Y)/H(X) > 0$.
 \item For any bijective functions $f,g$,
 \[
 I(U;g(Y)) = I(U;g(Y),Y) = I(U;Y) + I(U;g(Y)|Y) = I(U;Y).
 \]
 Similarly, $I(U;f(X)) = I(U;X)$. Hence,
 \begin{align*}
s(f(X); g(Y)) &= \sup_{U:U-f(X)-g(Y), I(U;f(X)) > 0} \frac{I(U;g(Y))}{I(U;f(X))} \\
&= \sup_{U:U-X-f(X)-g(Y)-Y, I(U;X) > 0} \frac{I(U;Y)}{I(U;X)} \\
&=  s(X;Y).
 \end{align*}
\item
By Theorem 3.1 in~\cite{Chechik--Amir--Naftali--Yair2005}, for $(X,Y)$ jointly Gaussian with correlation coefficient $\rho$,
\[
\min_{U\colon U-X-Y} \left(I(U;X) - \beta I(U;Y)  \right)  = 0
\]
for $\beta \le 1/\rho^2$. Equivalently,
\[
\max_{U\colon U-X-Y} \left( I(U;Y) - \rho^2 I(U;X) \right) = 0,
\]
which implies that $s(X;Y) \le \rho^2$.  To show that $s(X;Y) \ge \rho^2$, let $U_Z  = X + Z$ for $Z \sim (0,\sigma_1^2)$. Consider
\begin{align*}
s(X;Y) &\ge \lim_{\sigma_1^2 \to \infty} \frac{I(U_Z;Y)}{I(U_Z;X)} \\
		& = \lim_{\sigma_1^2 \to \infty} \frac{\log\left( \frac{(\sigma_X^2+\sigma_1^2)\sigma_Y^2}{(\sigma_X^2+\sigma_1^2)\sigma_Y^2 - \rho^2 \sigma_X^2 \sigma_Y^2}\right)}{ \log\left(1+\frac{\sigma_X^2}{\sigma_1^2}\right) }\\
		&= \lim_{\sigma_1^2 \to \infty} \frac{\rho^2 \sigma_X^2 \sigma_Y^2/\left((\sigma_X^2+\sigma_1^2)\sigma_Y^2 - \rho^2 \sigma_X^2 \sigma_Y^2\right)}{\sigma_X^2/\sigma_1^2} \\
		&= \rho^2.
\end{align*}
Hence, $s(X;Y) = \rho^2$. An alternative proof is provided in~\cite{nair2014gaussian}. 
\item To prove that $s(X;Y)$ satisfies Axiom 6, we first show the following lemma.
\begin{lemma}\label{LB-general}
\textnormal{Consider a pair of random variables $(X,Y) \in \Xcal \times \Ycal$. The hypercontractivity $s(X;Y)$ is lower bounded by
\begin{align}\label{LB-I}
s(X;Y) \ge \frac{I(U;Y|X \in \mathcal{X}_r)}{H(\alpha)/\alpha + I(U;X|X \in \mathcal{X}_r)}
\end{align}
for any $\Xcal_r$ such that $\Xcal_r \subseteq \Xcal$ for $\P \{X \in \Xcal_r\} =: \a > 0$.
}\end{lemma}
\begin{proof} 
Let 
\begin{align}\label{Us}
U_s = \begin{cases}
U \sim p(u|x) &\text{ if } X \in \mathcal{X}_r,\\
\emptyset &\text{ otherwise.}
\end{cases}
\end{align}
Let $S = \mathbb{I}_{\{U_s = \emptyset\}} = \mathbb{I}_{\{X \in \Xcal_r\}}$. Note that $S - U_s - X - Y$ holds, and that $S$ is a deterministic function of $X$. Hence, 
\begin{align*}
I(U_s;X) &= I(U_s,S;X)\\
       &= I(S;X) + I(U_s;X|S)\\
       &= H(\alpha) + \alpha I(U;X|X \in \mathcal{X}_r). \numberthis\label{IUX}
\end{align*}
Consider
\begin{align*}
I(U_s;Y) &= I(U_s,S;Y)\\
       &= I(S;Y) + I(U_s;Y|S)\\
       &\ge \alpha I(U;Y|X \in \mathcal{X}_r).  \numberthis\label{IUY}
\end{align*}
The proof is completed by combining~\eqref{IUX} and~\eqref{IUY}.
\end{proof}
Assume that $Y = f(X)$ for $X \in \mathcal{X}_r$. Considering $U = f(X)$ in~\eqref{Us} in Lemma~\ref{LB-general}, we obtain the following lower bound: 
\[
s(X;Y) \ge \frac{I(f(X);f(X)|X \in \mathcal{X}_r)}{H(\alpha)/\a + I(f(X);X|X \in \mathcal{X}_r)}.
\]
For any continuous random variable $X$ and a non-constant continuous function $f$, $I(f(X);f(X)|X \in \mathcal{X}_r) = I(f(X);X|X \in \mathcal{X}_r) = \infty$, which implies that $s(X;Y) = 1$. 
\end{enumerate}

\subsection{Proof of Theorem~\ref{prop-max}}
\label{sec:correlation_proof}

We first prove that ${\rm mCor}(X,Y) = \sqrt{\alpha} \; {\rm mCor}(X_r,Y)$ in \eqref{eq:mcor_bound}. 
Let $S = \mathbb{I}_{\{X \in \Xcal_r\}}$ be the indicator for whether $X \in \Xcal_r$ or not.
Consider
\begin{align*}
{\rm mCor}(X;Y) &= \max_{\substack{f,g\\ :\E[f(X)] = \E[g(Y)] = 0,\\ \E[f^2(X)] \le 1, \E[g^2(Y)] \le 1}} \E[f(X)g(Y)]\\
			&=\max_{\substack{f,g\\ :\E[f(X)] = \E[g(Y)] = 0,\\ \E[f^2(X)] \le 1, \E[g^2(Y)] \le 1}}  \E_{S}[\E[f(X)g(Y)|S]]\\
			&= \max_{\substack{f,g\\ :\E[f(X)] = \E[g(Y)] = 0,\\ \E[f^2(X)] \le 1, \E[g^2(Y)] \le 1}}  \left(\a \E[f(X)g(Y)|X \in \Xcal_r] + \ba \E[f(X)g(Y)|X \in \Xcal_d]\right)\\
			&= \max_{\substack{f,g\\ :\E[f(X)] = \E[g(Y)] = 0,\\ \E[f^2(X)] \le 1, \E[g^2(Y)] \le 1}}  \left(\a \E[f(X)g(Y)|X \in \Xcal_r] + \ba \E[f(X)|X \in \Xcal_d]\E[g(Y)|X \in \Xcal_d]\right)\\
			&\stackrel{(a)}{=} \a \max_{\substack{f,g\\ :\E[f(X)] = \E[g(Y)] = 0,\\ \E[f^2(X)] \le 1, \E[g^2(Y)] \le 1}} \E[f(X)g(Y)|X \in \Xcal_r] \\
			&\stackrel{(b)}{=} \sqrt{\a} \; {\rm mCor}(X_r,Y).
\end{align*}
Step $(a)$ holds since $\E[g(Y)|X \in \Xcal_r] = \E[g(Y)|X \in \Xcal_d]$ from the assumption that marginal distributions are equal, and that $\E[g(Y)] = \a \E[g(Y)|X \in \Xcal_r] + \ba \E[g(Y)|X \in \Xcal_d]$. To show step $(b)$, let $c = \E[f(X)|X \in \Xcal_d]$ and note that 
\begin{align*}
\a \E[f(X)|X \in \Xcal_r] &= - \ba c,\\
\a \E[f^2(X)|X \in \Xcal_r]  &=  \E[f^2(X)]  - \ba \E[f^2(X)|X \in \Xcal_d]  \\
				&\le 1 - \ba c^2,\\
\E[g(Y)|X \in \Xcal_r] &= 0.
\end{align*}
Hence, 
\begin{align*}
\max_{\substack{f,g\\ :\E[f(X)] = \E[g(Y)] = 0,\\ \E[f^2(X)] \le 1, \E[g^2(Y)] \le 1}} \E[f(X)g(Y)|X \in \Xcal_r] &= \max_{\substack{f_r,g\\ :\E[f_r(X)] = -\ba c/\a, \E[g(Y)] = 0,\\ \E[f_r^2(X)] \le (1 - \ba c^2)/\a, \E[g^2(Y)] \le 1}} \E[f_r(X)g(Y)]\\
&= \max_{\substack{f_{rc},g\\ :\E[f_{rc}(X)] = 0, \E[g(Y)] = 0,\\ \E[f_{rc}^2(X)] \le (\a - \ba c^2)/\a^2, \E[g^2(Y)] \le 1}} \E[(f_{rc}(X)g(Y)]\\
&= \max_{\substack{f_{rc},g\\ :\E[f_{rc}(X)] = 0, \E[g(Y)] = 0,\\ \E[f_{rc}^2(X)] \le 1/\a, \E[g^2(Y)] \le 1}} \E[f_{rc}(X)g(Y)]\\
&= \max_{\substack{f_{rca},g\\ :\E[f_{rca}(X)] = 0, \E[g(Y)] = 0,\\ \E[f_{rca}^2(X)] \le 1, \E[g^2(Y)] \le 1}} \frac{1}{\sqrt{\a}}\E[f_{rca}(X)g(Y)]\\
&= \frac{{\rm mCor}(X_r,Y)}{\sqrt{\a}},
\end{align*}
where $f_r(X)$, $f_{rc}(X) = f_r(X)+\ba c/\a$, and $f_{rca}(X) = \sqrt{\a}f_{rc}(X)$ are functions defined only for $X \in \Xcal_r$.

We next show  ${\rm dCor}(X,Y) = \alpha \; {\rm dCor}(X_r,Y)$ in \eqref{eq:dcor_bound}. 
Let
\[
h_X(s) = \E[e^{isX}],\  h_Y(t) = \E[e^{itY}],\ h_{XY}(s,t) = \E[e^{i(sX+tY)}].
\]
Note that
\begin{align*}
h_{XY}(s,t) &= \E[e^{i(sX+tY)}]\\
&= \alpha \E[e^{i(sX+tY)}|X \in \Xcal_r] + \bar{\alpha} \E[e^{isX}|X \in \Xcal_d] \E[e^{itY}|X \in \Xcal_d]\\
&= \alpha \E[e^{i(sX+tY)}|X \in \Xcal_r] + \bar{\alpha} \E[e^{isX}|X \in \Xcal_d] \E[e^{itY}],\numberthis\label{dist1}
\end{align*}
and
\begin{align}\label{dist2}
h_X(s) = \E[e^{isX}] = \alpha \E[e^{isX}|X \in \Xcal_r] + \bar{\alpha} \E[e^{isX}|X \in \Xcal_d].
\end{align}
By combining~\eqref{dist1} and~\eqref{dist2},
\begin{align*}
h_{XY}(s,t) - h_{X}(s)h_{Y}(t) &= \alpha \E[e^{i(sX+tY)}|X \in \Xcal_r] - \alpha \E[e^{isX}|X \in \Xcal_r] \E[e^{itY}]\\
					&= \alpha \E[e^{i(sX+tY)}|X \in \Xcal_r] - \alpha \E[e^{isX}|X \in \Xcal_r] \E[e^{itY}|X \in \Xcal_r]\\
					&=\alpha \; \mathrm{dCor}(X_r,Y).
\end{align*}

Finally, we show that ${\rm MIC}(X,Y) \leq \alpha\ {\rm MIC}(X_r,Y)$ in \eqref{eq:mic_bound}. 

Let $X_Q(X) \in \Xcal_Q(X)$ and $Y_Q(Y) \in \Ycal_Q(Y)$ denote a quantization of $X$ and $Y$, respectively.  Consider
\begin{align*}
\mathrm{MIC}(X,Y) &= \max_{X_Q(X), Y_Q(Y)}\frac{I(X_Q; Y_Q)}{\log \min\{|\Xcal_Q|,|\Ycal_Q|\}} \\
&\le \max_{X_Q(X), Y_Q(Y)}\frac{I(\mathbb{I}_{X \in \Xcal_r}, X_Q; Y_Q)}{\log \min\{|\Xcal_Q|,|\Ycal_Q|\}}\\
&\stackrel{(a)}{=} \alpha \max_{X_Q(X), Y_Q(Y)}\frac{I(X_Q;Y_Q|X \in \Xcal_r)}{\log \min\{|\Xcal_Q|, |\Ycal_Q|\}} \\
& \le \alpha \max_{X_Q(X_r), Y_Q(Y)}\frac{I(X_Q;Y_Q|X \in \Xcal_r)}{\log \min\{|\Xcal_Q(X_r)|, |\Ycal_Q|\}}\\
& = \alpha \; \mathrm{MIC}(X_r,Y),
\end{align*}
where step $(a)$ holds because $\mathbb{I}_{X \in \Xcal_r} \indep Y$ implies $\mathbb{I}_{X \in \Xcal_r} \indep Y_Q$ and $X \indep Y$ in $X \in \mathcal{X}_d$ implies $X_Q \indep Y_Q$ in $X \in \mathcal{X}_d$.

\subsection{Proof of Proposition~\ref{YX-domrare}}
\label{sec:inverse_proof}
The inverse hypercontractivity $s(Y;X)$ is defined as
\[
s(Y;X) = \sup_{U-Y-X}\frac{I(U;X)}{I(U;Y)}.
\]
Let $\mathbb{I}_r = \mathbb{I}_{\{X \in \mathcal{X}_r\}}$. Since the marginal distribution of $Y$ given $\{X \in \mathcal{X}_r\}$ and the one given $\{X \notin \mathcal{X}_r\}$ are equivalent, $Y$ and $\mathbb{I}_r$ are independent, i.e., $I(Y;\mathbb{I}_r) = 0$. For any $U$ such that Markov chain $U - Y - X$ holds, the Markov chain $U - Y - X - \mathbb{I}_r$ holds. Hence, $I(U;\mathbb{I}_r) = 0$. Hence, for any $U-Y-X$, consider  
\begin{align*}
    I(U;X)  &= I(U;X, \mathbb{I}_r)\\
            &= I(U;X|\mathbb{I}_r)\\
            &=(1-\alpha) I(U;X|\mathbb{I}_r = 0) + \alpha I(U;X|\mathbb{I}_r = 1)\\
            &\stackrel{(a)}{=}\alpha I(U;X|\mathbb{I}_r = 1)
\end{align*}
Step $(a)$ holds because $Y \indep\ X$ given $\mathbb{I}_r = 0$. Consider 
\begin{align*}
    I(U;Y) &\stackrel{(a)}{=} I(U;Y,\mathbb{I}_r)\\
    	  &= I(U;Y|\mathbb{I}_r) + I(U;\mathbb{I}_r)\\
	  & \stackrel{(b)}{=} I(U;Y|\mathbb{I}_r) \\
	  & = \alpha I(U;Y|\mathbb{I}_r = 1) + (1-\alpha) I(U;Y|\mathbb{I}_r = 0)\\
	  & \stackrel{(c)}{=} I(U;Y|\mathbb{I}_r = 1),
\end{align*}
where step $(a)$ folllows since $U-Y-\mathbb{I}_r$. Step $(b)$ follows from $I(U;\mathbb{I}_r) = 0$. Step $(c)$ holds since $H(U|\mathbb{I}_r = 1)=H(U|\mathbb{I}_r = 0)$ and $U-Y-\mathbb{I}_r$. Therefore, for any $U-Y-X$, it follows that
\[
s(Y;X) = \sup_{U-Y-X} \frac{\alpha I(U;X|I_r=1)}{I(U;Y|I_r=1)} = \alpha\ s(Y;X_r).
\]

\subsection{Noisy potential correlation in Example~\ref{example-noisy}}
\label{append-nr}
Let
\[
U = X + Z,\ \  Z \sim \mathcal{N}(0, \sigma_1^2).
\]
Consider
\[
\defMI \ge \sup_{\sigma_1^2 \ge 0} \frac{\log \frac{(1+\sigma_1^2)}{(1+\sigma_1^2)-\rho^2}}{H(\a)/\a + \log(1+1/\sigma_1^2)}.
\]

The inequality~\eqref{ex-noisy} follows by choosing $\sigma_1^2 = (1-\rho^2)/\rho^2$.

\subsection{Noisy discrete potential correlation in Example~\ref{example-discrete-noisy}}
\label{append-ndr}
The inequality~\eqref{noise-s} follows by choosing $r(x) = \mathbb{I}_{\{X=1\}}$ in~\eqref{defKL-dicrete}.
To show~\eqref{noise-rho}, we show that
\begin{align}\label{rho-r}
{\rm mCor}(X_r,Y) = 1-\frac{k}{k-1}\e.
\end{align}
The rest follows because ${\rm mCor}(X;Y) = \sqrt{\a}\ {\rm mCor}(X_r,Y)$ by Proposition~\ref{prop-max}. To show~\eqref{rho-r}, we use the fact that maximal correlation is the second eigenvalue of $Q = P_X^{-1/2} P_{XY} P_Y^{-1/2}$ (see~\cite{Kumar2010} for a detailed proof). We can easily show that
\[
Q = \left(1-\frac{k}{k-1}\e\right) I + \frac{\e}{k-1} 1 1^T.
\]
First singular vector of $Q$ is $P_X^{1/2} = 1/\sqrt{k}$. Second singular vector $u_2$ is orthogonal to $1/\sqrt{k}$. The equation~\eqref{rho-r} follows because ${\rm mCor}(X_r;Y) = u_2^T Q u_2 = u_2^T (1-k\epsilon/(k-1))u_2$.

\subsection{Proof of Proposition~\ref{propo:multi}}
\label{sec:multi-proof}

We first prove the second part of proposition: the hypercontractivity coefficient $\sqrt{s(X;Y)}$ satisfies Axioms 1-4, 5', and 6. It follows immediately from Theorem~\ref{thm:axiom} that $\sqrt{s(X;Y)}$ satisfies Axioms 1-4 and 6 because in the proof of Theorem~\ref{thm:axiom} -- 1-4 and 6, the same argument holds for random vectors $X$ and $Y$. We can show that that $\sqrt{s(X;Y)}$ satisfies Axiom 5' using results from~\cite{Chechik--Amir--Naftali--Yair2005}. In~\cite{Chechik--Amir--Naftali--Yair2005}, it is shown that that as we increase $\beta$ starting from zero, $\min\{I(U;X) - \beta I(U;Y)\}$ departs form zero at $\beta = 1/\| \Sigma_X^{-1/2} \Sigma_{XY}\Sigma_Y^{-1/2} \|^2$ for jointly Gaussian random vectors $X$ and $Y$. This result implies that $\sqrt{s(X;Y)} = \| \Sigma_X^{-1/2} \Sigma_{XY}\Sigma_Y^{-1/2} \|$.

To show that maximal correlation of two random vectors satisfies Axioms 1-4, 6', and 7', we can follow the same arguments for showing that maximal correlation for two random variables satisfies Axioms 1-4, 6', and 7' by~\cite{Renyi1959}.  To show that maximal correlation satisfies Axiom 5', note that maximal correlation is upper bounded by hypercontractivity as shown in Remark~\ref{rem:connect} in Section~\ref{sec2.3}: hence mCor$(X;Y) \le \| \Sigma_X^{-1/2} \Sigma_{XY}\Sigma_Y^{-1/2} \|$ for a jointly Gaussian $X,Y$. Equality holds because mCor$(X,Y)$ is lower bounded by its canonical correlation, which is $\| \Sigma_X^{-1/2} \Sigma_{XY}\Sigma_Y^{-1/2} \|$ for jointly Gaussian random vectors $(X,Y)$~\cite{Chechik--Amir--Naftali--Yair2005}.

\subsection{Proof of Theorem~\ref{thm:convergence}}
\label{sec:proof-consistency}

We begin with the following assumptions:
\begin{itemize}
    \item[$(a)$] There exist finite constants $C_1 < C'_1 < C'_2 < C_2$ such that the ratio of the optimal $r_x^*$ and the true $p_x$ satisfies $r^*_x(x)/p_x(x) \in [C'_1, C'_2]$ for every $x \in \mathcal{X}$.
    \item[$(b)$] There exist finite constants $C'_0 > C_0 > 0$ such that the KL divergence $D(r_x^*||p_x) > C'_0$.
\end{itemize}

With a little abuse of notations, we define $s(r_x) = D(r_y||p_y)/D(r_x||p_x)$ and $\widehat{s}({\bf w}) = {\bf w}^T {\bf A} \log ({\bf A}^T {\bf w}) / {\bf w}^T \log {\bf w}$. Therefore, $s(X;Y) = \max_{r_x \in R} s(r_x)$ and $\widehat{s}_{\Delta}(X;Y) = \max_{{\bf w} \in T_{\Delta}} \widehat{s}({\bf w})$. Here $R$ is the probability simplex over all $r_x$. We want to bound the error $|\widehat{s}_{\Delta}(X;Y) - s(X;Y)|$. 
First, consider the quantity:
\begin{eqnarray}
s_{\Delta}(X;Y) &\equiv& \max_{r_x \in T_{\Delta}(R)} s(r_x) \;,
\end{eqnarray}
where the constraint set $T_{\Delta}(R)$ is defined as:
\begin{eqnarray}
T_{\Delta}(R) = \{r_x \in \mathbb{R}^{|\mathcal{X}|}: [ ( r_x(x)/p_x(x)) ] \in T_{\Delta} \textrm{ and }\sum_{x \in \mathcal{X}} r_x(x) \in [1- {|\mathcal{X}|}\Delta,1+{|\mathcal{X}|}\Delta]\}
\end{eqnarray}
Now we rewrite the error term as
\begin{eqnarray}
\Big|\, \widehat{s}_{\Delta}(X;Y) - s(X;Y) \,\Big| &\leq& |s_{\Delta}(X;Y) - s(X;Y)|  + |\widehat{s}_{\Delta}(X;Y) - s_{\Delta}(X;Y)|\;. \label{eq:ub_hc}
\end{eqnarray}

The first error comes from quantization. Let $r^*$ be the maximizer of $s(X;Y)$. By assumption, $r^*(x)/p_x(x) \in [C_1, C_2]$, for all $x$. Since $T_{\Delta}(R)$ is a quantization of the simplex $R$, so there exists an $r_0 \in T_{\Delta}(R)$ such that $|r_0(x) - r^*(x)| < \Delta $ for all $x \in \mathcal{X}$. Now we will bound the difference between $s(r_0)$ and $s(r^*)$ by the following lemma:

\begin{lemma}
    \label{lem:hc_quantization}
    If $r(x)/p(x) \in [C_1,C_2] $ and $r'(x)/p(x) \in [C_1,C_2]$ for all $x \in \mathcal{X}$, and $D(r_x||p_x) > C_0$ and $D(r'_x||p_x) > C_0$, then
    \begin{eqnarray}
    \Big|\, s(r) - s(r') \,\Big| \leq L \max_{x \in \mathcal{X}} |r(x)-r'(x)|\;,
    \end{eqnarray}
    for some positive constant $L$.
\end{lemma}

Next we have:
\begin{eqnarray}
s(X;Y) &=& s(r^*) \leq s(r_0) +  L \max_{x \in \mathcal{X}} |r_0(x)-r^*(x)| \,\notag\\
&\leq& \max_{r \in T_{\Delta}(R)} s(r) + L \Delta = s_{\Delta}(X;Y) + L \Delta.
\end{eqnarray}

Similarly, let $r^{**}$ be the maximizer of $s_{\Delta}(X;Y)$, we can also find a $r_1 \in R$ such that $|r_1(x) - r^{**}(x)| < \Delta $ for all $x \in \mathcal{X}$. Using Lemma~\ref{lem:hc_quantization} again, we will obtain $s_{\Delta}(X;Y) \leq s(X;Y) + L \Delta$. Therefore, the quantization error is bounded by $O(\Delta)$ with probability 1.
\\

Now consider the second term. Upper bound on the second term relies on the convergence of estimation of $s$. We claim that for given $r_x$, the estimator is convergent in probability, i.e.,

\begin{lemma}
    \label{lem:hc_convergent}
    \begin{eqnarray}
    \lim_{N \to \infty} \Pr \left(\, \Big|\, \widehat{s}({\bf w}_r) - s(r_x) \,\Big| > \varepsilon \,\right) =  0\;.
    \end{eqnarray}
\end{lemma}

Here ${\bf w}_r(x) = r_x(x)/p_x(x)$. Since the set $T_{\Delta}(R)$ is finite, by union bound, we have:
\begin{eqnarray}
&&\lim_{N \to \infty} \Pr \big(\, \forall r \in T_{\Delta}(R), \Big|\, \widehat{s}({\bf w}_r) - s(r_x) \,\Big| \leq \varepsilon \,\big) \,\notag\\
&\geq& 1 - |T_{\Delta}(R)| \lim_{N \to \infty} \Pr \left(\, \Big|\, \widehat{s}({\bf w}_r) - s(r_x) \,\Big| \leq \varepsilon \,\right) = 1. \label{eq:hc_cond1}
\end{eqnarray}
Also, by the strong law of large numbers, we have that
\begin{eqnarray}
\lim_{N \to \infty} \Pr \left(\,\forall x \in \mathcal{X}, |p_x(x) - \frac{n_x}{n}| <  \frac{\Delta}{C_2|\mathcal{X}|} \,\right) = 1. \label{eq:hc_cond2}
\end{eqnarray}
where $n_x = {\rm card}\{i \in [n]: x_i = x\}$. We claim that if the events inside the probability in~\eqref{eq:hc_cond1} and~\eqref{eq:hc_cond2} happen simultaneously, then $|\widehat{s}_{\Delta}(X;Y) - s_{\Delta}(X;Y)| < \varepsilon + O(\Delta)$, which implies the desired claim.
\\

Let ${\bf w}^* = \arg\max_{{\bf w} \in T_{\Delta}} \widehat{s}({\bf w})$. Define $r_2(x) = {\bf w}^*(x) p_x(x)$. Since $[r_2(x)/p_x(x)] \in T_{\Delta}$ for all $x$ and
\begin{eqnarray}
\Big|\, \sum_{x \in \mathcal{X}} r_2(x) - 1 \,\Big| &=& \Big|\, \sum_{x \in \mathcal{X}} {\bf w}^*_x (p_x(x) - \frac{n_x}{n})\,\Big| + \frac{\Delta|\mathcal{X}|}{2} \,\notag\\
	&\leq&  \,|\mathcal{X}| \,  \left(\, \frac{\Delta}{2} + C_2 \max_{x \in \mathcal{X}} \Big|\, p_x(x) - \frac{n_x}{n} \,\Big| \,\right) \,\notag\\
	&\leq& ({|\mathcal{X}|}/2+1)  \Delta \;.
	\label{eq:q2}
\end{eqnarray}
Therefore, $r_2 \in T_{\Delta}(R)$, so
\begin{eqnarray}
\widehat{s}_{\Delta}(X;Y) &=& \widehat{s}({\bf w}^*) \leq s(r_2) + \varepsilon \leq s_{\Delta}(X;Y) + \varepsilon.
\end{eqnarray}

On the other hand, consider $r^{**} = \arg\max_{r_x \in T_{\Delta}(R)} s(r_x)$ again, and define ${\bf w}_0(x) = r^{**}(x)/p_x(x)$. We know that ${\bf w}_0 \in T_{\Delta}^{|\mathcal{X}|}$ but not necessarily $\sum_{i=1}^n {\bf w}_0(x_i) = n$. But we claim that the sum is closed to $n$ as follows: 
\begin{eqnarray}
\Big|\, \sum_{i=1}^n {\bf w}_0(x_i) - n \,\Big| &=& \Big|\, \sum_{x \in \mathcal{X}} \frac{n_x r^{**}(x)}{p_x(x)} - n| \,\notag\\
&\leq& n \max_{x \in \mathcal{X}} \big\{\, \frac{r^{**}(x)}{p_x(x)} \Big|\, \frac{n_x}{n} - p_x(x) \,\Big| \,\big\}\,\notag\\
&\leq& n C_2 \frac{\Delta}{C_2|\mathcal{X}|} < n\Delta
\end{eqnarray}
so we can find a ${\bf w}_1 \in T_{\Delta}(R)$ such that $|{\bf w}_1(x) - {\bf w}_0(x)| \leq \Delta$ for all $x$. Let $r_4(x) = {\bf w}_1(x) p_x(x)$, similar as~\eqref{eq:q2}, we know that $r_4 \in T_{\Delta}(R)$. Moreover, $\Big|\, r_4(x) - r^{**}(x) \,\Big| \leq p_x(x) \Big|\, {\bf w}_1(x) - {\bf w}_0(x) \,\Big| \leq \Delta$ for all $x$.
Then we have
\begin{eqnarray}
s_{\Delta}(X;Y) &=& s(r^{**}) \leq s(r_4) + L \max_{x \in \mathcal{X}} |r^{**}(x) - r_4(x)| \,\notag\\
&\leq& \widehat{s}({\bf w}_1) + \varepsilon + L \Delta = \widehat{s}_{\Delta}(X;Y) + \varepsilon + L \Delta.
\end{eqnarray}
We conclude that $|\widehat{s}_{\Delta}(X;Y)- s_{\Delta}(X;Y)| < \varepsilon + O(\Delta)$; thus our proof is complete.

\subsubsection{Proof of Lemma~\ref{lem:hc_quantization}}
We will show that for any $x \in \mathcal{X}$, we have $|\partial s(r_x) / \partial r_x(x)|\leq L/|\mathcal{X}|$ for some $L$. Therefore,
\begin{eqnarray}
\Big|\, s(r) - s(r') \,\Big| &\leq& \sum_{x \in \mathcal{X}} |\frac{\partial s(r)}{\partial r_x(x)}| \, |r_x(x) - r'_x(x)| \,\notag\\
&\leq& L \max_{x \in \mathcal{X}} |r_x(x) - r'_x(x)| \;.
\end{eqnarray}

The gradient can be written as 
\begin{eqnarray}
\frac{\partial s(r)}{\partial r_x(x)} &=& \frac{\partial}{\partial r_x(x)} \frac{D(r_y||p_y)}{D(r_x||p_x)} \,\notag\\
&=& \frac{1}{D^2(r_x||p_x)} \left(\, \frac{\partial D(r_y||p_y)}{\partial r_x(x)} D(r_x||p_x) - \frac{\partial D(r_x||p_x)}{\partial r_x(x)} D(r_y||p_y)\,\right) \;.
\end{eqnarray}
Since
\begin{eqnarray}
\frac{\partial D(r_x||p_x)}{\partial r_x(x)} &=& \log \frac{r_x(x)}{p_x(x)} + 1 \leq \max\{|\log C_1|, |\log C_2|\} + 1\,\notag\\
\frac{\partial D(r_y||p_y)}{\partial r_x(x)} &=& \int \frac{\partial r_y(y)}{\partial r_x(x)} \frac{\partial D(r_y||p_y)}{\partial r_y(y)} dy \,\notag\\
&=& \int p_{y|x}(y|x) (\log \frac{r_y(y)}{p_y(y)} + 1) dy \leq \max\{|\log C_1|, |\log C_2|\} + 1
\end{eqnarray}
Therefore, we have
\begin{eqnarray}
|\frac{\partial s(r)}{\partial r_x(x)}| &\leq& (\max\{|\log C_1|, |\log C_2|\} + 1) \frac{D(p_x||r_x) +D(r_y||p_y)}{D^2(r_x||p_x)} \,\notag\\
&\leq& \frac{2(\max\{|\log C_1|, |\log C_2|\} + 1)}{D(r_x||p_x)} \,\notag\\
&\leq& \frac{2(\max\{|\log C_1|, |\log C_2|\} + 1)}{C_0}
\end{eqnarray}
Since $C_0, C_1, C_2$ are constants and $| \mathcal{X}|$ is finite, our proof is complete by letting $L = 2 |\mathcal{X}| (\max\{|\log C_1|, |\log C_2|\} + 1)/C_0$.

\subsubsection{Proof of Lemma~\ref{lem:hc_convergent}}

Note that $\widehat{s}({\bf w}_r) = {\bf w}^T {\bf A} \log ({\bf A}^T {\bf w}) / {\bf w}^T \log {\bf w}$. Define $\widehat{D}(r_y||p_y) = {\bf w}^T {\bf A} \log ({\bf A}^T {\bf w})$ and $\widehat{D}(r_x||p_x) = {\bf w}^T \log {\bf w}$. We will prove that both $\widehat{D}(r_y||p_y)$ converges to $D(r_y||p_y)$ and $\widehat{D}(r_x||p_x)$ converges to $D(r_x||p_x)$ in probability. Since $D(r_x||p_x) > 0$ and $\widehat{D}(r_x||p_x) > 0$ with probability 1, we obtain that $\widehat{s}({\bf w}_r)$ converges to $D(r_y||p_y)/D(r_x||p_x) = s(r_x)$ in probability.

The convergence $\widehat{D}(r_x||p_x)$ comes from law of large number. Since $\widehat{D}(r_x||p_x) = \frac{1}{n} \sum_{i=1}^n \frac{r_x(X_i)}{p_x(X_i)} \log \frac{r_x(X_i)}{p_x(X_i)}$ and $D(r_x||p_x) = \mathbb{E}_{X \sim p_x} \left[ \frac{r_x(X)}{p_x(X)} \log \frac{r_x(X)}{p_x(X)} \right]$, the weak law of large number shows the convergence in probability.

For the convergence of $\widehat{D}(r_y||p_y)$. Consider the vector ${\bf v} = {\bf A}^T {\bf w}$, we have
\begin{eqnarray*}
    v_j &=&  \frac{1}{n} \sum_{i=1}^n \frac{p_{xy}(X_i,Y_j)}{p_x(X_i)p_y(Y_j)} w_i = \frac{1}{n} \sum_{i=1}^n \frac{p_{y|x}(Y_j|X_i)}{p_y(Y_j)} \frac{r_x(X_i)}{p_x(X_i)}.
\end{eqnarray*}
On the other hand, for fixed $Y_j = y$, we have
\begin{eqnarray*}
    \frac{r_y(y)}{p_y(y)} = \frac{\mathbb{E}_{X \sim p_x} \left[\, p_{y|x}(y|X) \frac{r_x(X)}{p_x(X)}\,\right]}{p_y(y)} = \mathbb{E}_{X \sim p_x} \left[\, \frac{p_{y|x}(y|X)}{p_y(y)} \frac{r_x(X)}{p_x(X)}\,\right] \;.
\end{eqnarray*}
Therefore, by law of large number, we conclude that $v_j$ converges to $\frac{r_y(Y_j)}{p_y(Y_j)}$ in probability. Hence, $\widehat{D}(r_y||p_y) = \frac{1}{n} \sum_{j=1}^n v_j \log v_j$ converges to $\frac{1}{n} \sum_{j=1}^n \frac{r_y(Y_j)}{p_y(Y_j)} \log \frac{r_y(Y_j)}{p_y(Y_j)}$ in probability. Furthermore, $\frac{1}{n} \sum_{j=1}^n \frac{r_y(Y_j)}{p_y(Y_j)} \log \frac{r_y(Y_j)}{p_y(Y_j)}$ converges to $D(r_y||p_y) = \mathbb{E}_{Y \sim p_y} \left[ \frac{r_y(Y)}{p_y(Y)} \log \frac{r_y(Y)}{p_y(Y)} \right]$ in probability, by law of large number again. Therefore, we conclude that $\widehat{D}(r_y||p_y)$ converges to $D(r_y||p_y)$ in probability.  



\bibliographystyle{IEEEtran}
{\small
\bibliography{MLbib}}

\end{document}